\documentclass[preprint,12pt,authoryear]{elsarticle}

\usepackage[cmex10]{amsmath}
\usepackage{amssymb}
\usepackage{amsfonts}
\usepackage{amsthm}
\usepackage{algorithm}
\usepackage{algpseudocode}
\usepackage{graphicx}
\usepackage{epsfig}
\usepackage{gensymb}
\usepackage{hyperref}
\hypersetup{
    colorlinks=true,
    linkcolor=blue,
    filecolor=magenta,      
    urlcolor=cyan,
}
\usepackage{cite}
\usepackage{tensor}
\usepackage{colortbl}
\usepackage[caption=false,font=footnotesize]{subfig}
\usepackage{color}
\input{abkuerzung_IEEE.sty}
\graphicspath{{img/}}
\usepackage{bm}
\usepackage{url}
\usepackage[modulo,pagewise]{lineno}

\DeclareGraphicsExtensions{.eps}
\newtheorem{proposition}{Proposition}

\theoremstyle{definition}
\newtheorem{remark}{Remark}

\newtheorem*{problem}{Problem Statement}

\DeclareMathAlphabet\mathbfcal{OMS}{cmsy}{b}{n}

\journal{Robotics and Autonomous Systems}

\begin{document}

\begin{frontmatter}



\title{Vision-based Manipulation of Deformable and Rigid Objects Using Subspace Projections of 2D Contours}


\author[1,2]{Jihong Zhu}
\author[3]{David Navarro-Alarcon}
\author[1]{Robin Passama}
\author[1]{Andrea Cherubini}

\address[1]{LIRMM - Universit\'{e} de Montpellier CNRS, 161 Rue Ada, 34090 Montpellier, France. {\tt\small firstname.lastname@lirmm.fr}}
\address[2]{Delft University of Technology and Honda Research Institute, Europe. {\tt\small j.zhu-3@tudelft.nl}}
\address[3]{The Hong Kong Polytechnic University, Department of Mechanical Engineering, Kowloon, Hong Kong. {\tt\small david.navarro-alarcon@polyu.edu.hk}}

\begin{abstract}
This paper proposes a unified vision-based manipulation framework using image contours of deformable/rigid objects. Instead of explicitly defining the features by geometries or functions, the robot automatically learns the visual features from processed vision data. Our method simultaneously generates---from the same data---both visual features and the interaction matrix that relates them to the robot control inputs. Extraction of the feature vector and control commands is done online and adaptively, and requires little data for initialization. Our method allows the robot to manipulate an object without knowing whether it is rigid or deformable. To validate our approach, we conduct numerical simulations and experiments with both deformable and rigid objects.
\end{abstract}


\begin{highlights}
\item We present a unique framework for manipulating both rigid and deformable objects.
\item Our framework is model-free and requires a short initialization phase.
\item Our framework does not require camera calibration, and works with different camera poses.
\end{highlights}


\end{frontmatter}
\section{Introduction}
Humans are capable of manipulating both rigid and deformable objects. However, robotic researchers tend to consider the manipulation of these two classes of objects as separate problems. Unless otherwise mentioned, object rigidity is an implicit assumption in most manipulation tasks. On the other hand, methods designed for deformable object manipulation \citep{Sanchez2018}, are never applied on rigid objects.
This paper presents our efforts in formulating a generalized framework for vision-based manipulation of both rigid and deformable objects, which does not require
prior knowledge of the object's mechanical properties.

In the visual servoing literature~\citep{chaumette2006visual}, vector $\vs$ denotes the set of features selected to represent the object in the image. These features represent both the object's pose and its shape. We denote the process of selecting $\vs$ as \textit{parameterization}. The aim of visual servoing is to minimize, through robot motion, the feedback error $\ve = \vs^* - \vs$ between the target $\vs^*$ and the current (i.e., measured) feature $\vs$. 

One of the initial works on vision-based manipulation of deformable objects is presented in \citep{inoue1984hand} to solve a knotting problem by a topological model. Smith et al. developed a relative elasticity model, such that vision can be utilized without a physical model for the manipulation task \citep{smith1996vision}. A classical model-free approach in manipulating deformable objects is developed in \citep{berenson2013manipulation}. More recent research \citep{lagneau2020active} and \citep{lagneau2020automatic} proposes a method for online estimation of the deformation Jacobian, based on weighted least square minimization with a sliding window. In \citep{navarro2014visual} and  \citep{navarro2017fourier}, the vision-based deformable objects manipulation is termed as \textit{shape servoing}. An expository paper on the topic is available in \citep{navarro2019model}. A recent work on vision-based shape servoing of plastic material was presented in \citep{cherubini2020model}.  

For a detailed survey on shape servoing we refer readers to \citep{Sanchez2018}. For \textit{shape servoing}, commonly selected features are curvatures \citep{navarro2014visual}, points \citep{wang2018unified} and angles \citep{navarro2013uncalibrated}. Laranjeira et al. proposed a catenary-based feature for tethered management on wheeled and underwater robots \citep{laranjeira2017catenary,LARANJEIRA2020107018}. 
A more general feature vector is that containing the Fourier coefficients of the object contour \citep{navarro2017fourier,zhu2018dual}. Yet, all these approaches require the user to specify a model, e.g., the object geometry \citep{wang2018unified,navarro2013uncalibrated,navarro2014visual} or a function \citep{laranjeira2017catenary,navarro2017fourier,zhu2018dual} for selecting the feature. Alternative data-driven (hence, model-free) approaches rely on machine learning. Nair et al. combine learning and visual feedback to manipulate ropes in \citep{nair2017combining}. Li et al. approximate the deformation and camera model using a neural network \citep{li2018vision}. The authors of \citep{hu20193} employ deep neural networks to manipulate deformable objects given their 3D point cloud. All these methods rely on (deep) connectionist models, which invariably require training through an extensive data set. The collected data has to be diverse enough to generalize the model learnt by this type of networks. Instead of relying on algorithmic solutions, \citep{she2019cable} utilizes a vision-based tactile sensor (GelSight) for manipulating cables. 

It is noteworthy that some of the above mentioned methods may apply to rigid objects. 
Yet, none of the previous works has investigated the possibility of this extension nor reported its experimental validation, as we do in this paper.

The trend in visual servoing, when \textit{controlling the pose of rigid objects} is to find features which are independent from the object characteristics. Following this trend, \citep{chaumette2004image} proposes the use of image moments. More recently, researchers have proposed direct visual servoing (DVS) methods, which eliminate the need for user-defined features and for the related image processing procedures. The pioneer DVS works \citep{collewet2008visual,collewet2011photometric} propose using the whole image luminance to control the robot, leading to ``photometric'' visual servoing. Bakthavatchalam et al. join the two ideas by introducing photometric moments \citep{bakthavatchalam2013photometric}. 
A subspace method \citep{marchand2019subspace} can further enhance the convergence of photometric visual servoing, via Principal Component Analysis (PCA). This method was first introduced for visual servoing in \citep{nayar1996subspace}. In that work, using an eye-in-hand setup, the image was compressed to obtain a low-dimensional vector for controlling the robot to a target pose. Similarly, the authors of \citep{deguchi1996visual} transformed the image into a lower dimensional hyper surface, to control the robot position via in-hand camera feedback. However, DVS generally considers rigid and static scenes, where the robot controls the motion of the camera (eye-in-hand setup) to change only the image viewpoint, and not the environment. These constraints on the setup avoid breaking the Lambertian hypothesis that is needed, since DVS relies on the raw image luminance, which should not vary with the viewpoint. For this reason, to our knowledge, DVS was never applied to object manipulation, since changes in the pose and/or shape of the object would break the Lambertian assumption. This is not the case of feature-based methods (such as the one we present here), as long as the feature is chosen to be reliable even when the viewpoint and/or scene change.

Compared with the above-mentioned works, our paper presents the following original contributions:
\begin{enumerate}
    \item We propose to use a feature vector -- based on PCA of sampled 2D contours -- for model-free manipulation of both deformable and rigid objects.
    \item We exploit the linear properties of PCA and of the local interaction matrix, to initialize our algorithm with little data -- the same data for feature vector extraction and for interaction matrix estimation. 
    \item We report experiments using the same framework to manipulate objects with different unknown geometric and mechanical properties.
\end{enumerate}


The paper is organized as follows. Sect. \ref{sec:problem} presents the problem. Sect. \ref{sec:overview} outlines the framework. Sect. \ref{sec:method} elaborates on the methods. In Sect. \ref{sec:simu}, we analyze and verify the methods by numerical simulations. Then, Sect. \ref{sec:exp} presents the robotic experiments and we conclude in Sect. \ref{sec:conclusion}.

\section{Problem statement}\label{sec:problem}

In this work, we aim at solving object manipulation tasks with visual feedback. We rely on the following hypotheses:
\begin{itemize}
    \item
The shape and pose of the object are represented by its 2-D contour on the image as seen from a camera fixed in the robot workspace (eye-to-hand setup). We denote this contour as
\begin{equation}\label{eq:single_vector_vc}
 \vc = [\vp_1~\cdots~\vp_K]^T \in \mathbb{R}^{2K},
\end{equation}
where $\vp_j = [u_j~v_j] \in \mathbb{I}$ denotes the $j$th pixel of the contour in the image $\mathbb{I}$.
    \item 
The contour is always entirely visible in the scene and there are no occlusions.
    \item
One of the robot's end-effectors holds one point of the object (we consider the grasping problem to be already solved). At each control iteration $i$, its pose is $\vr_i \in \mathbb{SE}\left(3\right)$, and it can execute motion commands $\delta \vr_i \in \mathbb{SE}\left(3\right)$ that drive the robot so that $\vr_{i+1} = \vr_i + \delta \vr_i$.    \item
The target constant shape (i.e., contour) of the object, $\vc^*$, is physically reachable with shaping motions of the grasping point $\vr$. To ensure this hypothesis, one can first command the robot to verify that it can move the shape to $\vc^*$. 
\end{itemize}

\begin{figure}[tpb]
    \centering
	\subfloat[Rigid objects]{\includegraphics[width = 0.3\textwidth]{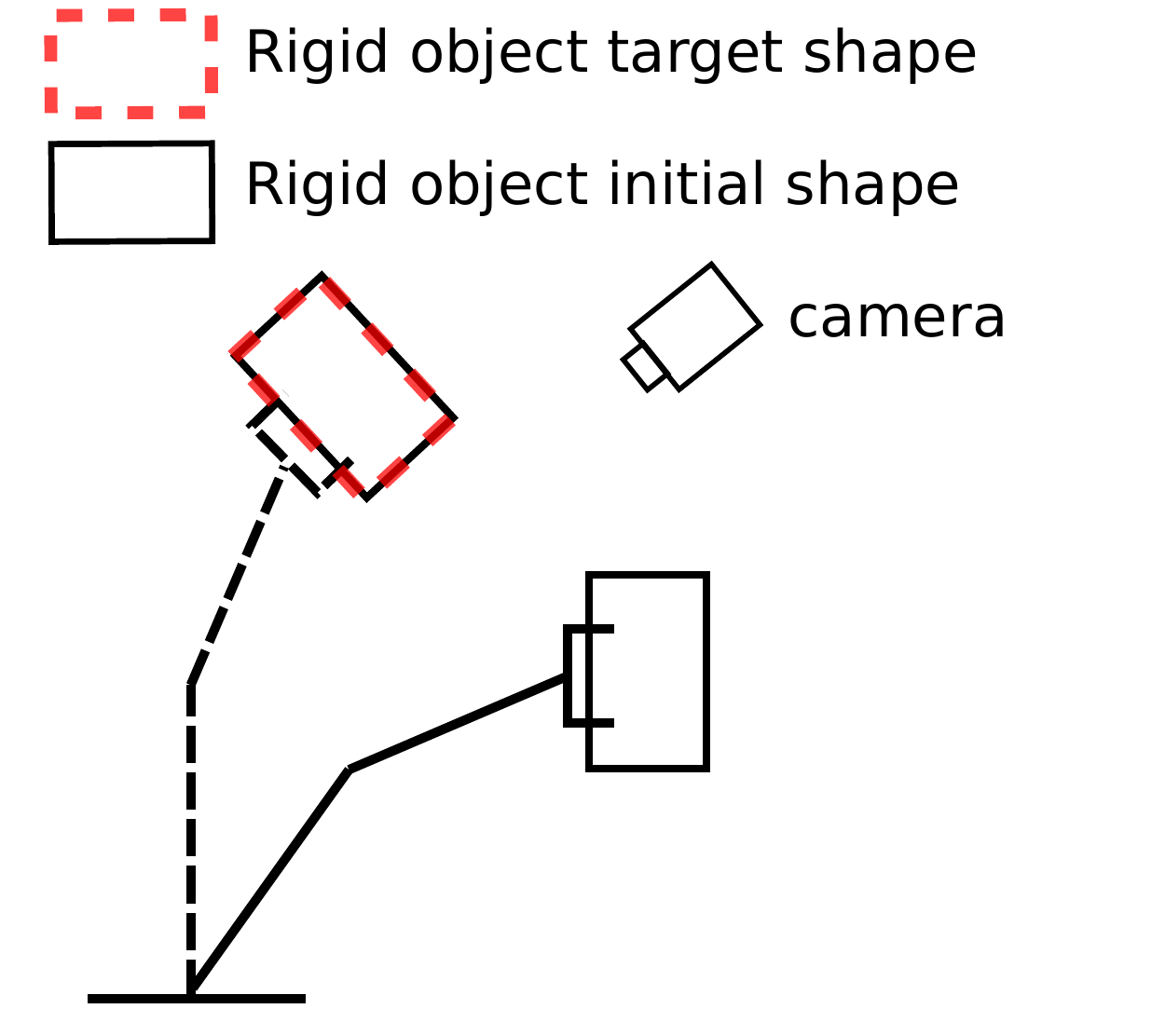}\label{fig:PCA_problem_statement_rigid}} \hspace{5mm}
    \subfloat[Deformable objects]{\includegraphics[width = 0.6\textwidth]{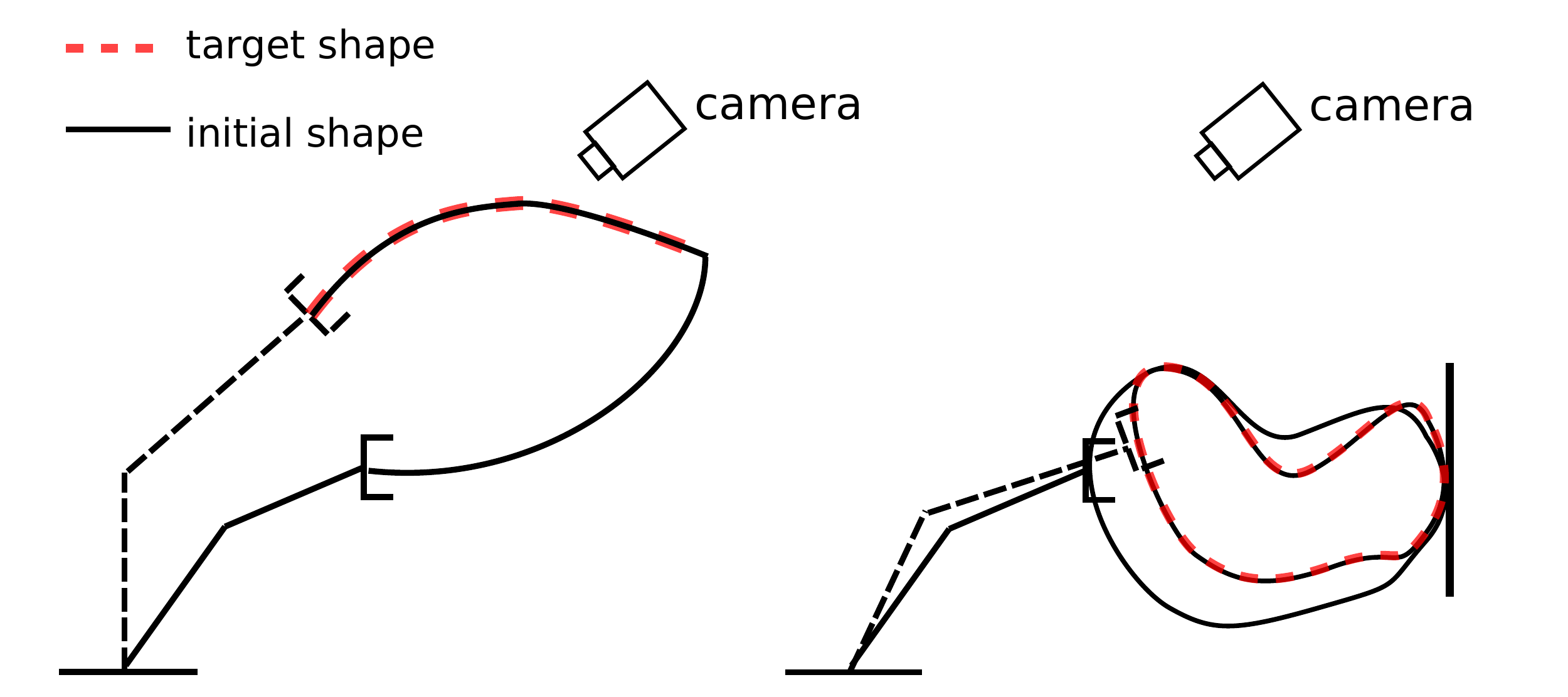}\label{fig:PCA_problem_statement_deform}}
	\caption{Vision-based manipulation of rigid and deformable objects. For rigid objects (left): control pose (translation and rotation). For deformable objects (right): control the pose, and also shape.}
	\label{fig:PCA_problem_statement}
\end{figure}

\begin{problem}
Given a target shape of the object, represented by a constant contour vector $\vc^*$, we aim at designing a vision-based controller that generates a sequence of robot motions $\delta\vr_i$ to drive the initial contour to the target one.
\end{problem}

The controller should work without any knowledge of the object physical characteristics, i.e., for both rigid and deformable objects. In the latter case, we assume that the deformation is homogeneous. Since rigid and deformable objects behave differently during manipulation, we set the following manipulation goals:
\begin{itemize}
    \item Rigid objects: move them to a target pose (see Fig. \ref{fig:PCA_problem_statement_rigid}).
    \item Deformable objects: move them to a target pose with a target shape (see Fig. \ref{fig:PCA_problem_statement_deform}).
\end{itemize}

The formulation of the problem is general, but due to challenges in perception (discussed in Sect. \ref{sec:conclusion}), we carried out the cases of study with movements in $\mathbb{SE}\left(2\right)$.    

\section{Preliminary}\label{sec:overview}
In this section, we present an overview of the proposed approach, motivated by the problem analysis. Throughout the paper, we use $\vc$ to indicate the \emph{object contour} and $\vs$ as the \emph{feature vector} obtained from the contour. The subscript $i$ indicates the instance of the variable at iteration $i$ (e.g., $\vc_i$ is the contour at iteration $i$). 

We can work directly on the object shape space by selecting the contour as the feature vector $\vs \equiv \vc \in \mathbb{R}^{2K}$. With image and data processing, we can extract a fixed number of ordered (i.e., identified) contour points to represent the shape/pose of the object. However, this will result in an unnecessarily large dimension of the feature vector (e.g., if $K = 50$, $\vs$ has 100 components). The high dimensional feature vector increases the computation demand and complicates the control due to the high under-actuation of the system. Therefore, instead of working on this feature vector, we work on one with smaller dimensions. To this end, we split the problem into two sub-problems: \emph{parameterization} and \emph{control}, see Fig. \ref{fig:shape-servo}. 

\begin{figure}[thpb]
	\centering
	\includegraphics[width = 0.6\textwidth]{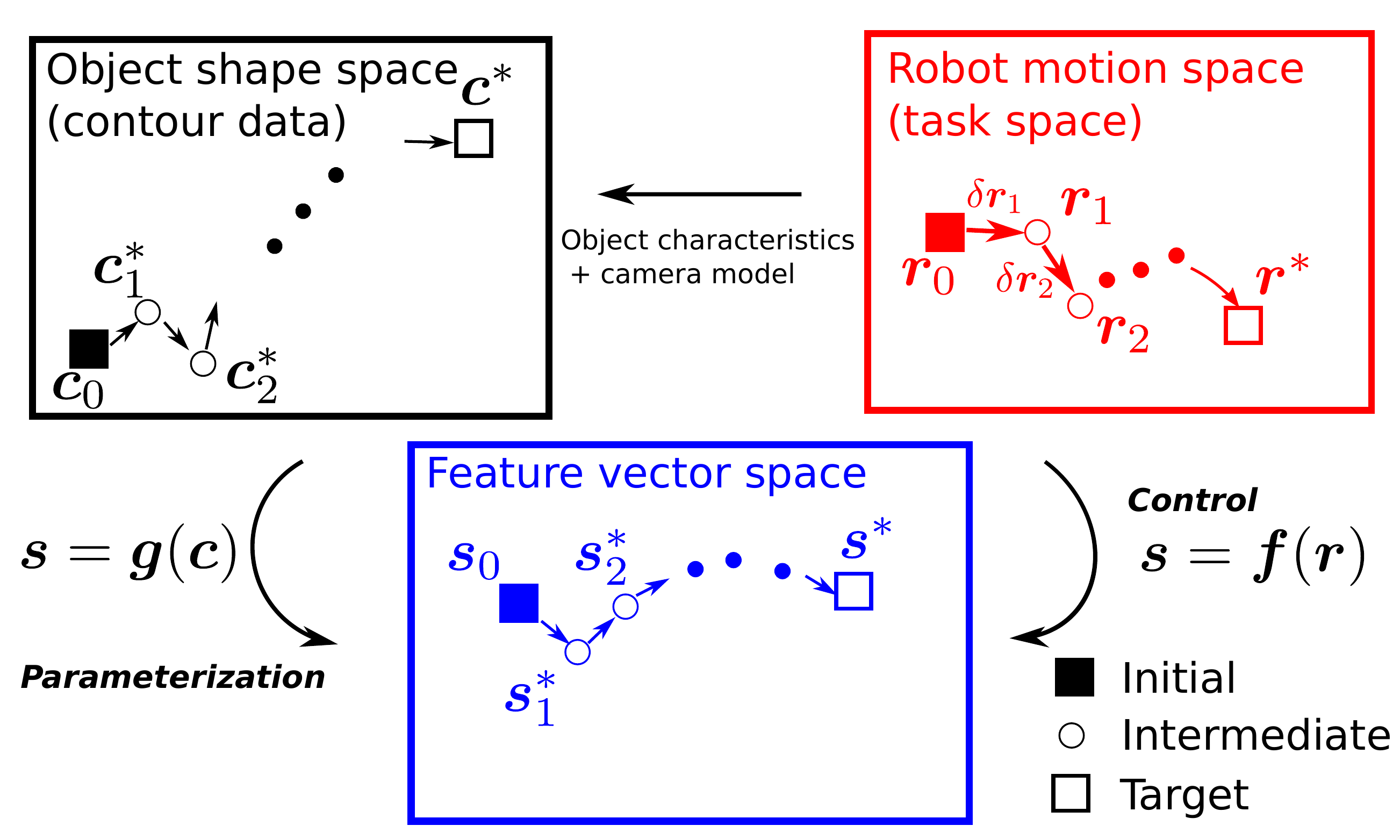}
	\caption{Graphic representation of the vision-based manipulation problem, with its two sub-problems, \emph{parameterization} and \emph{control}.}
	\label{fig:shape-servo}
\end{figure}

\emph{Parameterization} consists in representing the contour via a compact feature vector $\vs \in \mathbb{R}^k$, such that $k << 2K$. We denote this representation as $\vs = \vg(\vc)$. We introduce the method for parameterization in Sect. \ref{sec:feat_vector}.

\emph{Control} consists in computing robot motions $\delta \vr_1, \delta \vr_2, \dots$, so that the object's representation $\vs$ converges to the target $\vs^*$. \emph{Control} can be broken down to solving the optimization problem:
\begin{equation}\label{eq:solve_r}
    \vr^* = \arg\min_{\vr} (\vf(\vr) - \vs^*) 
\end{equation}
where $\vs = \vf(\vr)$ denotes the mapping between robot pose and feature vector, which is assumed to be smooth and generally nonlinear. The smoothness assumption requires that the objects' contour is at least twice differentiable with respect to the robot motion. If we know the analytic solution to $\vf(\vr)$, we can solve (\ref{eq:solve_r}) and obtain the target shape in a single iteration by commanding $\vr^*$. 

A solution to this problem is to approximate the full mapping $\vf(\vr)$ from sensor observations. Classic deep learning-based approaches typically require a long training phase to collect vast and diverse data for approximating $\vf(\vr)$. In some cases (for instance, robotics surgery), it is not possible to collect such data beforehand. Moreover, if the object changes, new data has to be collected to retrain the model, leading to a cumbersome process. In this paper instead, we aim at doing the data collection online, with minimum initialization.  

Thus, instead of estimating the full nonlinear mapping $\vf(\vr)$, we divide it into piece-wise linear models \citep{Journals:Sang2012} at successive equilibrium points. The locality assumption refers to both the time and spatial dimensions. These models are considered time invariant in the neighbourhood of the equilibrium points. We then compute the control law for each linear model and apply it to the robot end-effector. We will dedicate Sections~\ref{sec:local_target} and
~\ref{sec:local_model_est} to the local models and Sections \ref{sec:crtl} and~\ref{sec:m_target} to derive the control inputs and to analyze (local) stability.



\section{Methodology}\label{sec:method}
Given a target shape $\vc^*$, we define an intermediate local target $\vc^*_i$ at each $i = 1,2,\ldots$ (see Fig. \ref{fig:shape-servo}). At the $i^{\text{th}}$ iteration, the robot autonomously generates a local mapping $\vg_i$ to produce the feature vector $\vs_i = \vg_i(\vc_i)$. The robot then finds the local mapping $\vs_i = \vf_i(\vr_i)$ online.

Consider at the current time instant $i$, the shape $\vc_i$, the intermediate target $\vc^*_i$ and the local parameterization $\vg_i$. We can transform shape data into a feature vector by:
\begin{equation}
    \vs_i = \vg_i(\vc_i),~\vs^*_i = \vg_i(\vc^*_i).
\end{equation}
The linearized version of $\vs = \vf(\vr)$ centered at $\left(\vs_i, \vr_i\right)$ is then:
\begin{equation}\label{eq:linear_sys_A}
    \delta \vs_i = \vL_i \delta \vr_i,
\end{equation}
with
\begin{align}
    \begin{split}
        \vL_i &= \frac{\partial \vf_i}{\partial \vr}\mid_{\vr = \vr_i}, \\
        \delta \vs_i &= \vs_{i + 1} - \vs_i, \\
        \delta \vr_i &= \vr_{i + 1} - \vr_i.
    \end{split}
    \label{eq:differentials}
\end{align}
The matrix $\vL_i$ represents a local mapping, referred to as the interaction matrix in the visual servoing literature \citep{chaumette2006visual}. If $\vL_i$ can be estimated online at each iteration $i$, then, we can design one-step control laws to drive $\vs_i$ towards $\vs_i^*$. 

After the robot has executed the motion command $\delta \vr_i $, we update the next target to be $\vs_{i + 1}^*$, and so on, until it reaches the final target $\vs^*$. Although the validity region of this local mapping is smaller than that of the original nonlinear mapping, it enables to use an online training approach that requires less data and reduced computational demand.

Figure \ref{fig:overall-alg} shows the building blocks of the overall framework. In this section, we focus on the red dashed part of the diagram. We will elaborate on each red block in the subsequent subsections. The blue block represents the image processing pipeline that will be discussed in Sect. \ref{sec:5-vision-pipeline}.

\begin{figure}[thpb]
	\vspace{-0.3cm}
	\centering
	\includegraphics[width = 0.7\textwidth]{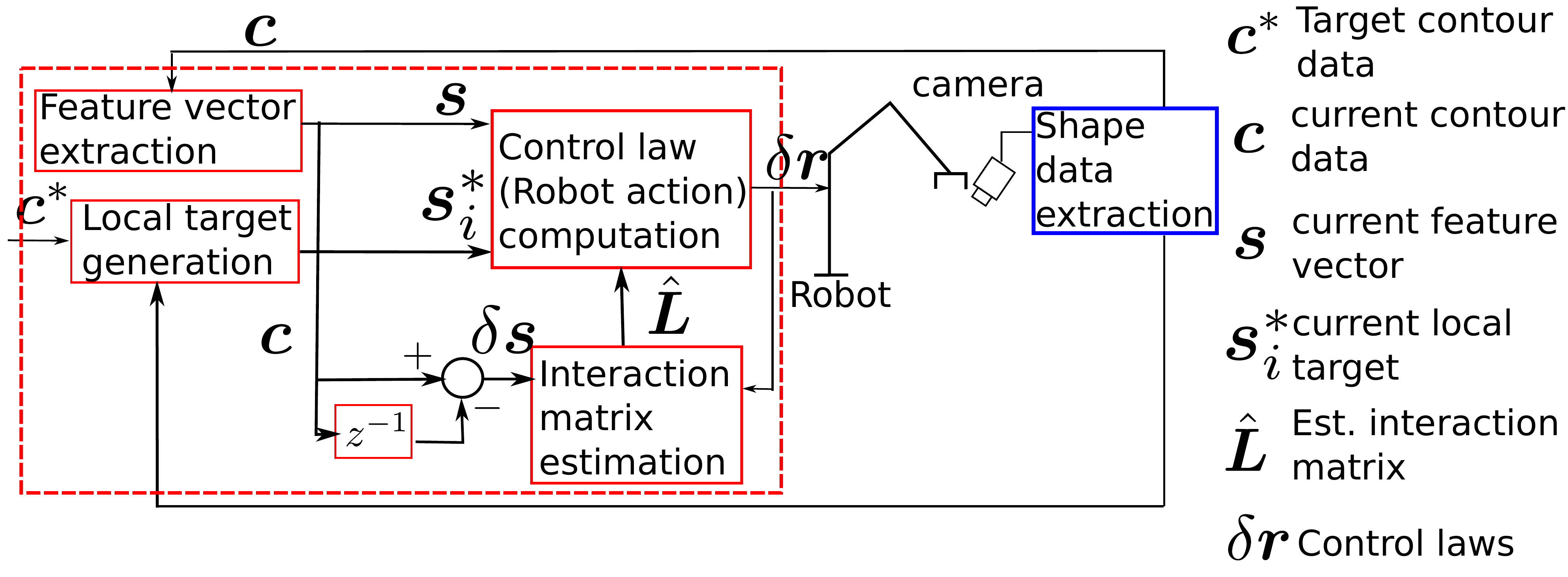}
	\vspace{-0.4cm}
	\caption{The block diagram that represents the overall framework.}
	\label{fig:overall-alg}
	\vspace{-0.5cm}
\end{figure}

\subsection{Feature vector extraction}\label{sec:feat_vector} 

There are many ways to parameterize $\vc$ in order to reduce its dimension. One of the prominent dimension reduction methods is Principal Component Analysis (PCA). PCA finds a new orthogonal basis for high-dimensional data. This enables projection of the data to lower dimension with the minimal sum of squared residuals. It was used in image processing \citep{zhang2010two} and classification \citep{zeng2016color}. In visual servoing, the method was first introduced in \citep{nayar1996subspace}. PCA is proven to be an effective, yet easy to implement, algorithm for dimension reduction. By projecting to the new orthogonal space, each feature component is linearly independent. Besides, by checking the explained variance of a feature, we can intuitively measure if it represents the original shape. 


We apply PCA to reduce $\vc \in \mathbb{R}^{2K}$ to $\vs \in \mathbb{R}^k$. To find the projection, we collect $M$ images with different shapes of the object and construct the data matrix $\vGamma = [\vc_1~\vc_2~\cdots~\vc_M] \in \mathbb{R}^{2K \times M}$.
Then, we shift the columns of $\vGamma$ by the mean contour $\bar{\vc} = \sum_{i = 1}^M \vc_i / M$:
\begin{equation}\label{eq:PCA_normalized}
    \bar{\vGamma} = [\vc_1 - \bar{\vc}~~~\vc_2 - \bar{\vc}~~~\cdots~\vc_M - \bar{\vc}] \in \mathbb{R}^{2K \times M}.
\end{equation}
We then compute the covariance matrix $\vC = \bar{\vGamma} \bar{\vGamma}^T$, and apply Singular Value Decomposition (SVD) to it:
\begin{equation}\label{eq:svd}
    \vC = \vU \vSigma \vV^T.
\end{equation} 
Once we have obtained the eigenvector matrix $\vU \in \mathbb{R}^{2K \times 2K}$, we can move on to select the first $k$ columns\footnote{In the SVD algorithm, the first $k$ columns correspond to the $k$ largest eigenvalues of matrix $\Sigma$.} of $\vU$ denoted by $\vU(k) \in \mathbb{R}^{2K \times k}$. Then, the $2K$-dimensional contour $\vc$ can be projected into a smaller $k$-dimensional feature vector $\vs$ as:
\begin{equation}\label{eq:k_feature_vec}
    \vs = \vU^T(k) (\vc - \bar{\vc}) \in \mathbb{R}^k.
\end{equation}

To assess the quality of this projection, we can compute the \textit{explained variance} using the eigenvalue matrix $\vSigma \in \mathbb{R}^{2K \times 2K}$ in (\ref{eq:svd}). By denoting the diagonal entries of $\vSigma$ as $\sigma_1, \cdots, \sigma_{2K}$, the explained variance of the first $k$ components is:
\begin{equation}\label{eq:explained_var}
    \Upsilon(k) = \frac{\sum^k_{j = 1} \sigma_j}{\sum^{2K}_{j = 1} \sigma_j}.
\end{equation} 
where $\Upsilon$ is a scalar between $0$ and $1$ (since $\sigma_j > 0$, $\forall j$), indicating to what extent the $k$ components represent the original data (a larger $\Upsilon$ suggests a better representation). 

Since PCA calculate features that lie on an orthonormal basis, these features are linearly independent. For controlling $n$ DoF, at least the same number of independent visual features should be used. Therefore, we set $k = n$ features.

\subsection{Local target generation}\label{sec:local_target}

Let us now explain how we generate a local target contour $\vc_i^*$ given a current contour $\vc_i$ and final target contour $\vc^*$. We also show this in Algorithm \ref{alg:localtarget}.
The overall shape error is given by:
\begin{equation}
    \vc_{\ve} = \vc^* - \vc_i.
\end{equation}
We define the intermediate target contour as:
\begin{equation}
    \vc_i^* = \vc_i + \frac{1}{\eta}\vc_{\ve}, 
\end{equation}
with $\eta = 1,2,\ldots$ an integer that ensures that $\vc_i^*$ is a ``good'' local target for $\vc_i$ (i.e., the two are similar). Therefore, if we project the intermediate local data using the eigenvector matrix at the current iteration, $\vU_i \in \mathbb{R}^{2K \times 2K}$ (note that we are using the full projection matrix and not just the first $k$ columns), the projection $\vs^p_i = \vU_i (\vc_i^* - \bar{\vc}) \in \mathbb{R}^{2K}$ should fulfil:
\begin{equation}\label{eq:local_target_condition}
    \Psi \left( k \right) = \frac{\sum^k_{j = 1} \left|s_{i,j}^p\right|}{\sum^{2K}_{j = 1} \left|s_{i,j}^p\right|} \geq \epsilon,
\end{equation}
with $\epsilon \in \left[ 0; 1 \right]$ a threshold and $s_{i,j}^p$ the $j$-th component of the projection. Then, we select the first $k$ components in $\vs^p_i$ to be the local target $\vs^*_i \in \mathbb{R}^k$. 

Algorithm~\ref{alg:localtarget} outlines the steps for computing the local intermediate targets, so that:
\begin{itemize}
    \item they are near the final target,
    \item the corresponding feature vector can be extracted with the current learned projection matrix. 
\end{itemize}
\begin{remark}
The reachability of a local target can only be verified with a global deformation model which we want to avoid identifying in our methods. We will further discuss this issue in the Conclusion (Sect. \ref{sec:conclusion}).
\end{remark}

\begin{algorithm}
    \caption{Local target generation}\label{alg:localtarget}
    \begin{algorithmic}
        \State localTargetFound = \textbf{false}
        \State $\Psi_0 = 0$
        \State $\eta = 1$
        \While{\textbf{not} localTargetFound}
        \State $\vc_i^* = \vc_i + \frac{1}{\eta}(\vc^* - \vc_i)$
        \State $\vs_i^p = \vU_i \vc_i^*$
        \State $\Psi_\eta = {\sum^k_{j = 1} \left| s_{i,j}^p \right|} / {\sum^{2K}_{j = 1} \left| s_{i,j}^p\right|}$
        \If{$\Psi_\eta \geq \epsilon$ \textbf{or} $\Psi_\eta < \Psi_{\eta -1}$}
        \State localTargetFound = \textbf{true}
        \State $\vs^*_i = [\vI~ \vnull] \vs_i^p$
        \EndIf
        \State $\eta = \eta + 1$
        \EndWhile
    \end{algorithmic}
\end{algorithm}

\subsection{Interaction matrix estimation}\label{sec:local_model_est}

Let us consider the current contour $\vc_i$ and the local target $\vc^*_i$. In this section, we show how we can implement the PCA and model estimation together and online.
We denote the robot motions and corresponding object contours over the last $M$ iterations (prior to iteration $i$, with $i \geq M$) as:
\begin{align}
    \begin{split}
    \vDelta \vR_i & = \begin{bmatrix}
        \delta \vr_{i-M+1}\; \delta \vr_{i-M+2} \cdots \delta \vr_{i}     
    \end{bmatrix} \in \mathbb{R}^{n \times M}\\
            \vGamma_i & =
    \begin{bmatrix}
        \vc_{i-M} \; \vc_{i-M+1} \; \vc_{i-M+2} \cdots \vc_{i}     
    \end{bmatrix} \in \mathbb{R}^{2K \times (M+1)},
    \end{split}
\end{align}
with $M$ the number of data samples collected during initialization, i.e., the size of the sliding window used for model adaptation (see Sect. \ref{sec:m_target}).

By selecting $k = n$ (note that $n$ is also the number of DoFs of the robot manipulator we considered in the task execution), we compute the projection matrix $\vU_{i}(n) \in \mathbb{R}^{2K \times n}$, from $\vGamma_i$ and $\bar{\vc}_i$ via~(\ref{eq:PCA_normalized}) and (\ref{eq:svd}). Then, using $\vU_{i}(n)$, we project current contour $\vc_i$, target contour $\vc^*_i$ and shape matrix $\vGamma_i$:
\begin{align}\label{eq:s_S}
    \begin{split}
        \vs_i & = \vU_{i}(n)^T (\vc_i - \bar{\vc}_i) \in \mathbb{R}^n, \\
        \vs^*_i & = \vU_{i}(n)^T (\vc^*_i - \bar{\vc}_i) \in \mathbb{R}^n, \\
        \vS_i & = \vU_{i}(n)^T \bar{\vGamma}_i = \begin{bmatrix}
        \vs_{i-M} \; \vs_{i-M+1} \; \cdots \; \vs_{i}     
    \end{bmatrix} \in \mathbb{R}^{n \times (M+1)}.
    \end{split}
\end{align}
In~(\ref{eq:s_S}), $\bar{\vGamma}_i$ is normalized by $\bar{\vc}_i$ as in (\ref{eq:PCA_normalized}). We can then compute $\vDelta \vS_i$ from (\ref{eq:differentials}) and (\ref{eq:s_S}), by subtracting consecutive columns of $\vS_i$:
\begin{equation}
    \vDelta \vS_i = \begin{bmatrix}
\delta \vs_{i-M+1}\; \delta \vs_{i-M+2} \cdots \delta \vs_{i}
    \end{bmatrix} \in \mathbb{R}^{n \times M}.
\end{equation}
Using $\vDelta \vS_i \in \mathbb{R}^{n \times M}$ and $\vDelta \vR_i \in \mathbb{R}^{n \times M}$ we can now estimate the local interaction matrix $\vL_i \in \mathbb{R}^{n \times n}$ at iteration $i$. We assume that near this iteration, the system remains linear and time invariant: $\vL_i$ is constant. Using the local linear model (\ref{eq:linear_sys_A}), we can write the following:
\begin{equation}\label{eq:solve_A}
   \vDelta \vS_i =  \vL_i \vDelta \vR_i.
\end{equation}
Our goal then is to solve for $\vL_i$, given $\vDelta \vS_i$ and $\vDelta \vR_i$. Note that this is an overdetermined linear system (with $n \times M$ equations for $n^2$ unknowns). 
Let us consider $\vDelta \vR_i \in \mathbb{R}^{n \times M}$ has full row rank. Note this sufficiently implies $M \geq n$. With this prerequisite, $\operatorname{rank}(\vDelta \vR_i) = n$. Therefore, $\operatorname{rank}(\vDelta \vR_i \vDelta \vR_i^T) = n$, and its inverse exists. We post multiply (\ref{eq:solve_A}) by $\Delta \vR_i^T$:
\begin{equation}\label{eq:solve_A_step_1}
    \vDelta \vS_i \vR_i^T =  \vL_i \vDelta \vR_i \vDelta \vR_i^T.
\end{equation}
Then, since $\vDelta \vR_i \vDelta \vR_i^T$ is invertible, the $\vL_i$ that best fulfills~(\ref{eq:solve_A}) is:
\begin{equation}\label{eq:A_M_geq_n}
    \hat{\vL}_i = \vDelta \vS_i \vDelta \vR_i^T (\vDelta \vR_i \vDelta \vR_i^T)^{-1}.
\end{equation}
If, in practice, the full row rank condition of $\vDelta \vR_i$ is not satisfied, $\operatorname{rank}(\vDelta \vR_i \vDelta \vR_i^T) < n$ and $\vDelta \vR_i \vDelta \vR_i^T$ becomes singular. Then, instead of~(\ref{eq:A_M_geq_n}), we can use Tikhonov regularization:
\begin{equation}
\label{eq:intMatEst}
    \hat{\vL}_i = \vDelta \vS_i \vDelta \vR_i^T (\vDelta \vR_i \vDelta \vR_i^T + \lambda \vI)^{-1},
\end{equation}
with $\lambda$ an arbitrary (generally small) scalar.

Practically, this implies that one or more inputs motions do not appear in $\vDelta \vR_i$. Therefore, we cannot infer the relationship between these motions and the resulting feature vector changes. In this case it is better to increase $M$ and obtain more data, so that $\vDelta \vR_i$ has full row rank.

Instead of computing the interaction matrix, it is also possible to directly compute its inverse, since this guarantees better control properties~\citep{lapreste2004efficient}. 
With the same data, one can re-write (\ref{eq:solve_A}) as:
\begin{equation}\label{eq:solve_inverse}
    \vL_i^\oplus \vDelta \vS_i = \vDelta \vR_i.
\end{equation}
We can also solve (\ref{eq:solve_inverse}) with Tikhonov regularization:
\begin{equation}\label{eq:intMatInvEst}
    \hat{\vL}_i^\oplus = \vDelta \vR_i \vDelta \vS_i^T( \vDelta\vS_i \vDelta\vS_i^T + \lambda\vI)^{-1}.
\end{equation}
\subsection{Control law and stability analysis}\label{sec:crtl}

We can now control the robot, with either of the following strategies:
\begin{equation}\label{eq:one-step-ctrl}
    \delta \vr_i = - \alpha \hat{\vL}_i^{\dagger}(\vs_{i} - \vs^*_i),
\end{equation}
if one estimates the interaction matrix with~(\ref{eq:intMatEst}), where $^{\dagger}$ denotes the pseudo-inverse, or:
\begin{equation}\label{eq:one-step-ctrl-inv}
\delta \vr_i = - \alpha \hat{\vL}_i^\oplus (\vs_{i} - \vs^*_i)
\end{equation}
if one estimates the inverse of the interaction matrix with~(\ref{eq:intMatInvEst}). In both equations, $\alpha > 0$ is an arbitrary control gain.
\begin{proposition}\label{prop:2}
    Consider that locally, the model (\ref{eq:linear_sys_A}) closely approximates the interaction matrix $\vL_i = \hat{\vL}_i$. For $M$ number of linearly independent displacement vectors $\delta \vr$ such that the interaction matrix $\hat{\vL}_i$ is invertible, the update rule (\ref{eq:one-step-ctrl}) asymptotically minimizes the error $\ve_i = \vs_i^* - \vs_i$, where $\vs_i^*$ is the local target.
\end{proposition}
\begin{proof}
    With $\delta \vs_i = \vs_{i + 1} - \vs_{i}$, we can write (\ref{eq:linear_sys_A}) in discretized form as
    \begin{equation}\label{eq:prop2_1}
        \vs_{i + 1} = \vs_{i} + \vL_i \delta \vr_i.
    \end{equation}
    From the definition of $\ve_i$ we have (Note here the target $\vs_i^*$ is not updated with $i$ since we want to prove local convergence to a constant target):
    \begin{align}\label{eq:prop2_2}
        \begin{split}
            \ve_{i} &= \vs_i^* - \vs_{i} \\ 
            \ve_{i + 1} &= \vs_i^* - \vs_{i + 1} 
        \end{split}
    \end{align}
    Taking (\ref{eq:prop2_1}) into (\ref{eq:prop2_2}):
    \begin{align}\label{eq:prop2_3}
        \begin{split}
            \ve_{i + 1} & = \vs_{i}^* - \vs_{i + 1} \\
                     & = \vs_i^* - \vs_{i} - \vL_i \delta \vr_i \\
                     & = \ve_{i} - \vL_i \delta \vr_i.
        \end{split}
    \end{align}
    We replace $\delta \vr_i$ in (\ref{eq:prop2_3}) with the control (\ref{eq:one-step-ctrl}), the error dynamic is then:
    \begin{align}\label{eq:prop2_4}
        \begin{split}
            \ve_{i+1} & =  \ve_{i} - \alpha \vL_i \vL_i^{-1}(\vs_i^* - \vs_{i}) \\
                        & = \ve_i - \alpha \ve_i = (1 - \alpha) \ve_i.
        \end{split}
    \end{align}
    is asymptotically stable for $\alpha \in \left[ 0 ; 1 \right]$. This can be proved by considering the Lyapunov function
    \begin{equation}
        \mathcal{V}(\ve) = \ve^T \ve.
    \end{equation}
    Using the error dynamic (\ref{eq:prop2_4}), one can derive:
    \begin{align}
        \begin{split}
            \Delta \mathcal{V}  & = \mathcal{V}(\ve_{i + 1}) - \mathcal{V} (\ve_i) \\
                                & = \ve_i^T ((1 - \alpha)^2 - 1) \ve_i < 0.
        \end{split}
    \end{align}
    This proves the local asymptotic stability of the error $\ve$ using our inputs. 
\end{proof}

\subsection{Model adaptation}\label{sec:m_target}

Since both the projection matrix $\vU^T(n)$ and the interaction matrix are local approximations of the full nonlinear mapping, they need to be updated constantly. We choose a receding window approach with window size $M$.

At current iteration $i$, we estimate the projection matrix $\vU_i^T$ and local interaction matrix $\vL_i$ with $M$ samples of the most recent data. Using the interaction matrix and the local target $\vc^*_{i}$,  we can derive the one-step command $\delta \vr_i$ by (\ref{eq:one-step-ctrl}). Once we execute the motion $\delta \vr_i$, a new contour data $\vc_{i + 1}$ is obtained. We move to the next iteration $i + 1$. A new pair of input and shape data $[\delta \vr_i, \vc_{i + 1}]$ is obtained. We shift the window by deleting the oldest data in the window and add in the new data pair. Then, using the shifted window, we compute one step control at iteration $i + 1$.  

The receding window approach ensures that, at each iteration, we are using the latest data to estimate the interaction matrix. 
%
The overall algorithm is initialized with small random motions around the initial configuration. First, $M$ samples of shape data and the corresponding robot motions are collected. With this initialization, we can simultaneously solve for the projection matrix and estimate the initial interaction matrix using the methods described in Sect.~\ref{sec:feat_vector} and \ref{sec:local_model_est}. Using the projection matrix and the initial/target shapes, we can then find an intermediate target (see Sect.~\ref{sec:local_target}).

We consider quasi-static deformation. Hence, at each iteration the system is in equilibrium and can be linearized according to~(\ref{eq:linear_sys_A}). The data that best captures the current system are the most recent ones. The choice of $M$ is a trade-off between locality and richness. For fast varying deformations\footnote{The notion of fast or slow varying depends on both the speed of manipulation, and on the objects deformation characteristics (which affect the rate of change in shapes) with regard to the image processing time.}, we would expect to reduce $M$ since a larger $M$ will hinder the locality assumption. Yet, if $M$ is too small, it affects the estimation of $\hat{\vL}_i$ (refer to the detailed discussion in Sect.~\ref{sec:local_model_est}).

\section{Simulation results}\label{sec:simu}

In this section, we present the numerical simulations that we ran to validate our method.

\subsection{Simulating the objects}
We ran simulations on MATLAB (R2018b) with two types of objects: a rigid box and a deformable cable, both constrained to move on a plane. The rigid object is represented by a uniformly sampled rectangular contour. The controllable inputs are its position and orientation. For the cable, we developed a simulator, which is publicly available at \href{https://github.com/Jihong-Zhu/cableModelling2D}{https://github.com/Jihong-Zhu/cableModelling2D}. The simulator relies on the differential geometry cable model introduced in~\citep{wakamatsu2004static}, with the shape defined by solving a constrained optimization problem. The underlying principle is that the object's potential energy is minimal for the object's static shape \citep{wakamatsu1995modeling}. Position and orientation constraints (imposed at the cable ends) are input to the simulator. The output is the sampled cable. Figures \ref{fig:test_k} -- \ref{fig:sim_singlearm_DLO},~\ref{fig:sim_singlearm_rigid},~\ref{fig:rigid_analysis_random_move} show simulated shapes of cables and rigid boxes. We choose $K = 50$ samples for both rigid objects and cables. The camera perspective projection is simulated, with optical axis perpendicular to the plane.

\begin{figure}[!thpb]
	\vspace{-0.6cm}
	\centering
	\subfloat{\includegraphics[width = 0.24\textwidth]{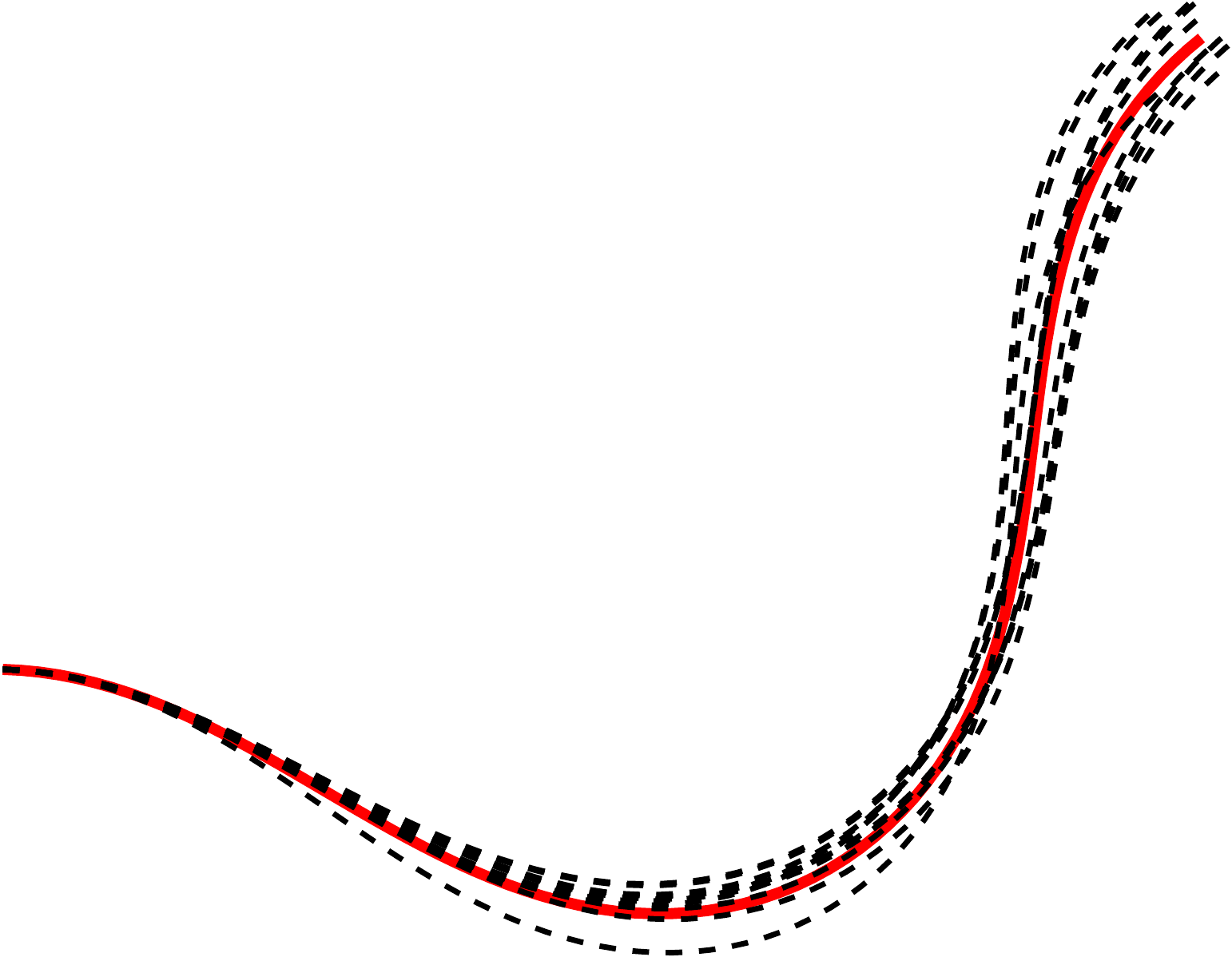}\label{fig:test_k_1}}
	\hspace{5mm}
    \subfloat{\includegraphics[width = 0.2\textwidth]{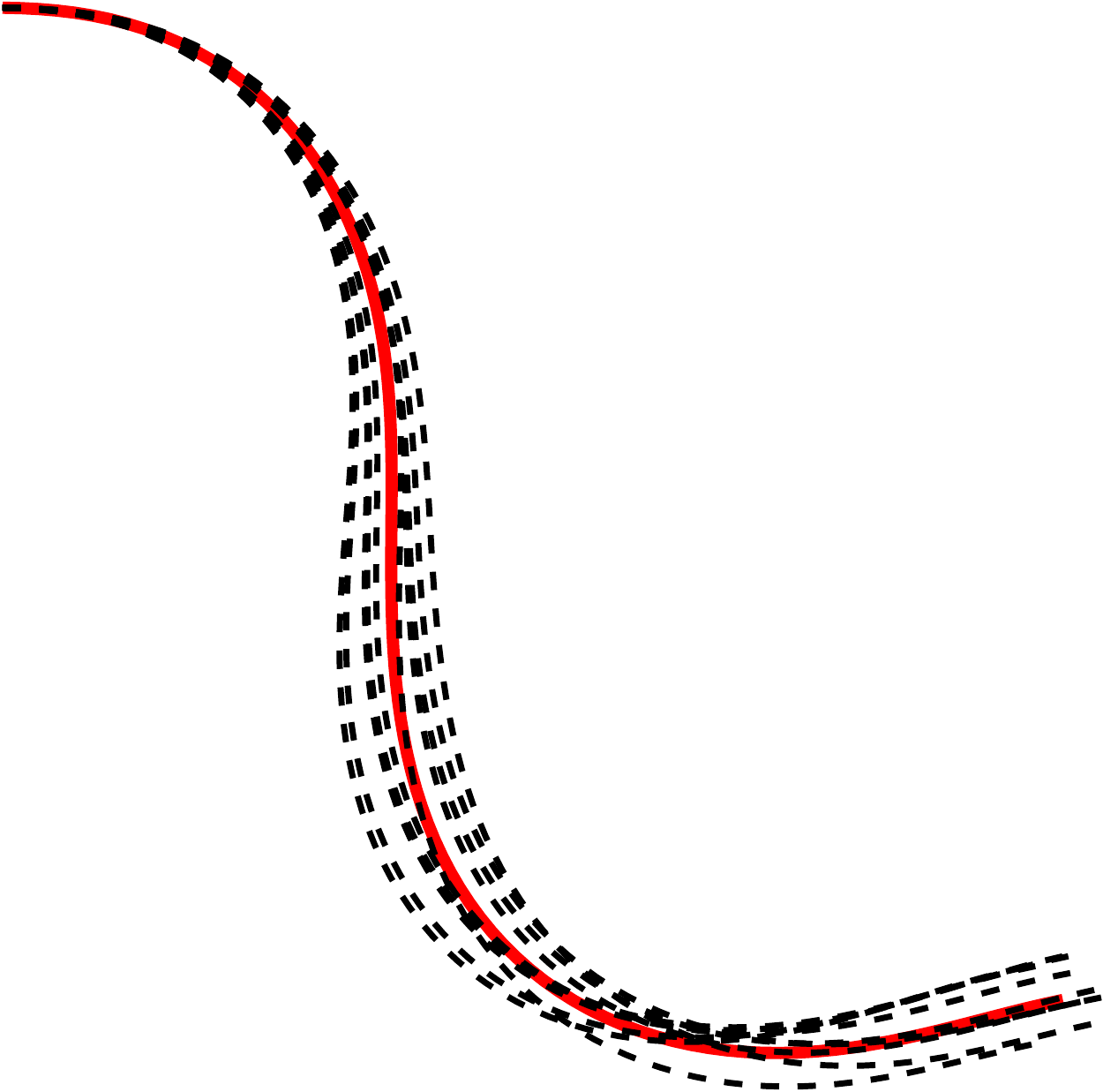}\label{fig:test_k_2}}
    \hspace{5mm}
    \subfloat{\includegraphics[width = 0.16\textwidth]{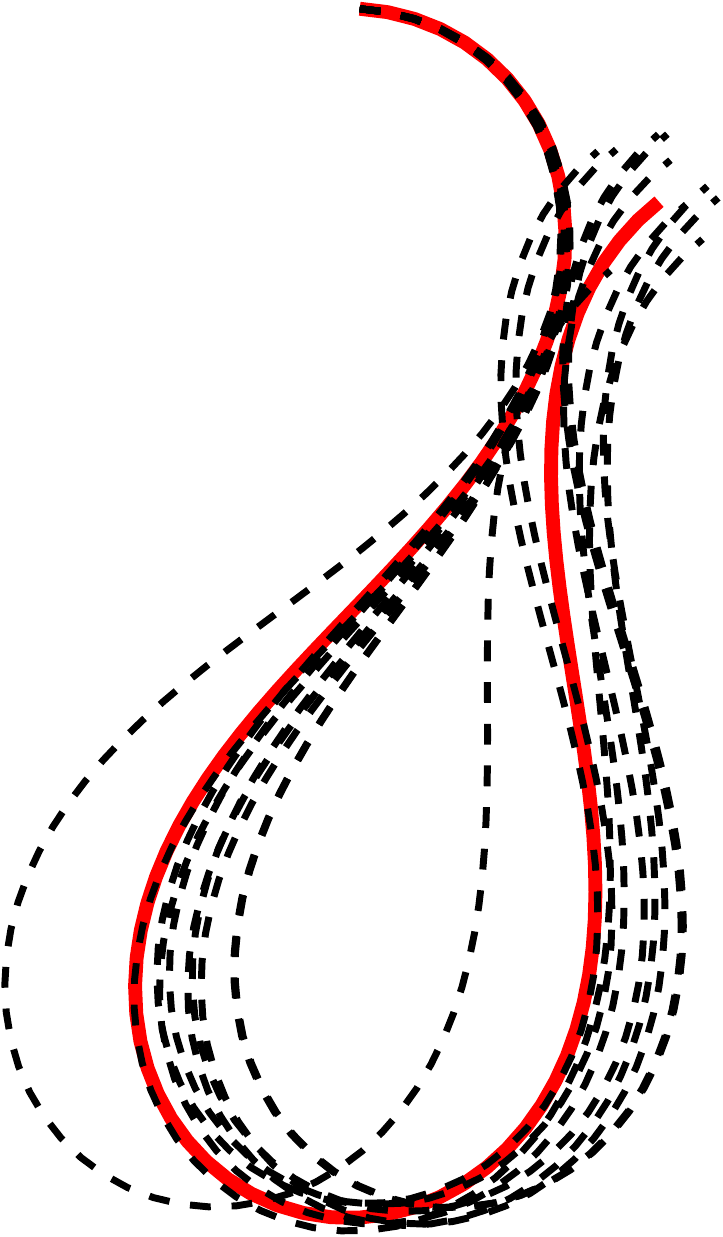}\label{fig:test_k_3}} \\
    \subfloat{\includegraphics[width = 0.24\textwidth]{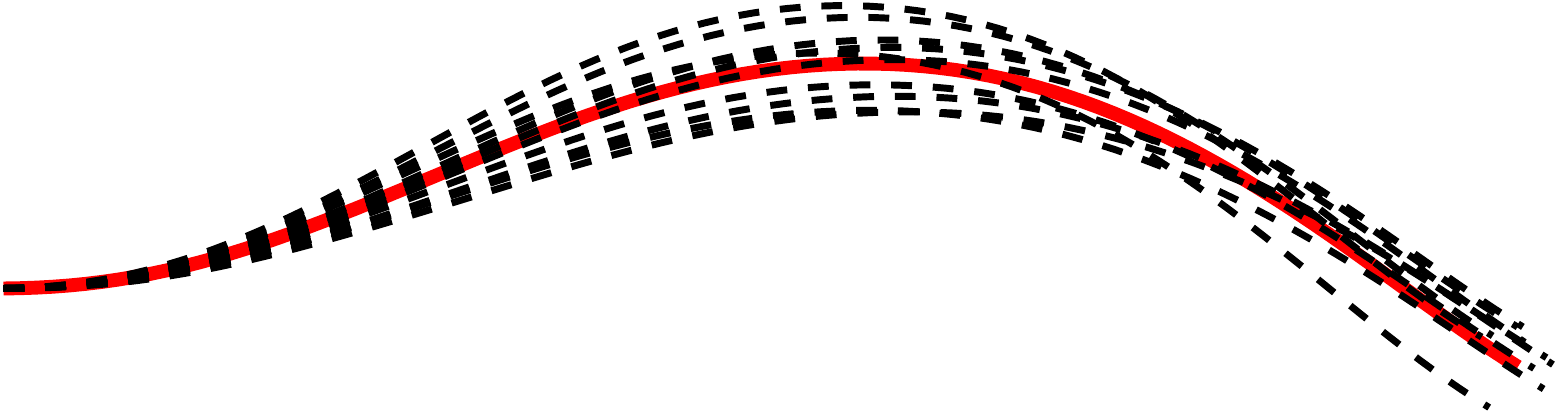}\label{fig:test_k_4}}
	\hspace{0.5mm}
    \subfloat{\includegraphics[width = 0.24\textwidth]{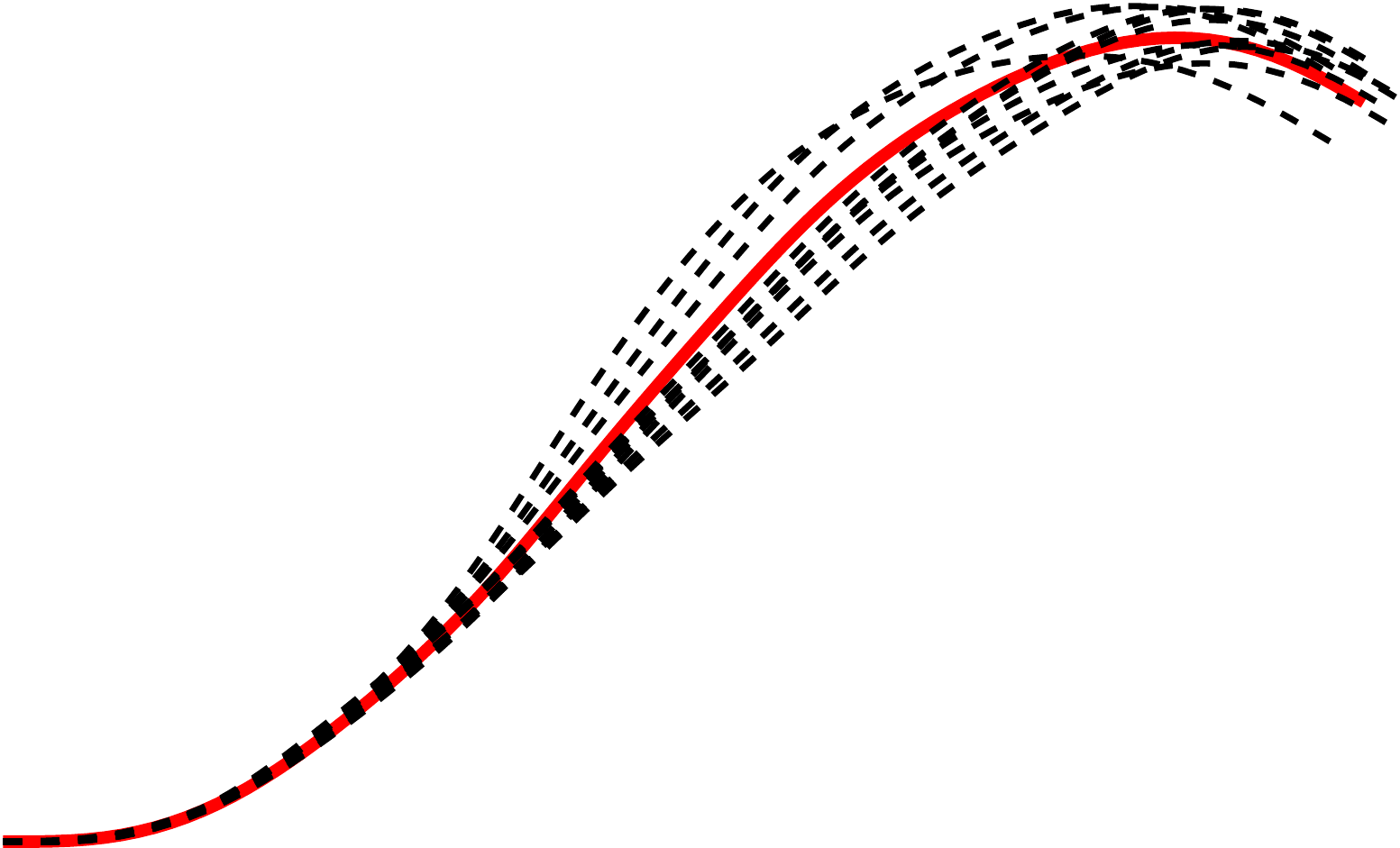}\label{fig:test_k_5}}
    \hspace{0.5mm}
    \subfloat{\includegraphics[width = 0.24\textwidth]{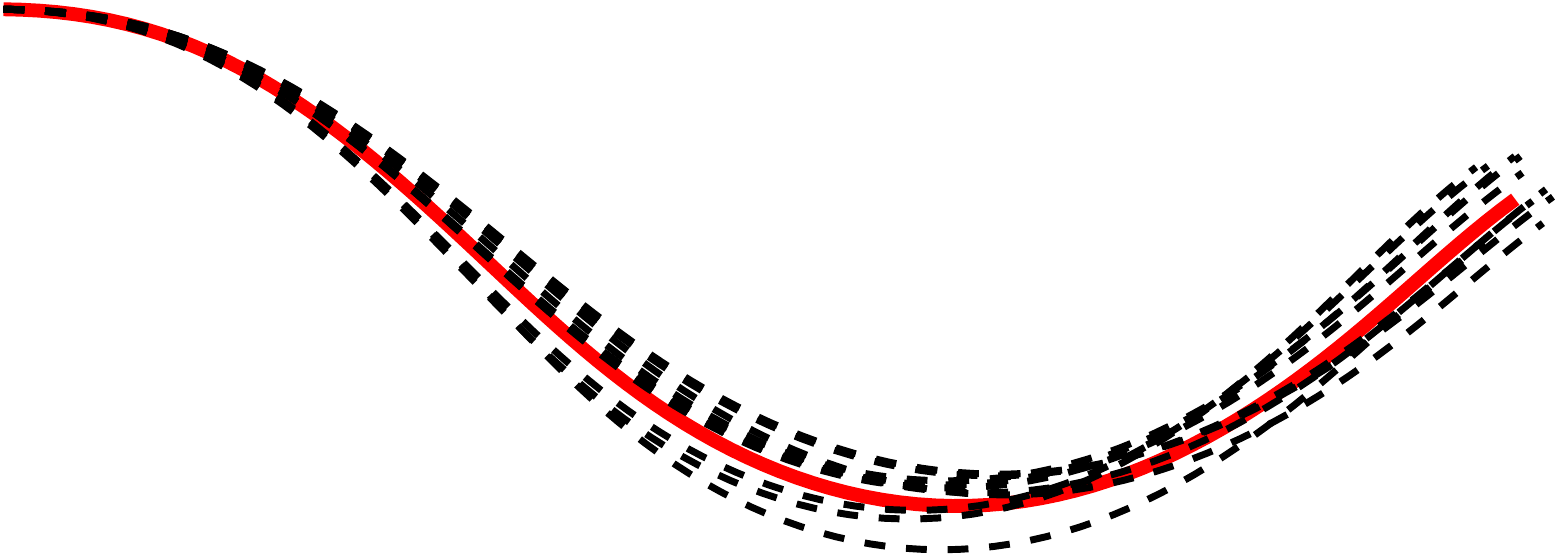}\label{fig:test_k_6}}
    \caption{Six trials conducted to test various choices of feature dimension $k$ for a cable. In each sub-figure, the solid red lines are the initial shapes and the dashed black are the shapes resulting from $10$ random motions of the right tip (translations limited to $\pm 5 \%$ of the length, rotations limited to $\pm 5 \degree$).}
    \label{fig:test_k}
    \vspace{-0.5cm}
\end{figure}

\subsection{Selecting the feature dimension $k$}\label{sec:select_feature_dimension}

To check whether choosing $k = n$ can represent the shape accurately, we simulate $6$ trials with distinct initial shapes of a cable. The dimension of the robot motion vector $\delta \vr$ is $n = 3$ (two translations and one rotation of the right tip), and the motions are limited: each translation to $\pm 5 \%$ of the cable length and the rotation to $\pm 5 \degree$. This range of motion gives a rule of thumb, which we have used for generating random movements throughout our experiments. For each trial, we command $M = 10$ random motions around the initial shape using our simulator. Figure \ref{fig:test_k} shows the $6$ initial cable shapes (solid red) and the resulting shapes from $10$ random movements (dashed black).

For each trial, we apply PCA to map the cable contour $\vc \in \mathbb{R}^{2K}$ to feature vector $\vs \in \mathbb{R}^k$, as explained in Sect.~\ref{sec:feat_vector}. We do this for $k = 1$, $2$ and $3$ and for each of these $18$ experiments, we calculate the explained variance $\Upsilon(k)$ with (\ref{eq:explained_var}). Table \ref{tab:Explained_var} shows these explained variances. In all $6$ trials, $k = n = 3$ yields explained variances very close to $1$. This result confirms that choosing $k = n$ as the dimension of the feature vector gives an excellent representation of the shape data. It is also possible to select $k = 2$, since the first two components can represent more than $99 \%$ of the variance. Nevertheless, the simulation is noise-free. Therefore, although $\Upsilon(k)$ increases little from $k = 2$ to $k = 3$, this increase is not related to noise but to an actual gain in data information. 
\begin{table}[thpb]
	\vspace{-0.2cm}
    \centering
    \caption{Explained variance $\Upsilon(k)$ for the $6$ trials with small motion.}
    \begin{tabular}{c|c|c|c|c|c|c}
        \hline
          & trial 1 & trial 2 & trial 3 & trial 4 & trial 5 & trial 6 \\
          \hline
        $k = 1$ & 0.727 & 0.795 & 0.871 & 0.847 & 0.847 & 0.705 \\
        \hline
        $k = 2$ & 0.992 & 0.995 & 0.996 & 0.997 & 0.997 & 0.994 \\
        \hline
        $k = 3$ & 0.999 & 0.999 & 0.999 & 0.999 & 0.999 & 0.999 \\
        \hline
    \end{tabular}
    \vspace{-0.3cm}
    \label{tab:Explained_var}
\end{table}

At this stage, it is legitimate to ask: \emph{how does this scale to larger movements}? Figure \ref{fig:large_movement} illustrates $10$ cable shapes generated by large movements (angle variation: $[-\frac{\pi}{2}, \frac{\pi}{2}]$, maximum translation: $106 \%$). Again, we apply PCA ($M = 10$); Table \ref{tab:Explained_var_large} shows the $\Upsilon(k)$ resulting from various values of $k$. 

\begin{table}[thpb]
	\vspace{-0.2cm}
    \centering
    \caption{Explained variance $\Upsilon(k)$ computed with large motion.}
    \begin{tabular}{c|c|c|c|c|c|c}
        \hline
          k & $0$ & $1$  & $2$  & $3$  & $4$ & $5$ \\
          \hline
        $\Upsilon(k)$ & 0 & 0.5444 & 0.7218 & 0.8927 & 0.9919 & 0.9990 \\
        \hline
    \end{tabular}
    \label{tab:Explained_var_large}
\end{table}

\begin{figure}[!thpb]
    \centering
    \includegraphics[width=0.3\textwidth]{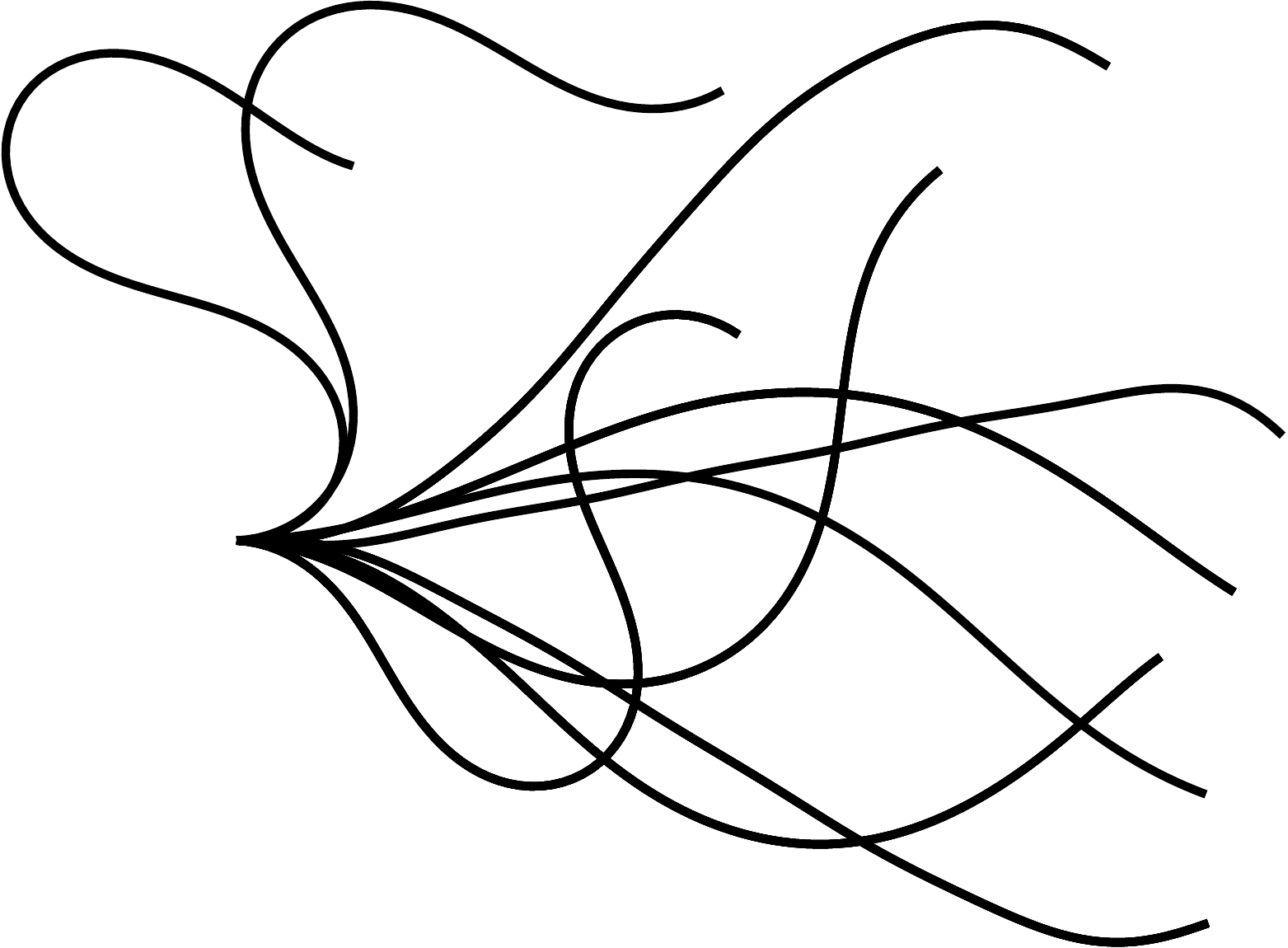}
    \caption{Ten distinctive cable shapes generated by large motion: angle variation: $[-\frac{\pi}{2}, \frac{\pi}{2}]$, maximum translation: $106 \%$ of the cable length.}
    \label{fig:large_movement}
    \vspace{-0.2cm}
\end{figure}

Comparing Tables~\ref{tab:Explained_var} and~\ref{tab:Explained_var_large}, it is noteworthy that $\Upsilon(4)$ with large motion is smaller than $\Upsilon(2)$ with small motion. There are two possible explanation here. One is that when shapes stays local, the local linear mapping $\vL$ in (\ref{eq:linear_sys_A}) remains constant and we need less features to characterize it; the more the shape varies, the more features we need. Another possible explanation is that for larger motions, $M = 10$ shapes may be insufficient for PCA. Likely, the larger the changes, the larger the number of shapes $M$ needed.

\subsection{Manipulation of deformable objects}

With our cable simulator, we can now test the controller to modify the shape from an initial to a target one. Again, the left tip of the cable is fixed, and we control the right tip with $n=3$ degrees of freedom (two translations and one rotation). Using the methods described in Sect. \ref{sec:method}, we choose window size $M = 5$, the Tikhonov factor $\lambda = 0.01$, the local target threshold $\epsilon = 0.8$, the control gain $\alpha = 0.01$. To quantify the effectiveness of our algorithms in driving the contour to $\vc^*$, we define a scalar measure: the Average Sample Error (ASE). At iteration $i$, with current contour $\vc_i$ it is:
\begin{equation}\label{eq:ASE}
    \text{ASE} = \frac{\|\vc_i - \vc^*\|_2}{2K}.
\end{equation}
A small \text{ASE} indicates that the current contour is near the target one. In Sect.~\ref{sec:crtl}, we have proved that our controller asymptotically stabilizes the feature vector, $\vs$ to $\vs^*$. Hence, since we have also shown that $\vs$ is a ``very good'' representation'' of $\vc$, we also expect our controller to drive $\vc$ to $\vc^*$, thus \text{ASE} to $0$. This measure is also used in the real experiments.
 
Using the cable simulator, we compare the convergence of two control laws proposed in our paper (\ref{eq:one-step-ctrl}) and (\ref{eq:one-step-ctrl-inv}) against a baseline algorithm in \citep{zhu2018dual} which uses Fourier parameters as feature. To make methods compatible, we choose first order Fourier approximation. Note that this results in a feature vector of dimension of $6$ (see \citep{zhu2018dual}) which is still twice the number $k$ used in our method. We also normalize the computed control action and then multiply by the same gain factor $0.01$. 

We also introduced artificial noise to the contour data, to test the robustness of our method. For a unit length cable, we add Gaussian noise of zero mean and $0.01$ standard deviation to the contour sample points. Fig. \ref{fig:cable_sim3} shows that our algorithm converges in these conditions as well. It is worth mentioning that in robotic experiments, as shape data is obtained and sampled from the images, signal noise is inevitable. Yet, our framework is robust enough to still converge to the target shape/pose (see Sect. \ref{sec:exp} for detailed real robot experiments).

Figure \ref{fig:sim_singlearm_DLO} shows two other simulation results: on the left, a reachable target and on the right an unreachable one. In Fig. \ref{fig:cable_sim1}, the cable shape successfully evolves towards the target thanks to our controller~(\ref{eq:one-step-ctrl-inv}). Figure \ref{fig:cable_sim2} shows the starting shape (blue), unreachable target (dash black) and final shape (solid black) obtained using~(\ref{eq:one-step-ctrl-inv}). Note that the controller gets stuck in a local minimum.

Figure \ref{fig:sim_singlearm_DLO_ASE} compares the evolution of \text{ASE} with our methods against the Fourier-based method for the reachable target; in the same figure, we also plot the evolution of \text{ASE} using~(\ref{eq:one-step-ctrl-inv}) for the unreachable target. We can observe that our method provides faster convergence using half the features than \citep{zhu2018dual}. Also, directly computing the inverse (\ref{eq:one-step-ctrl-inv}) provides faster convergence than (\ref{eq:one-step-ctrl}). It is noteworthy to point out that the Fourier-based method requires a different parameterization for closed and open contours (see \citep{navarro2017fourier} and \citep{zhu2018dual}), whereas in our framework, the parameterization can be kept the same. Last but not least, our approach is the only one among the three, which has been validated on both rigid and deformable objects.

\begin{figure}[!thpb]
	\vspace{-0.2cm}
	\centering
	\subfloat[]{\includegraphics[width = 0.3\textwidth]{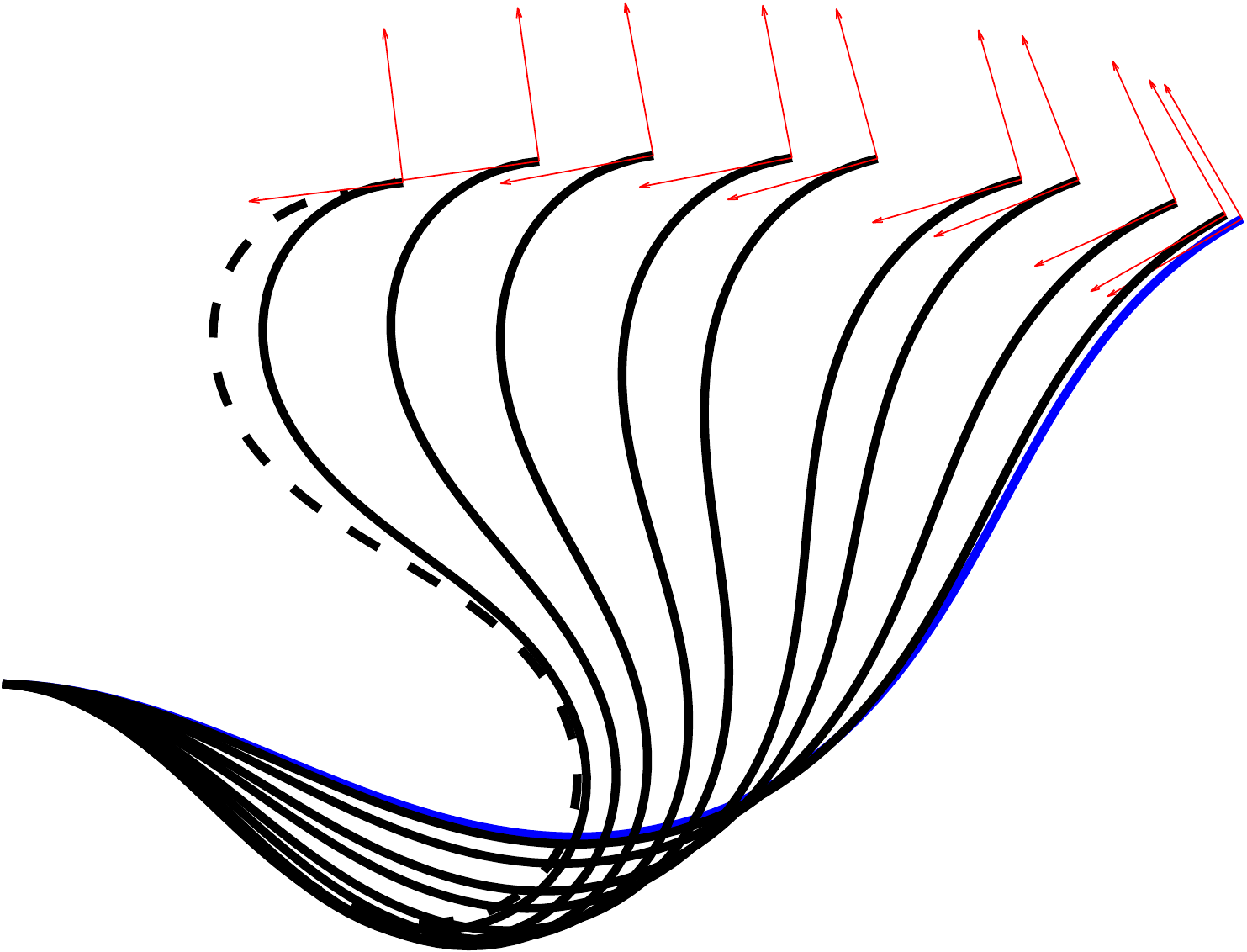}\label{fig:cable_sim1}}
    \subfloat[]{\includegraphics[width = 0.3\textwidth]{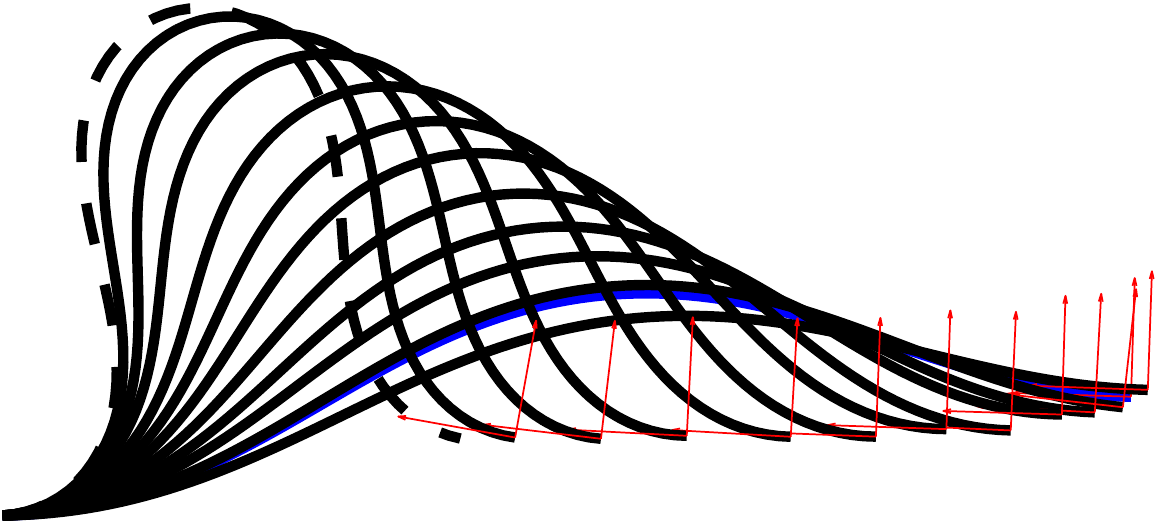}\label{fig:cable_sim3}}
	\subfloat[]{\includegraphics[width = 0.3\textwidth]{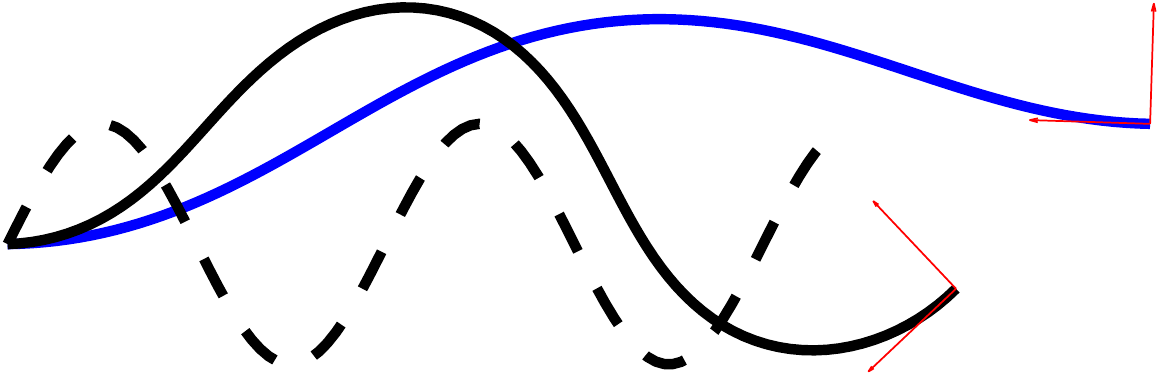}\label{fig:cable_sim2}}
	\vspace{-0.3cm}
	\caption{Cable manipulation with a single end-effector, moving the right tip. (a): a reachable target, the blue and black lines are the initial and intermediate shapes, respectively, and the dashed black line is the target shape. The red frame indicates the end-effector position and orientation generated by our controller. (b): Adding Gaussian noise with zero mean and $0.01$ standard deviation to the shape data with a reachable target. (c): An unreachable target and the final shape obtained with our controller.}
	\label{fig:sim_singlearm_DLO}
\end{figure}

\begin{figure}[!thpb]
	\centering
	\subfloat{\includegraphics[width = 0.8\textwidth]{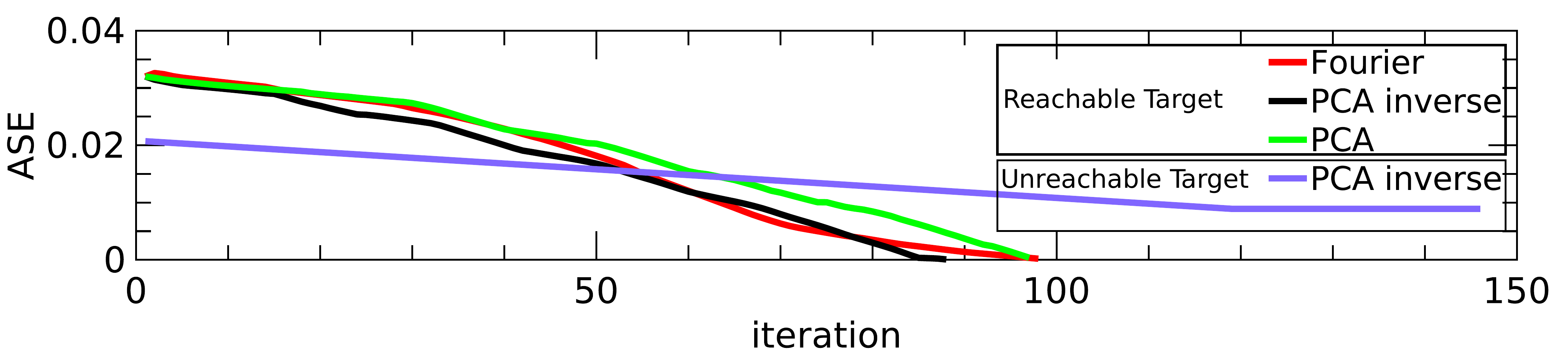}\label{fig:cable_sim_1}}
	\vspace{-0.3cm}
	\caption{The evolution of the ASE of the simulated cable manipulation using our method against the Fourier-based method as baseline and the ASE of the unreachable target with (\ref{eq:one-step-ctrl-inv}).}
	\label{fig:sim_singlearm_DLO_ASE}
\end{figure}

\subsection{Comparison with the Broyden update law}

The Broyden update law \citep{broyden1965class}, has been used to update the interaction matrix in classic visual servoing \citep{hosoda1994versatile,jagersand1997experimental,chaumette2007visual} and shape servoing~\citep{navarro2013model}. 

In this section, we compare it with our method for updating the interaction matrix (\ref{eq:intMatEst}), which relies on a receding horizon. We will hereby show why the Broyden update law is not applicable in our framework.

The Broyden update is an iterative method for estimating $\vL_i$ at iteration $i$. Its standard discrete-time formulation is:
\begin{equation}
    \hat{\vL}_{i} = \hat{\vL}_{i - 1} + \beta \frac{\delta \vs_{i - 1} - \hat{\vL}_i \delta \vr_{i - 1} }{\delta \vr_{i - 1}^T \delta \vr_{i - 1}} \delta \vr_{i - 1}^T,~\forall \vr_{i - 1} \neq \vnull
\end{equation}
with $\beta \in \left[ 0 ; 1 \right]$ an adjustable gain. 
Using our simulator, we estimate the interaction matrix using both Broyden update (with three different values of $\beta$) and our receding horizon method (\ref{eq:intMatEst}). We then compare (with $\hat{\vL}$ estimated with either method) the one-step prediction of the resulting feature vector:
\begin{equation}
    \hat{\vs}_{i + 1} = \hat{\vL}_i \vr_i + \vs_i,
\end{equation} 
with the ground truth ${\vs}_{i + 1}$ from the simulator. The results (plotted in Fig. \ref{fig:RH_Brod}) show that receding horizon outperforms all three Broyden trials. One possible reason is that the components of $\vs$ fluctuate since (at each iteration) a new matrix $\vU$ is used. These variations cause the Broyden method to accumulate the result from old interaction matrices, and therefore perform badly on a long term. This result contrasts with that of~\citep{navarro2013model}, where the Broyden method performs well since there is a fixed mapping from contour data to feature vector. Another advantage of the receding horizon approach is that it does not require any gain tuning.

\begin{figure}[!thpb]
	\centering
	\subfloat[Receding horizon $s_1$]{\includegraphics[width = 0.45\textwidth]{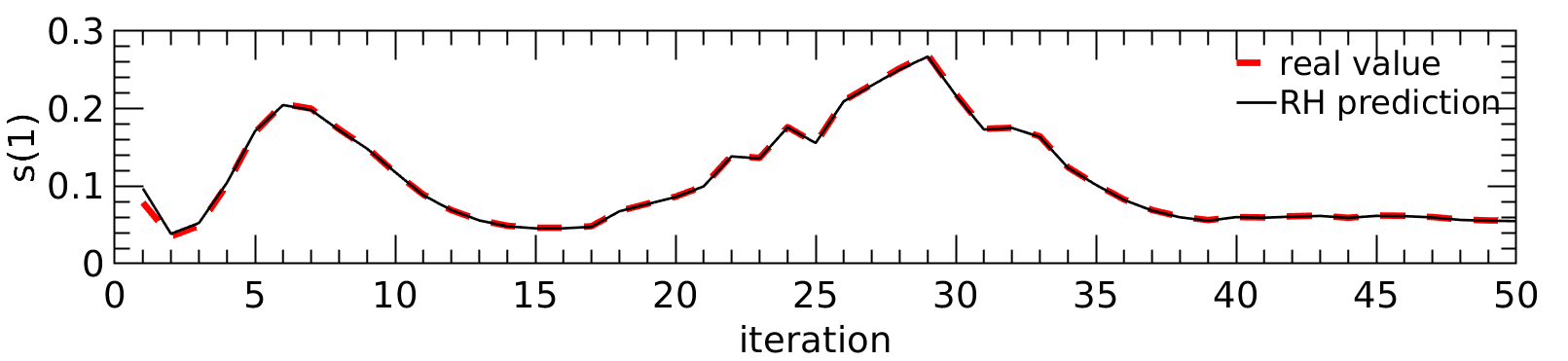}\label{fig:RH_1}}   \hspace{5mm}
    \subfloat[Broyden update $s_1$]{\includegraphics[width = 0.45\textwidth]{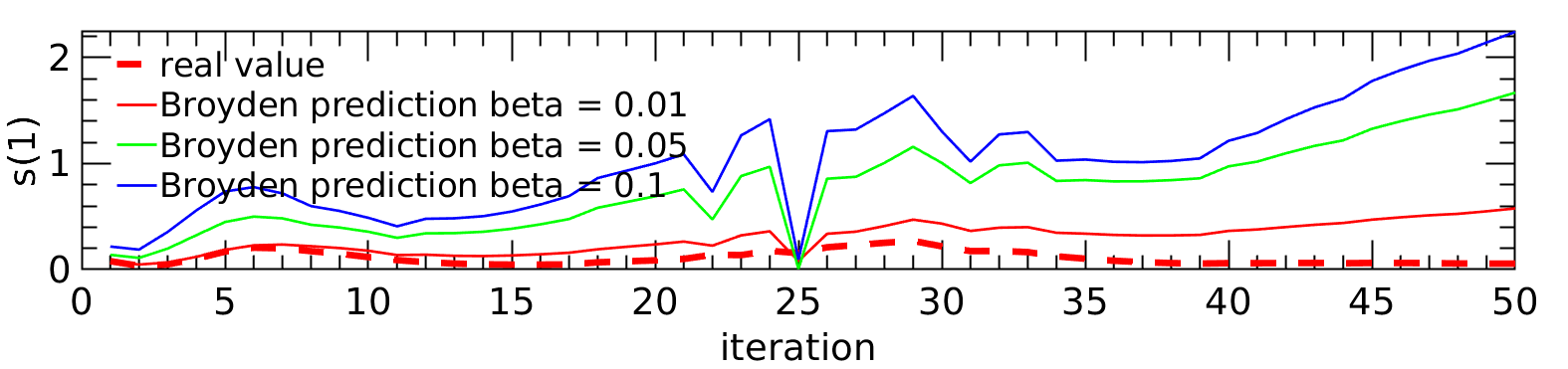}\label{fig:Brod_1}} \\ 
    \subfloat[Receding horizon $s_2$]{\includegraphics[width = 0.45\textwidth]{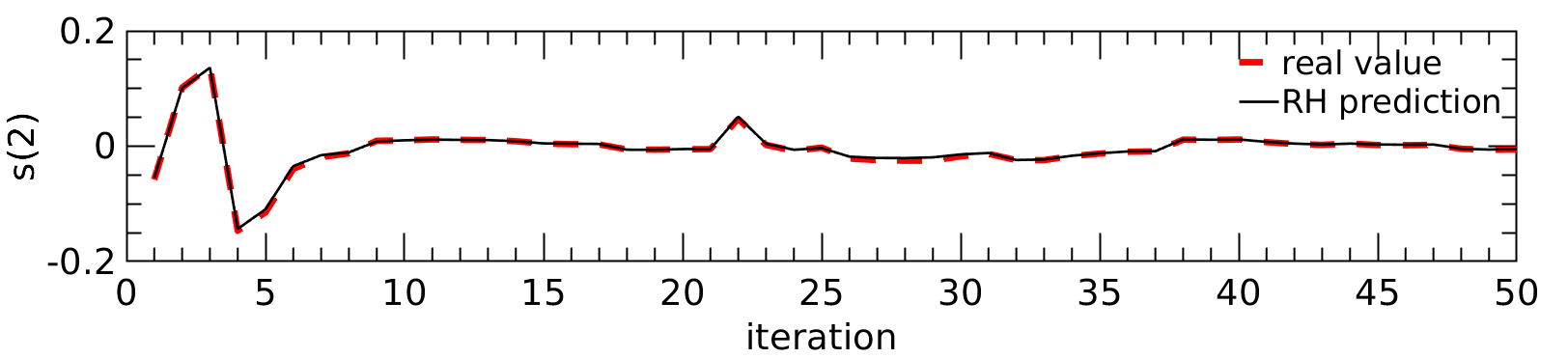}\label{fig:RH_1}}   \hspace{5mm}
    \subfloat[Broyden update $s_2$]{\includegraphics[width = 0.45\textwidth]{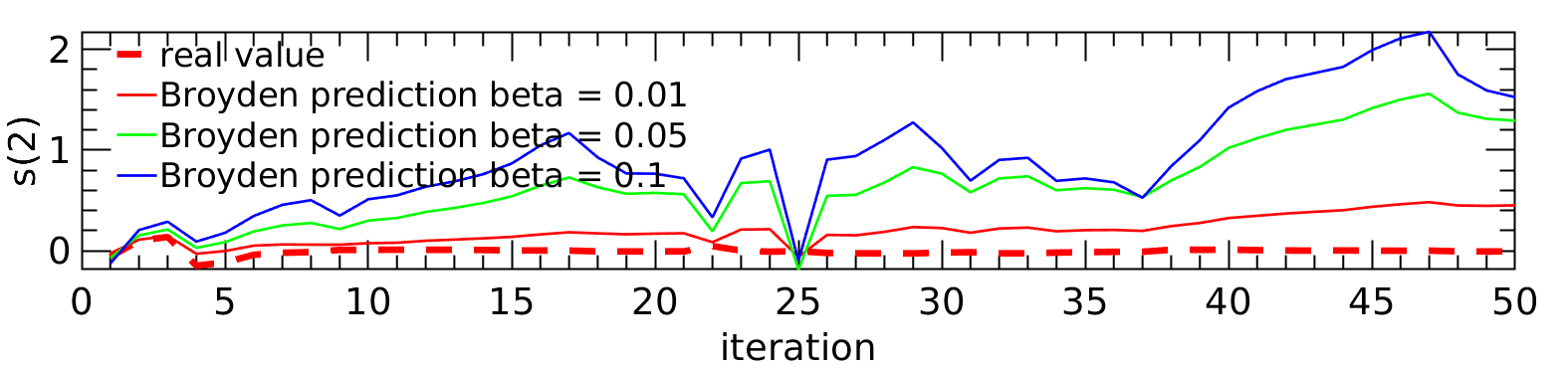}\label{fig:Brod_1}} \\ 
    \subfloat[Receding horizon $s_3$]{\includegraphics[width = 0.45\textwidth]{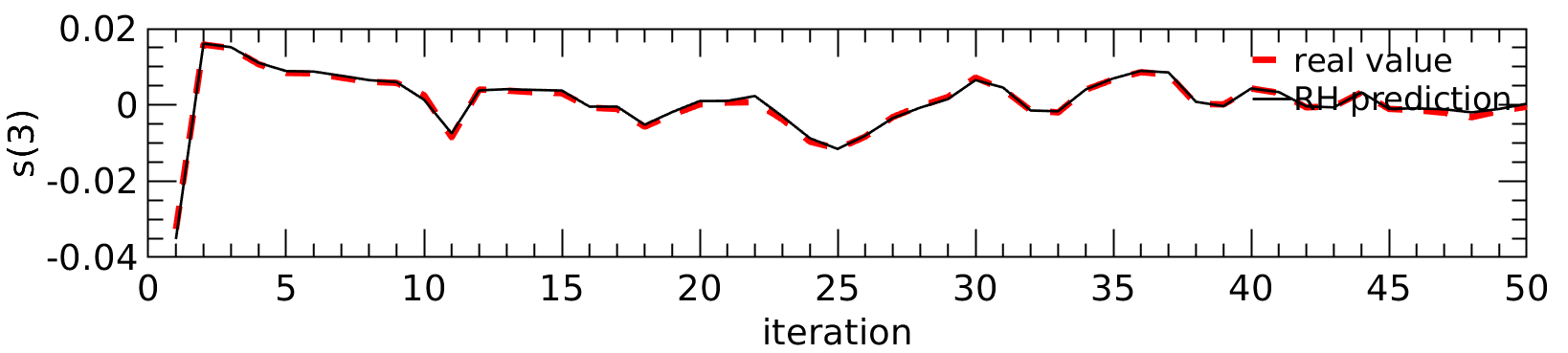}\label{fig:RH_1}}   \hspace{5mm}
    \subfloat[Broyden update $s_3$]{\includegraphics[width = 0.45\textwidth]{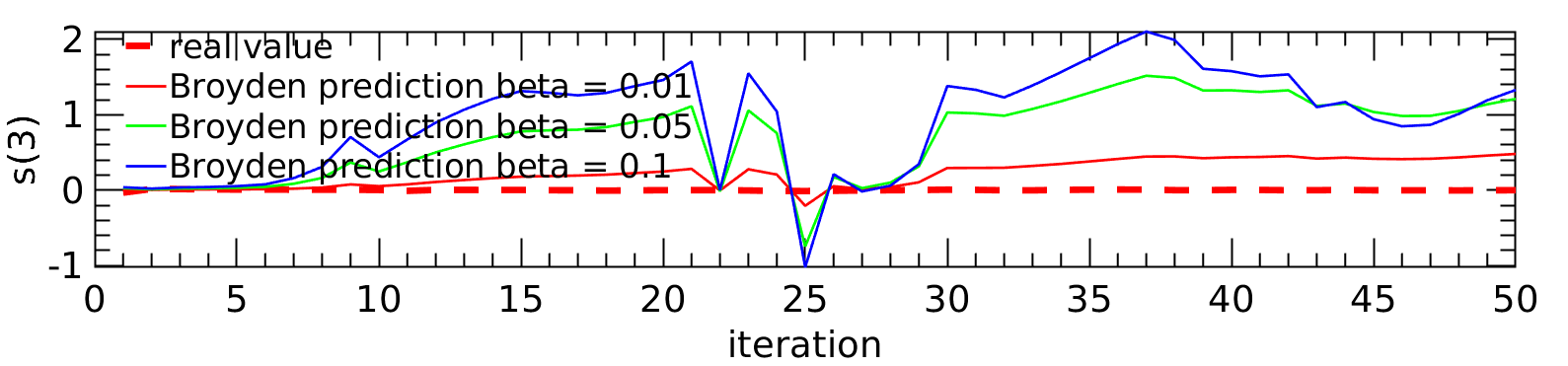}\label{fig:Brod_1}} 
    \caption{Comparison -- for estimating ${\vs}$ -- of the receding horizon approach (RH, left) and of the Broyden update (right, with three values of $\beta$). The topmost, middle and bottom plots show the one step prediction of $s_1$, $s_2$ and $s_3$, respectively. In all plots, the dashed red curve is the ground truth from the simulator. The plots clearly show that the receding horizon approach outperforms all three Broyden trials.} \vspace{-5mm}
    \label{fig:RH_Brod}
\end{figure}

\subsection{Manipulation of rigid objects}

The same framework can also be applied to rigid object manipulation. Consider the problem of moving a rigid object to a certain position and orientation via visual feedback. This time, the shape of the object does not change, but its pose will (it can translate and rotate). We use the same $M$, $\lambda$, $\epsilon$ and $\alpha$ as for cable manipulation. We compare the convergence of two control laws proposed in our paper (\ref{eq:one-step-ctrl}) and (\ref{eq:one-step-ctrl-inv}) against a baseline using image moments \citep{chaumette2004image}. The translation and orientation can be represented with image moments and the analytic interaction matrix can be computed as explained in~\citep{chaumette2004image}). To make methods compatible, we normalize the computed control and then multiply it by the same factor $0.01$. 

Figure \ref{fig:sim_singlearm_rigid} shows two simulations where our controller successfully moves a rigid object from an initial (blue) to a target (dashed black) pose using control law  (\ref{eq:one-step-ctrl-inv}). Figure \ref{fig:sim_singlearm_rigid_ASE} compares convergence of our methods against the image moments method. We can observe that our method provides a slightly slower convergence. Directly computing the inverse (\ref{eq:one-step-ctrl-inv}) provides a convergence similar to (\ref{eq:one-step-ctrl}). Later, we will show why our method is slower. Yet, the fact that it can be applied on both deformable and rigid objects makes it stand out over the other techniques.

\begin{figure}[t]
	\centering
	\subfloat{\includegraphics[width = 0.34\textwidth]{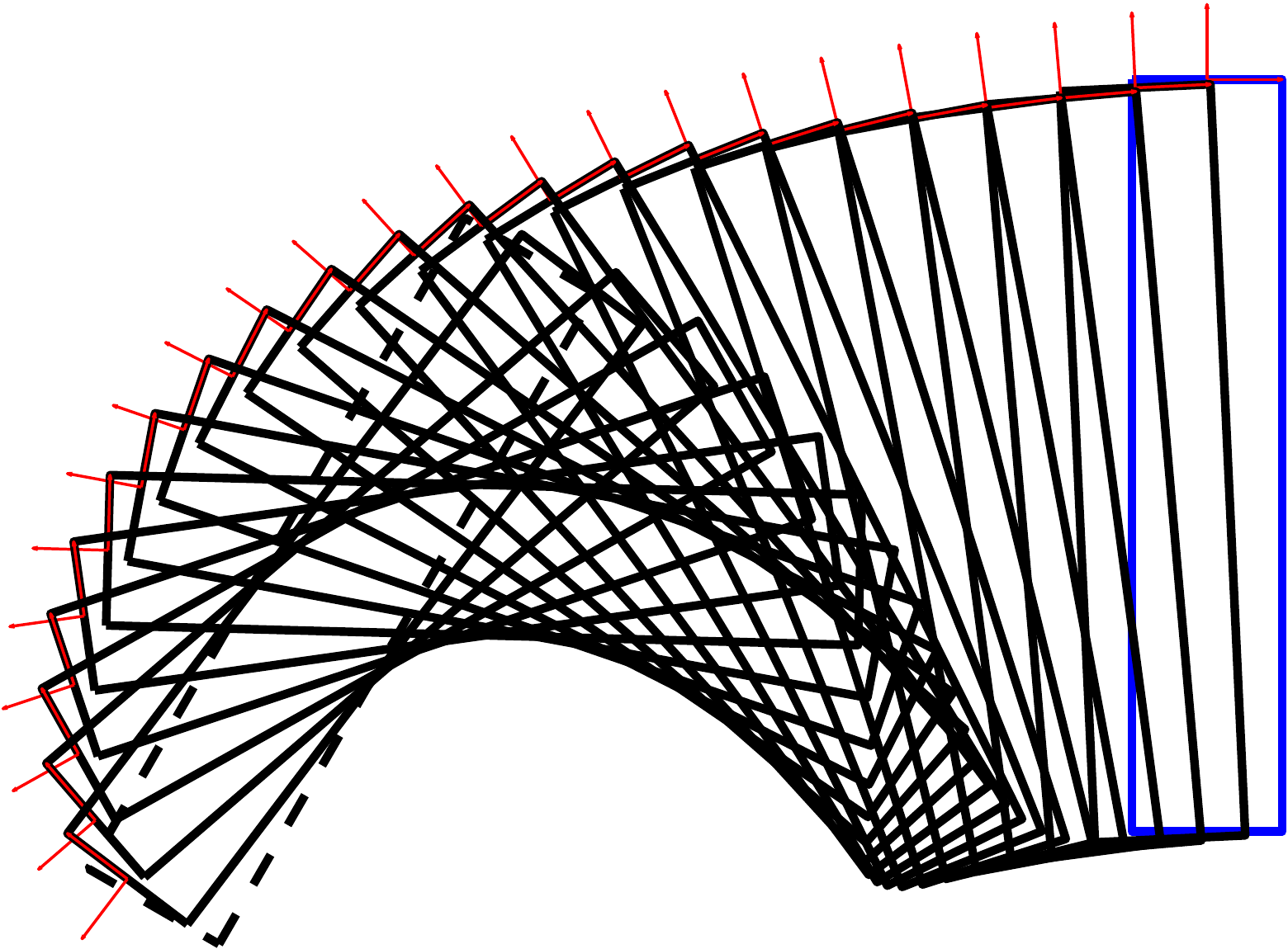}\label{fig:sim1}}   \hspace{1mm}
    \subfloat{\includegraphics[angle=90, width = 0.34\textwidth]{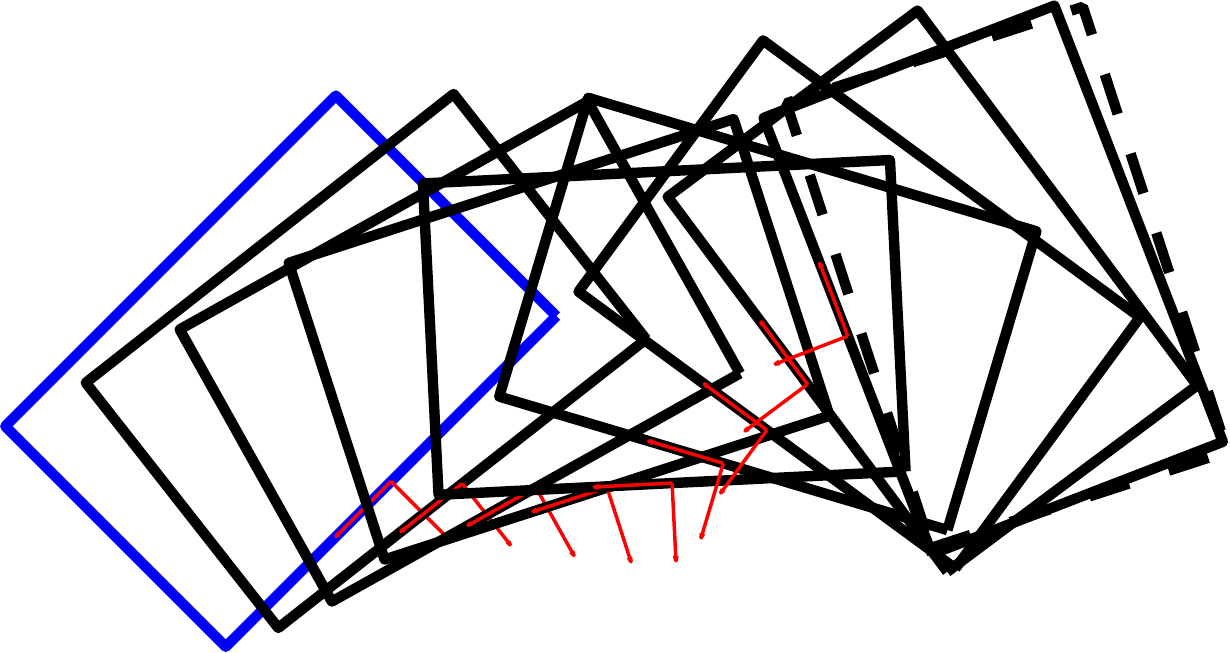}\label{fig:sim2}}
    \vspace{-0.3cm}
    \caption{Manipulation of a rigid object with a single end-effector (red frame). The initial, intermediate and target contours are respectively blue, solid black and dashed black. Note that in both cases, our controller moves the object to the target pose.}
    \vspace{-0.4cm}
    \label{fig:sim_singlearm_rigid}
\end{figure}

\begin{figure}[!thpb]
	\vspace{-0.1cm}
    \centering
    \includegraphics[width = 0.8\textwidth]{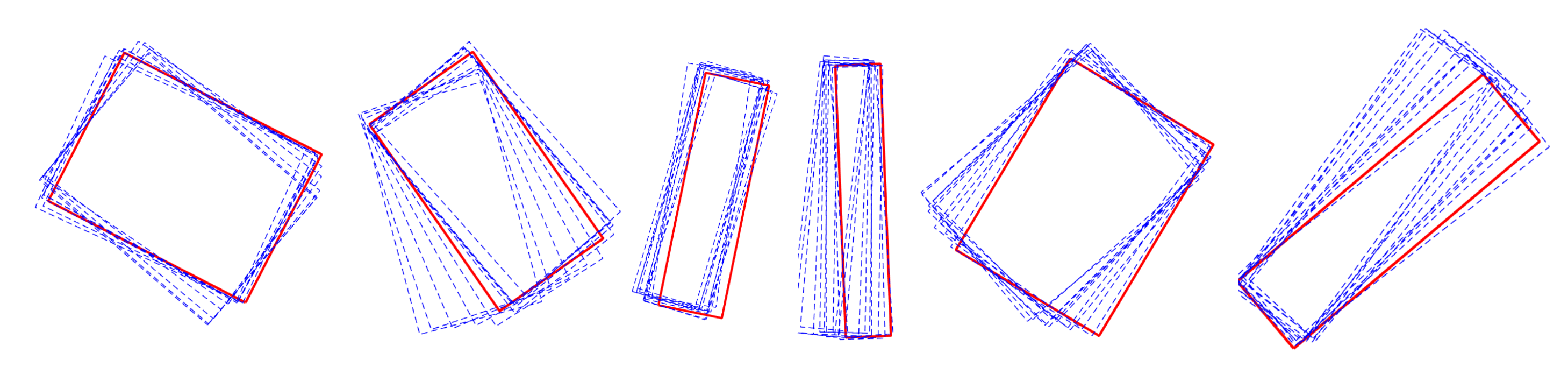}
    \vspace{-0.3cm}
    \caption{From an initial (red) pose, we generate $10$ (dashed blue) random motions of a rigid object. This figure shows multiple examples of different rectangular rigid objects.}
    \vspace{-0.3cm}
    \label{fig:rigid_analysis_random_move}
\end{figure}

\begin{figure}[b!]
	\centering
	\subfloat{\includegraphics[width = 0.5\textwidth]{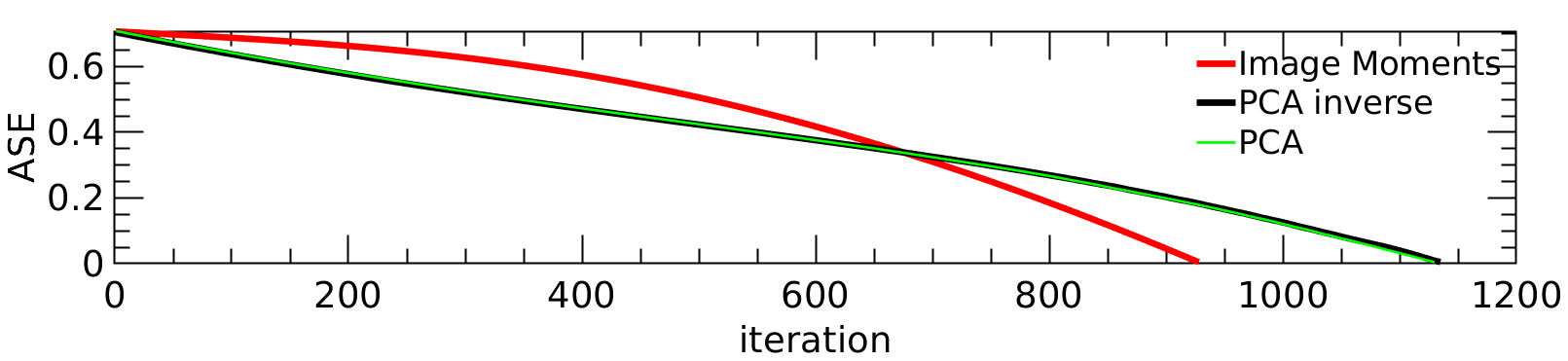}\label{fig:rigid_sim_1}} 
    \subfloat{\includegraphics[width = 0.5\textwidth]{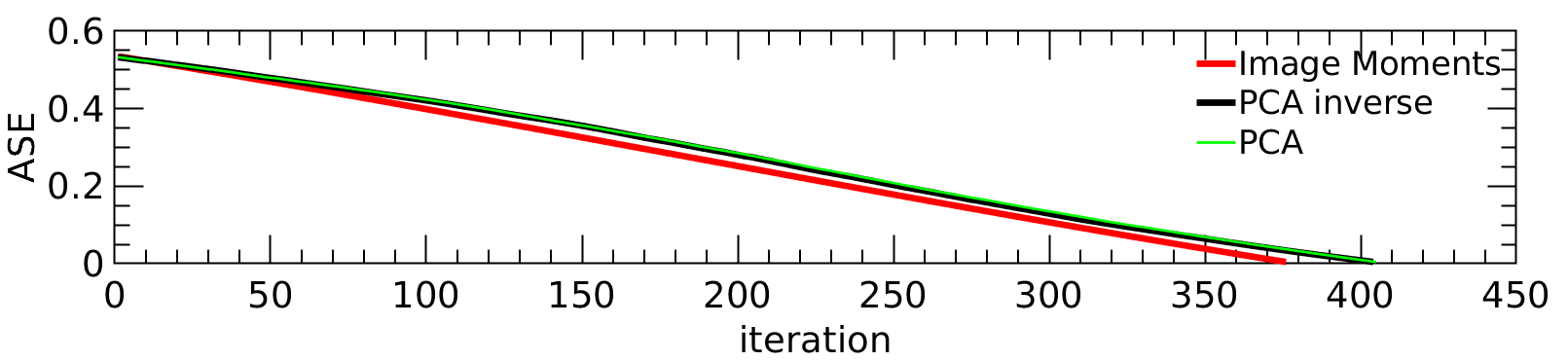}\label{fig:rigid_sim_2}}
    \vspace{-0.2cm}
    \caption{Evolution of ASE of the simulated rigid object manipulation using our method against image moments. Left: simulation in Fig. \ref{fig:sim_singlearm_rigid},  Right: simulation in Fig. \ref{fig:sim_singlearm_rigid}.}
    \label{fig:sim_singlearm_rigid_ASE}
\end{figure}

\subsection{Feature analysis for rigid objects}
In this section, we analyze locally what each component of the feature vector represents, in the case of rigid object manipulation. To this end, we apply $M = 10$ random movements (rotation range $[-0.11 , 0.09]$, maximum translation $15\%$ of the width) to multiple rigid rectangular objects (see Fig. \ref{fig:rigid_analysis_random_move}). 
We compute the projection matrix as explained in Sect.~\ref{sec:feat_vector}, and transform the contour samples to feature vectors. Then, we seek the relationship -- at each iteration -- between the object pose $x$, $y$, $\theta$ and the components of the feature vector $\vs$ generated by PCA. To this end, we use bivariate correlation \citep{feller2008introduction} defined by:
\begin{equation}
    \rho = \frac{E[(\xi-\bar{\xi})(\zeta-\bar{\zeta})]}{\sigma_\xi \sigma_\zeta},
\end{equation}
where $\xi$ and $\zeta$ are two variables with expected values $\bar{\xi}$ and $\bar{\zeta}$ and standard deviations $\sigma_\xi$ and $\sigma_\zeta$. An absolute correlation $|\rho|$ close to $1$ indicates that the variables are highly correlated. All the simulations in Fig. \ref{fig:rigid_analysis_random_move} exhibit similar correlation between the computed feature vector and the object pose. In Table \ref{tab:correlation}, we show one instance (Left first simulation in Fig. \ref{fig:rigid_analysis_random_move}) of the correlation between variables, with high absolute correlations marked in red. It is clear from the table that each component in the feature vector relates strongly to one pose parameter. We further demonstrate the correlation in Fig. \ref{fig:rigid_analysis}, where we plot the evolution of object poses and feature components. Note that $s_2$ and $\theta$ are negatively correlated. The slower convergence could be a result of the fact that, in contrast with image moments, here the extracted features and object pose are not completely decoupled. Yet, the main contribution of our method is that it can be directly used for both rigid and deformable objects. Therefore, it can be expected to be slower than methods specifically designed for rigid objects (such as image moments).

\begin{table}[!thpb]
    \centering
    \caption{Correlation $\rho$ between $s_1$, $s_2$, $s_3$ and $x$, $y$, $\theta$.}
        \begin{tabular}{c|c|c|c}
        \hline 
          & $x$   &  $y$ & $\theta$ \\
         \hline
         $s_1$ & -0.2819 & -0.3343  & \cellcolor{red} 0.9887 \\
         \hline
         $s_2$ & 0.2607 & \cellcolor{red} -0.8547  & -0.0465 \\
         \hline
         $s_3$ & \cellcolor{red} 0.9230 & 0.3629  & -0.1426 \\
         \hline
    \end{tabular}
        \label{tab:correlation}
\end{table}

\begin{figure}[!thpb]
	\vspace{-0.6cm}
	\centering
	\subfloat{\includegraphics[width = 0.45\textwidth]{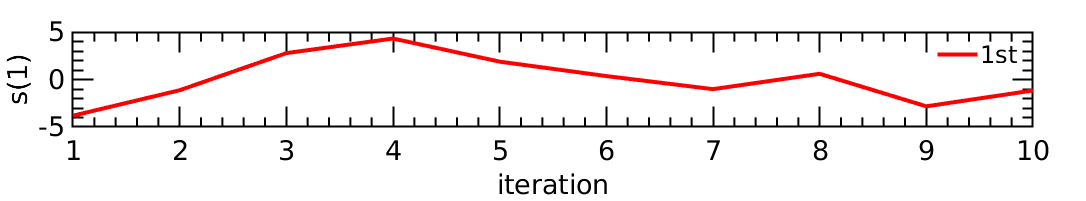}} 
	\vspace{-0.3cm}
    \subfloat{\includegraphics[width = 0.45\textwidth]{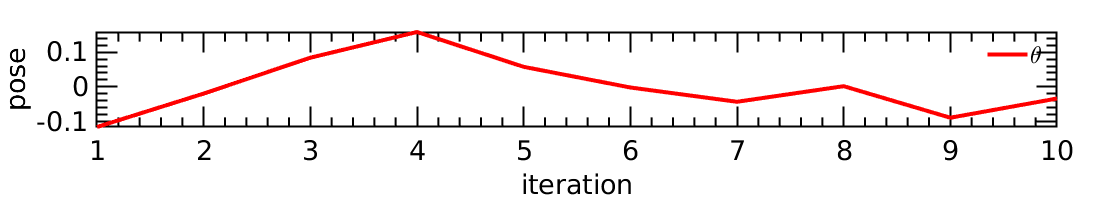}}
    \\
    \vspace{-0.3cm}
    \subfloat{\includegraphics[width = 0.45\textwidth]{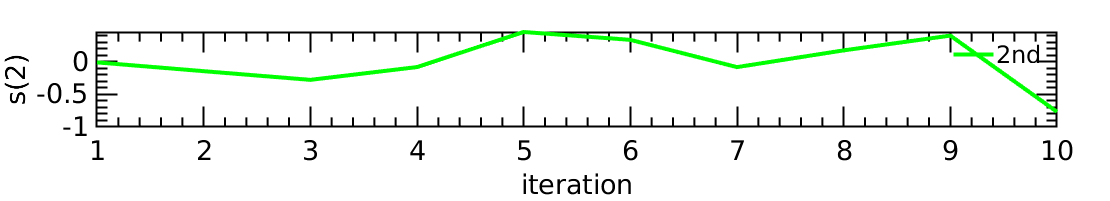}} 
    \vspace{-0.3cm}
    \subfloat{\includegraphics[width = 0.45\textwidth]{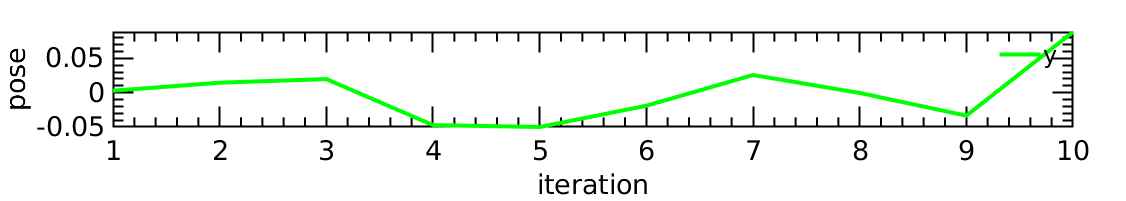}}
    \vspace{-0.3cm}
    \subfloat{\includegraphics[width = 0.45\textwidth]{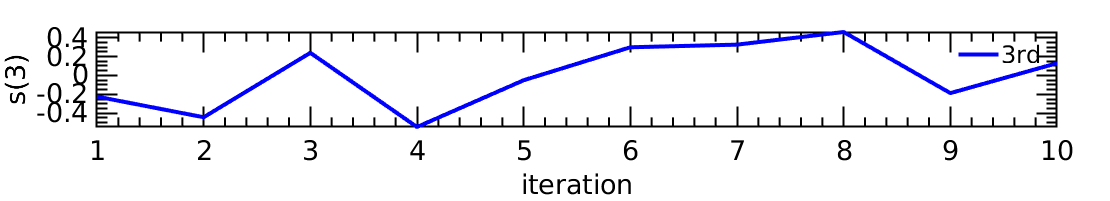}} 
    \vspace{-0.3cm}
    \subfloat{\includegraphics[width = 0.45\textwidth]{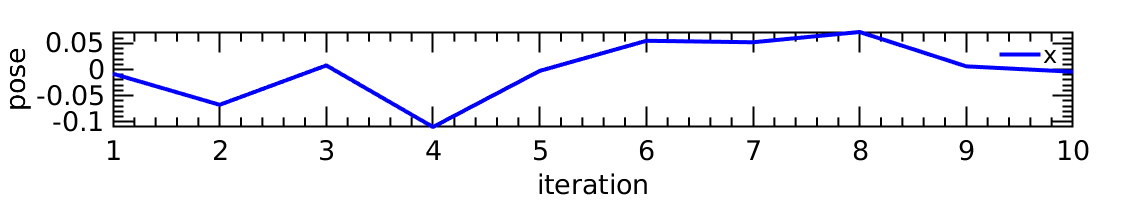}}
    \vspace{-0.1cm}
    \caption{Progression of the auto-generated feature components (row 1, 3, 5: $s_1$, $s_2$, $s_3$) vs. object pose (row 2, 4, 6: $x$, $y$, $\theta$). We have purposely arranged the variables with high correlation with the same color.}
    \label{fig:rigid_analysis}
\end{figure}

\section{Experiments}\label{sec:exp}

Figure~\ref{fig:TRO_setup} outlines our experimental setup. We use a KUKA LWR IV arm. We constrain it to planar ($n = 3$) motions $\delta \vr$, defined in its base frame (red in the figure): two translations $\delta x$ and $\delta y$ and one counterclockwise rotation $\delta \theta$ around $z$. A Microsoft Kinect V2 observes the object\footnote{We only use the RGB image -- not the depth -- from the sensor.}. A Linux-based $64$-bit PC processes the image at $30$ fps. In the following sections, we first introduce the image processing for contour extraction, then present the experiments.

\begin{figure}[thpb]
    \centering
    \includegraphics[width = 0.6\textwidth]{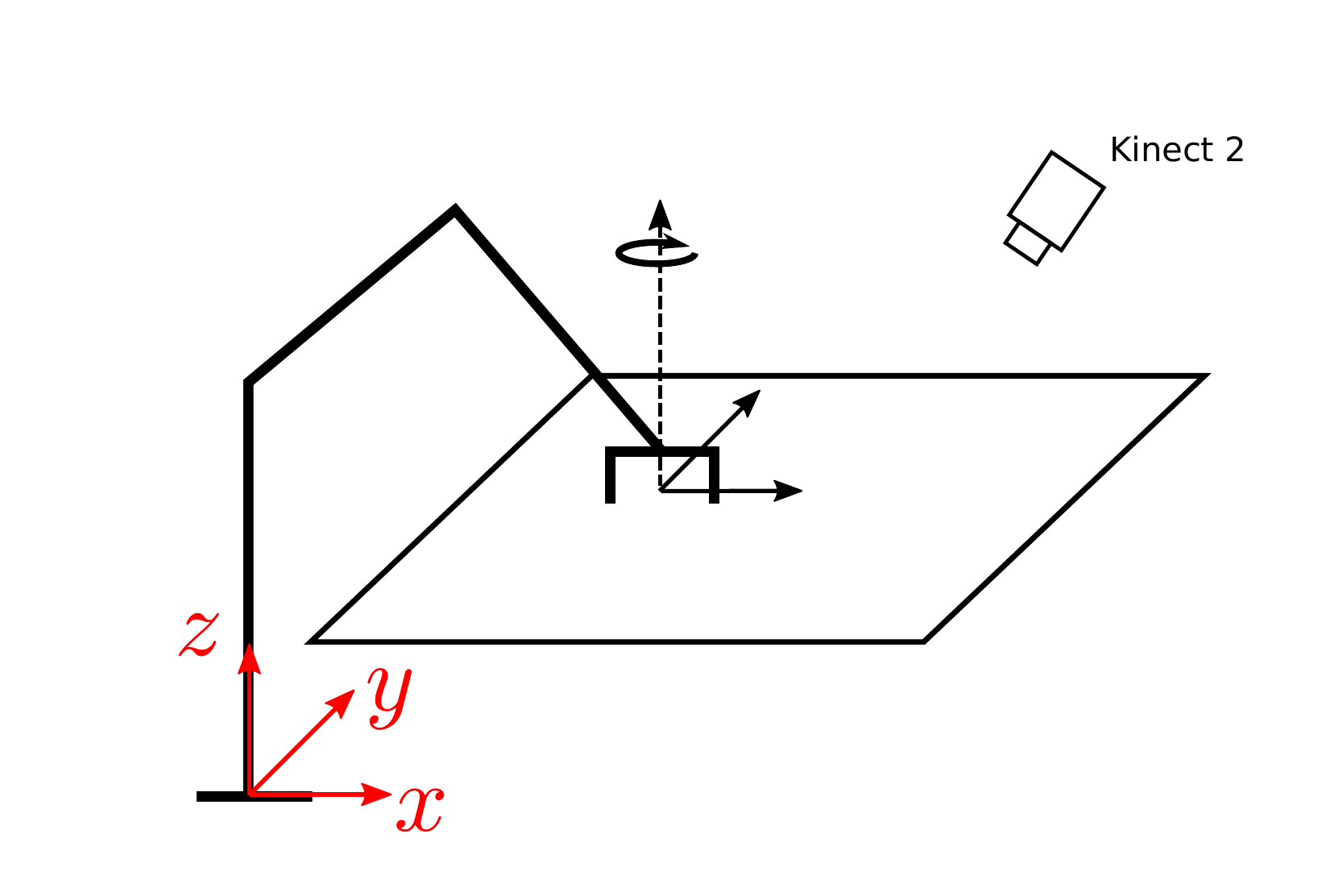}
    \vspace{-0.6cm}
    \caption{Overview of the experimental setup.}
    \label{fig:TRO_setup}
\end{figure}


\subsection{Image processing}\label{sec:5-vision-pipeline}

This section explains how we extract and sample the object contours from an image. We have developed two pipelines, according to the kind of contours (See Fig. \ref{fig:contourInImage}): open (e.g., representing a cable) and closed. We hereby describe the two.

\begin{figure}[thpb]
    \centering
    {\includegraphics[width = 0.4\textwidth]{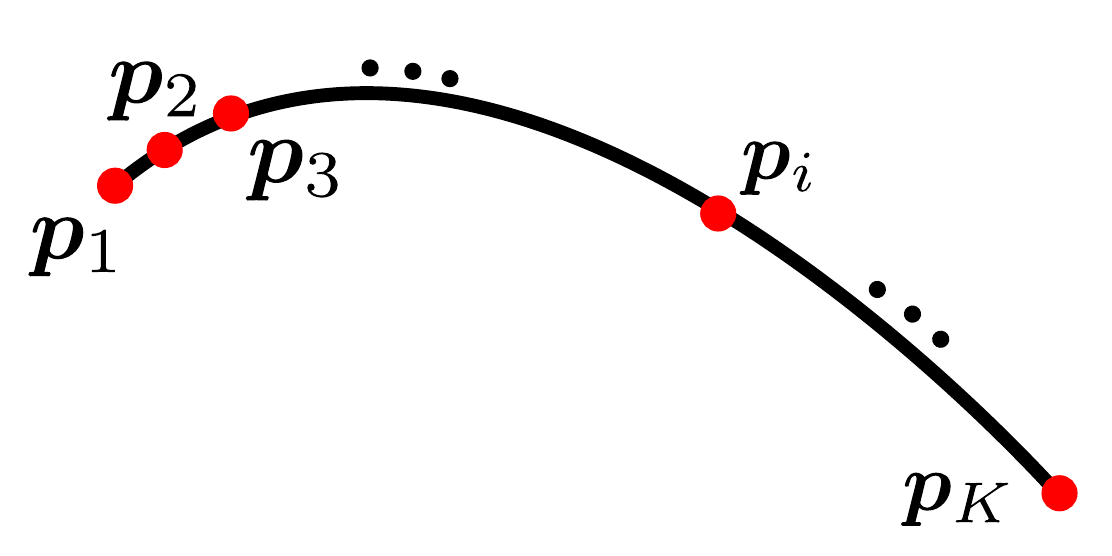}}
    {\includegraphics[width = 0.4\textwidth]{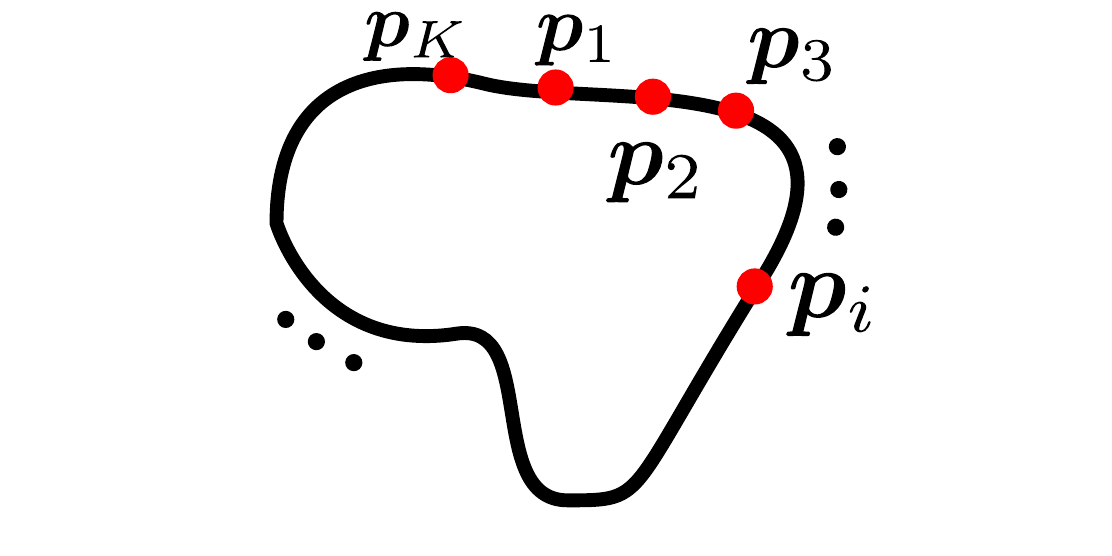}}
    \vspace{-0.4cm}
	\caption{Open (left) and closed (right) contours can be both represented by a sequence of sample pixels in the image.}
	\vspace{-0.7cm}
	\label{fig:contourInImage}
\end{figure}

\subsubsection{Open contours}

The overall pipeline for extracting an open contour is illustrated in Fig.~\ref{fig:image_processing_open} and \textbf{Algorithm~\ref{alg:resample}}. On the initial image, the user manually selects a Region of Interest (ROI, see Fig. \ref{fig:original_cable_22}) containing the object. In this ROI, we apply thresholding, followed by a morphological opening, to obtain a binary image as in Fig. \ref{fig:binary_cable_22}. This image is dilated to generate a mask (Fig. \ref{fig:mask_cable_22}) used to compute the new ROI for the following image. Figure \ref{fig:original_cable_23} is the object after a small manipulation motion and \ref{fig:mask_on_next} shows the mask (in grey color) which contains the cable. The OpenCV \textit{findContour} function is applied to binary image, then two contours are extracted based on the two known ends of the cable,  both are re-sampled (with same value of $K$) using \textbf{Algorithm~\ref{alg:resample}} and finally merged (by interpolation, for each sample, between the two contours' corresponding point) into the uniformly sampled open contour $\vc$ (see Fig.~\ref{fig:debug_cable_22}, where the green box indicates the end-effector).

\begin{figure}[!thpb]
	\vspace{-0.7cm}
	\centering
	\subfloat[ROI]{\includegraphics[width = 0.3\textwidth]{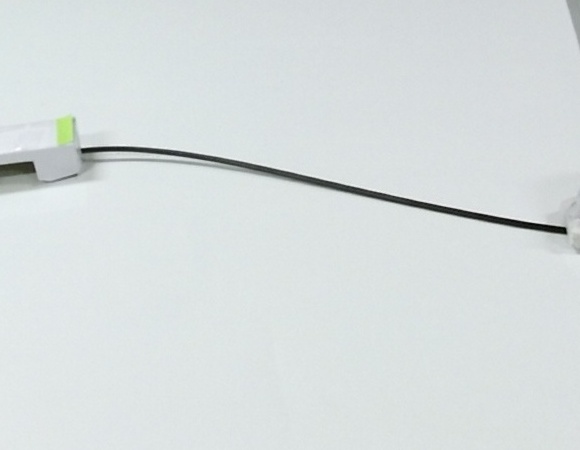}\label{fig:original_cable_22}}  \hspace{0.5mm}
    \subfloat[After thresholding]{\includegraphics[width = 0.3\textwidth]{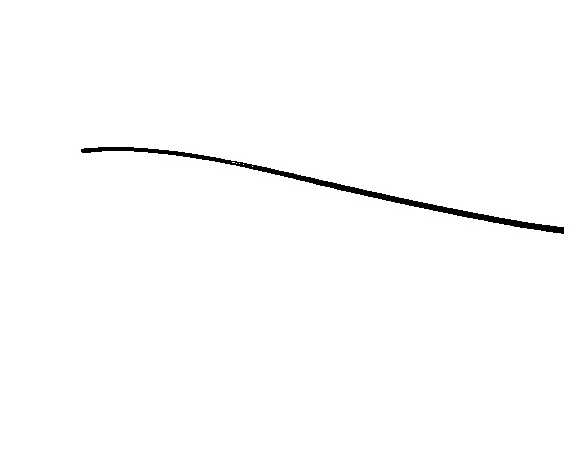}\label{fig:binary_cable_22}}
    \hspace{0.5mm}
    \subfloat[Mask]{\includegraphics[width = 0.3\textwidth]{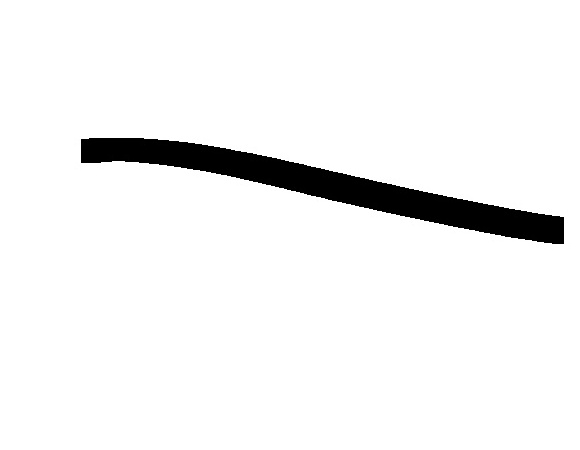}\label{fig:mask_cable_22}}\\
    \subfloat[Sampled data]{\includegraphics[width = 0.3\textwidth]{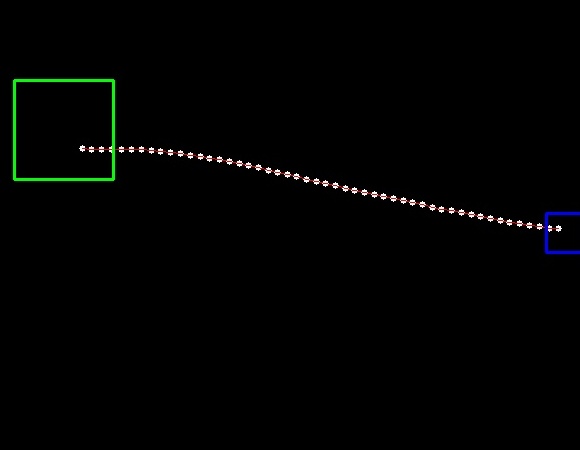}\label{fig:debug_cable_22}}
    \subfloat[Next image]{\includegraphics[width = 0.3\textwidth]{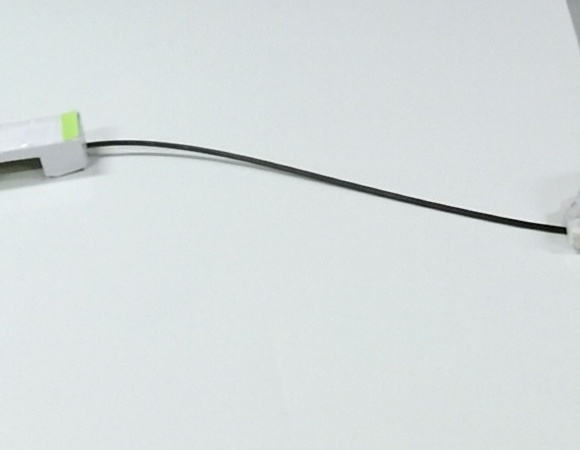}\label{fig:original_cable_23}} \hspace{0.5mm}
    \subfloat[Cable in the mask]{\includegraphics[width = 0.3\textwidth]{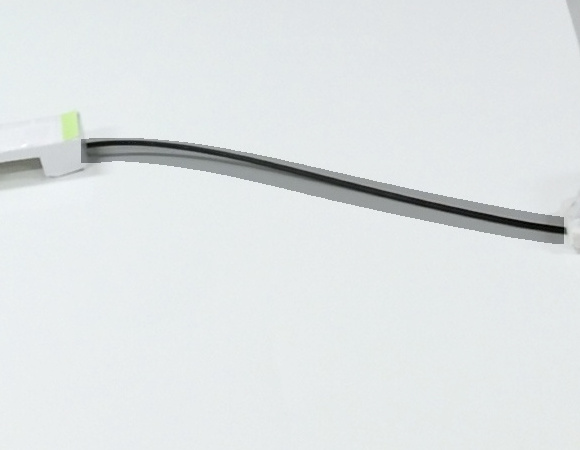}\label{fig:mask_on_next}}
    	\vspace{-0.1cm}
    \caption{Image processing steps needed to obtained the sampled open contour of an object (here, a cable). 
    	}
    	\vspace{-0.3cm}
    \label{fig:image_processing_open}
\end{figure}

\begin{figure}[b!]
		\vspace{-0.6cm}
		
	\centering
	\subfloat[]{\includegraphics[width = 0.22\textwidth]{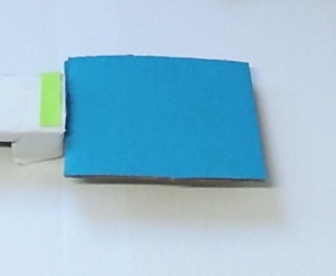}\label{fig:rigid_original}}  
	\hspace{0.2mm}
	\subfloat[]{\includegraphics[width = 0.22\textwidth]{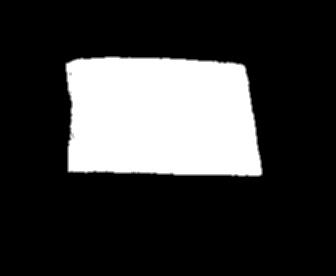}\label{fig:rigid_blob}}
	\hspace{0.2mm}
	\subfloat[]{\includegraphics[width = 0.22\textwidth]{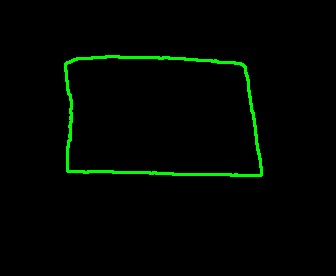}\label{fig:rigid_contour}}
	\hspace{0.2mm}
	\subfloat[]{\includegraphics[width = 0.22\textwidth]{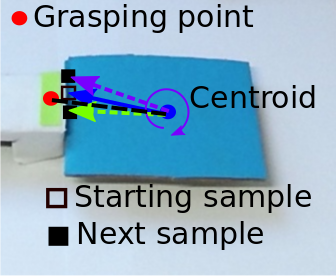}\label{fig:rigid_debug_image}}
	\vspace{-0.2cm}
	\caption{Image processing for getting a sampled closed contour: (a) original image, (b) image after thresholding and Gaussian blur, (c) extracted contour, (d) finding the starting sample and the order of the samples.}
	\label{fig:image_processing_closed}
\end{figure}

\subsubsection{Closed contours}
The procedure is shown in Fig. \ref{fig:image_processing_closed}. For an object with uniform color (in the experiment blue), we apply HSV segmentation, followed by Gaussian blur of size $3$, and finally the OpenCV \textit{findContour} function, to get the object contour. The contour is then re-sampled using \textbf{Algorithm \ref{alg:resample}}. The starting point and the order of the samples is determined by tracking the grasping point (red dot in Fig. \ref{fig:rigid_debug_image}) and the centroid of the object (blue dot). We obtain the vector connecting the grasping point to the centroid. Then, the starting sample is the one closest to this vector, and we proceed along the contour clockwise. Therefore, the extracted contour is ordered clockwise and always starting from the same point. This solves the contour ambiguity for symmetrical objects.

\begin{algorithm}[!thp]
	\caption{Generate fixed number of points with uniform spacing \newline
		\textit{Input} $\vP_I = [\vp_I(1),\cdots,\vp_I(N)]$: original ordered sampled data \newline
		\textit{Input} $\vK$ = target number of data samples \newline
		\textit{Input} $\vepsilon$ = infinitesimal threshold for equality of distances \newline
		\textit{Output} $\vP_O = [\vp_O(1),\cdots,\vp_O(K)]$: re-sampled data with uniform spacing.}
	\label{alg:resample}
	\begin{algorithmic}[1]
		\State compute the full length $\mathcal{L}$ of $\vP_I$. 
		\State compute desired distance per sample: $\mu = \mathcal{L}/K$
		\State $l = 1$, $dist = 0$ 
		\State $\vp_{curr} = \vp_O(l) = \vp_I(l)$ 
		\State $h = l = l + 1$ 
		\While{$l \leq N$}
		\State $\vp_{next} = \vp_I(l)$ 
		\State $d = ||\vp_{next} - \vp_{curr}||_2$ 
		\If{$d + dist \leq \mu$} 
		\State $dist = dist + d$ 
		\State $\vp_{curr} = \vp_{next}$
		\State $l = l + 1$
		\Else 
		\State $\vp_{curr} = \vp_O(h) = \vp_{curr} + (\vp_{next} - \vp_{curr})\frac{\mu - dist}{d}$ 
		\State $h = h + 1$ 
		\State $dist = 0$ 
		\EndIf
		\EndWhile
		\If{$|\mu - dist| < \epsilon $} 
	    \State $\vp_O(h) = \vp_{curr} $
		\EndIf
	\end{algorithmic}
\end{algorithm}
\subsection{Vision-based manipulation}

In this section, we present the experiments that we ran to validate our algorithms, also visible at \href{https://youtu.be/gYfO2ZxZ5KQ}{https://youtu.be/gYfO2ZxZ5KQ}. To demonstrate the generality of our framework, we tested it with:
\begin{itemize}
    \item Rigid objects represented by closed contours,
    \item Deformable objects represented by open contours (cables),
    \item Deformable objects represented by closed contours (sponges).
\end{itemize}
We carried out different experiments with a variety of initial and target contours and camera-to-object relative poses. 
The variety of both geometric and physical properties demonstrates the robustness of our framework. The variety of camera-to-object relative poses shows that---as usual in image-based visual servoing~\citep{chaumette2006visual}---camera calibration is unnecessary. The algorithm and parameters are the same in all experiments; the only differences are in the image processing, depending on the type of contour (closed or open, see Sect.~\ref{sec:5-vision-pipeline}).

We obtain the target contours by commanding the robot with predefined motions. Once the target contour is acquired, the robot goes back to the initial position, and then should autonomously reproduce the target contour. 
Again, we set the number of features $k = n =3$, and use $K = 50$ samples to represent the contour $\vc$. We set the window size $M = 5$, both for obtaining the feature vector $\vs$ and the interaction matrix $\vL$.  We choose the control gain to be $0.01$. A larger control gain may result in faster convergence, but could also lead to oscillation. The local target threshold $\epsilon$ is set to $0.8$. A higher threshold will result in closer local target and vice versa. The Tikhonov factor used to ensure numerical stability for matrix inversion, is set to $\lambda = 0.01$.   

At the beginning of each experiment, the robot executes $5$ steps of small\footnote{The notion of small is relative, and usually dependent on the size of the object the robot is manipulating. Refer to Sect. \ref{sec:select_feature_dimension} (especially Fig. \ref{fig:test_k}) for a discussion on this.} random motions to obtain the initial features and interaction matrix.

For all the experiments, we set the same termination condition at iteration $i + 1$ using ASE defined in (\ref{eq:ASE}) such that:
\begin{enumerate}
    \item $\text{ASE}_i < 1$ pixel and
    \item $\text{ASE}_{i+1} \geq \text{ASE}_i$.
\end{enumerate}

In the graphs that follow, we show the evolution of \text{ASE} in blue before the termination condition, and in red after the condition (until manual stop by the operator). 

Figure \ref{fig:robot_experiment} presents $8$ experiments, one per column. Columns 1 -- 3, 4 -- 6 and 7 -- 8 show respectively manipulation of: cable, rigid object and sponge. The first row presents the full RGB image obtained from Kinect V2. The second and third rows zoom in on the manipulation at the initial and final iterations. We track the end-effector in the image with a green marker for contour sampling. The target and current contours are drawn in red and blue, respectively. 

\begin{figure*}[!thpb] 
	\centering
	\subfloat{\includegraphics[width = 0.11\textwidth]{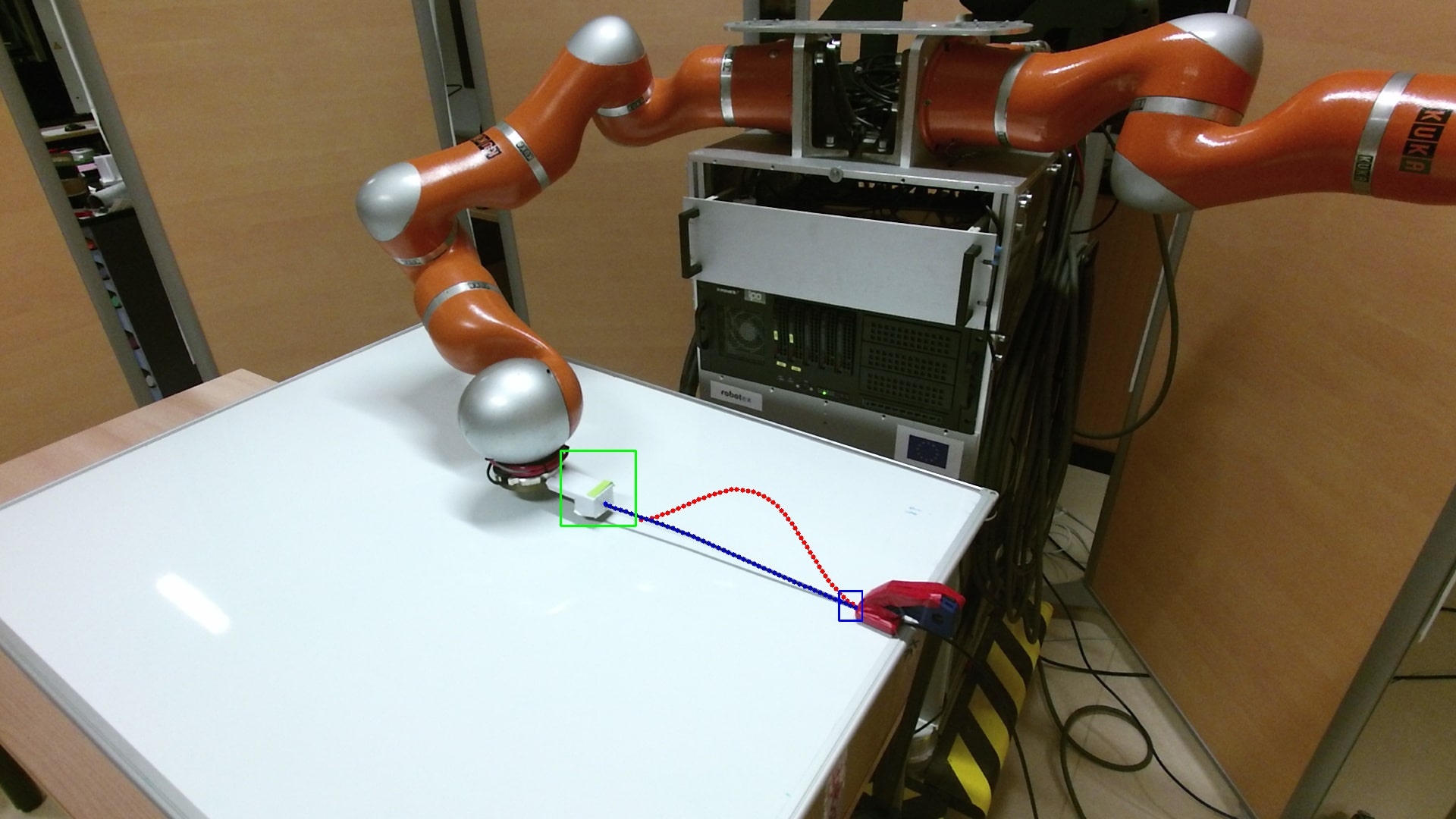}}   
	\hspace{1mm}
    \subfloat{\includegraphics[width = 0.11\textwidth]{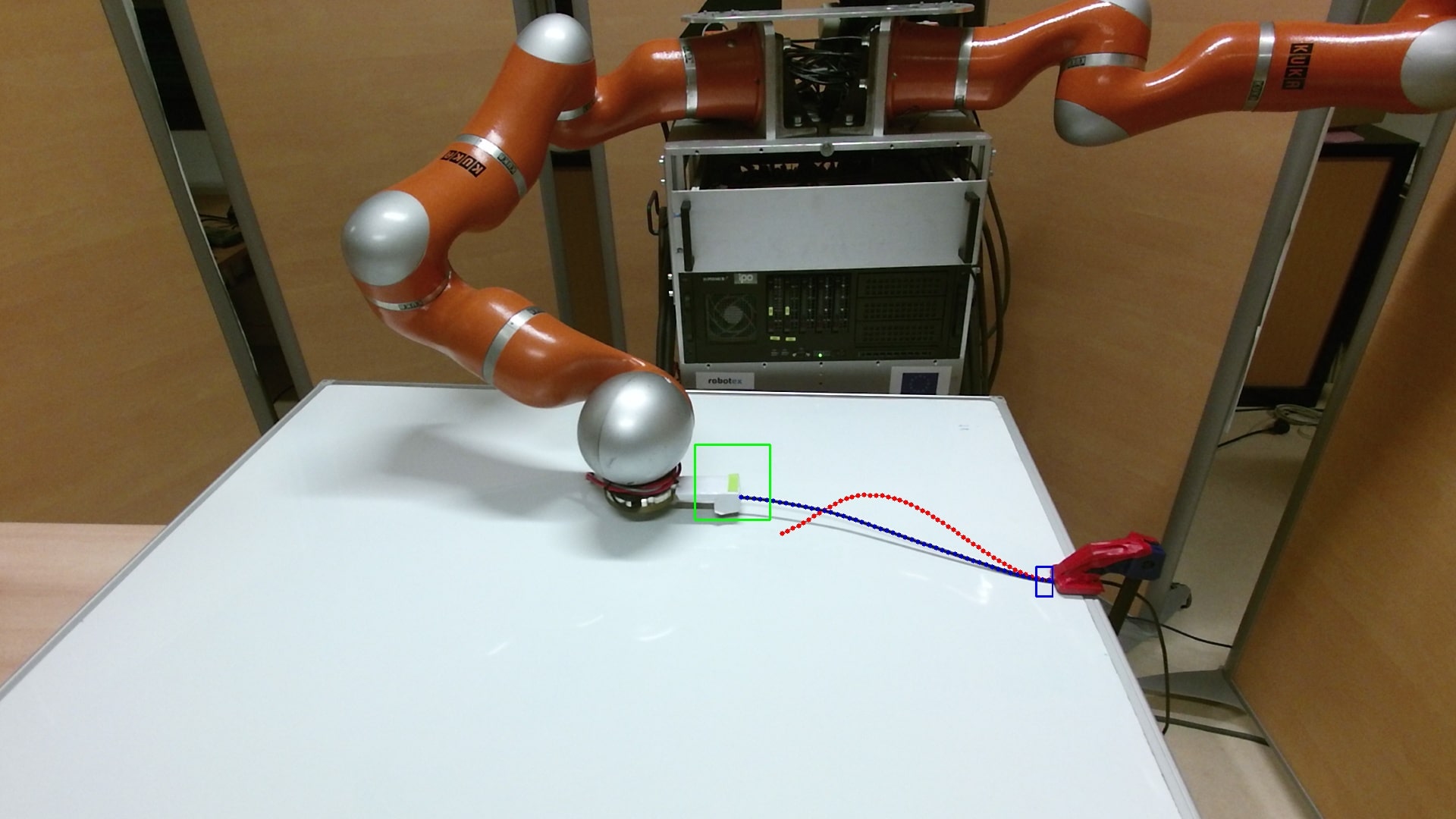}}   
	\hspace{1mm}
	\subfloat{\includegraphics[width = 0.11\textwidth]{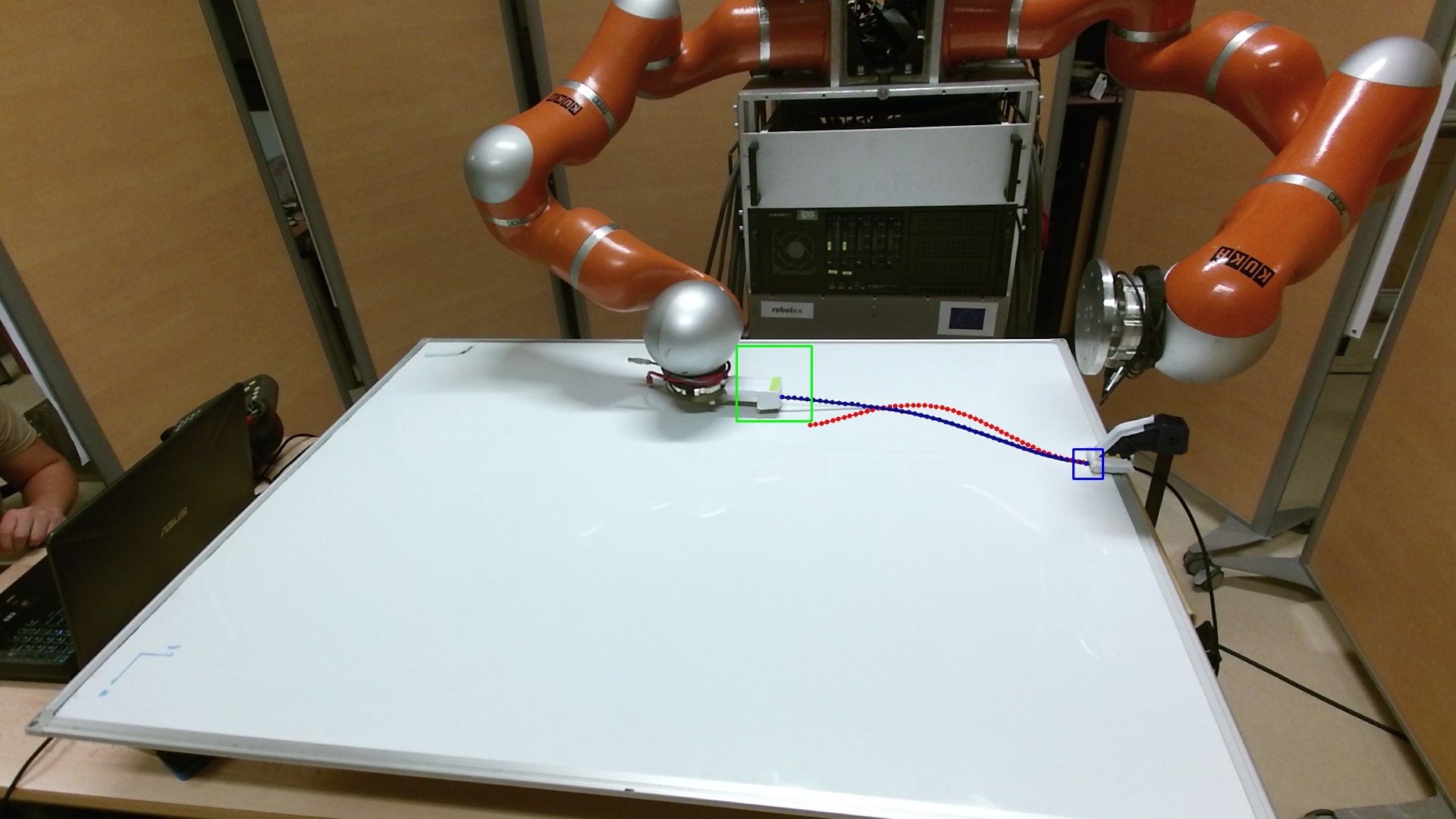}}   
	\hspace{1mm}
	\subfloat{\includegraphics[width = 0.11\textwidth]{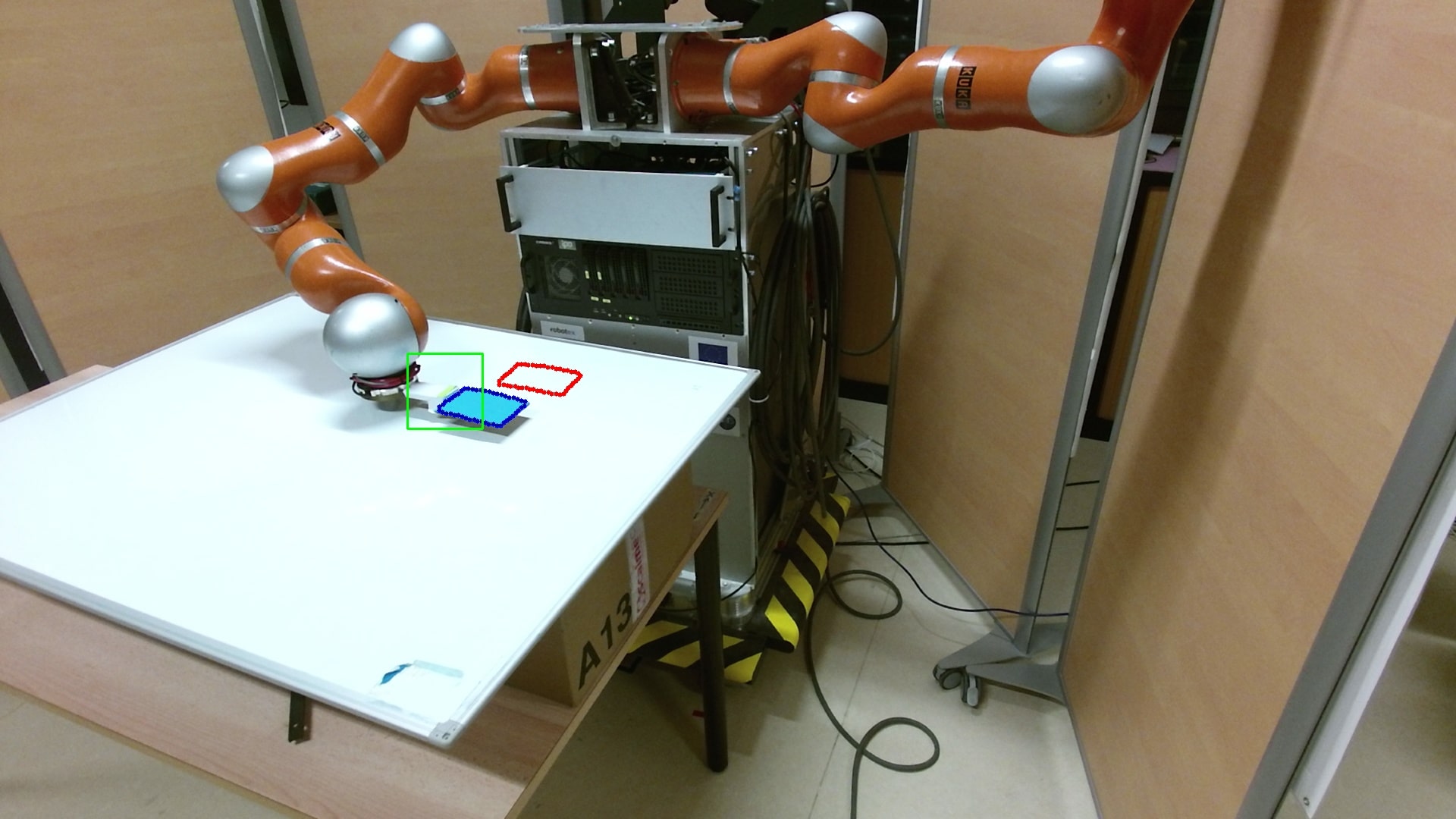}}   
	\hspace{1mm}
    \subfloat{\includegraphics[width = 0.11\textwidth]{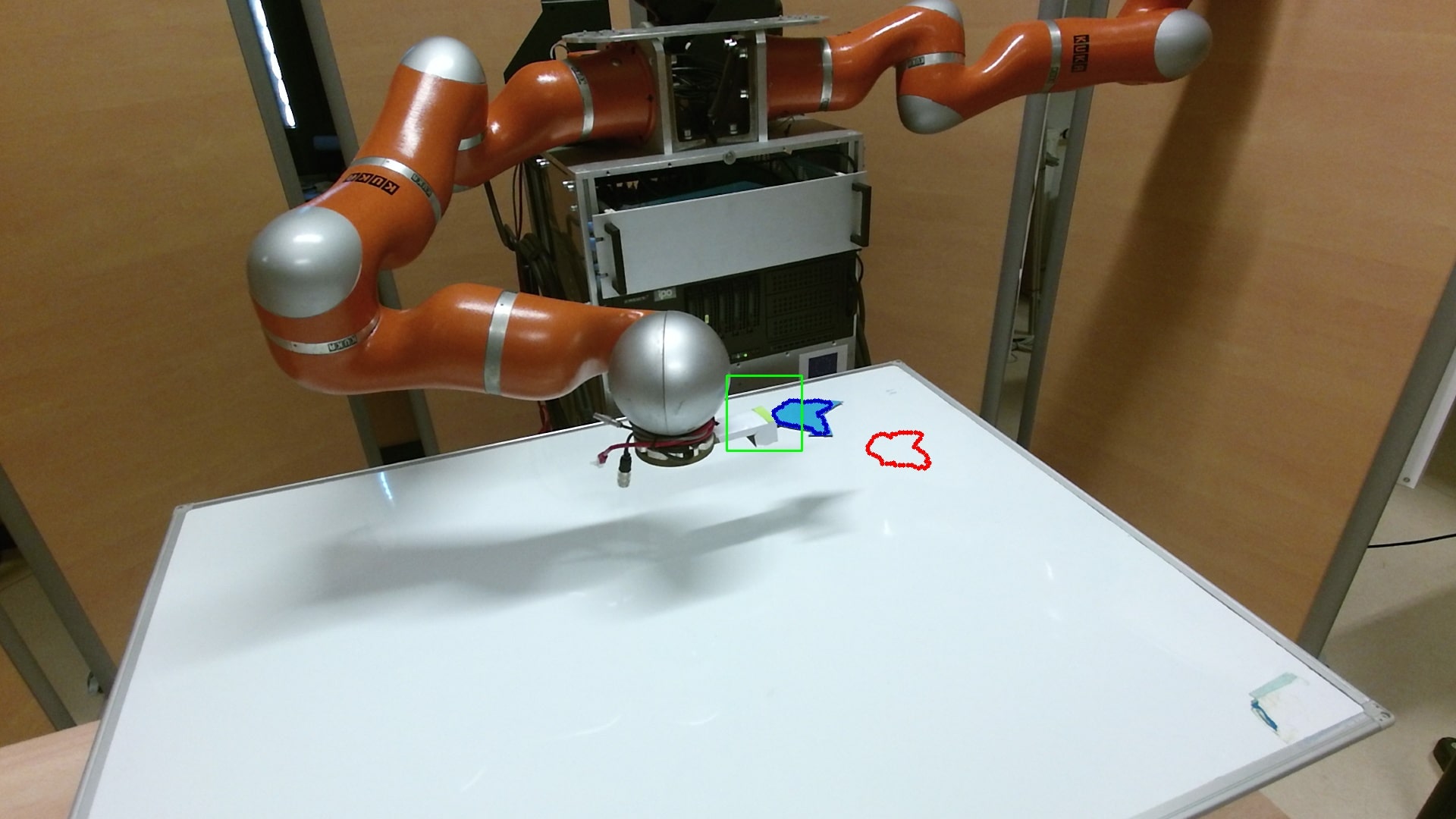}}   
	\hspace{1mm}
	\subfloat{\includegraphics[width = 0.11\textwidth]{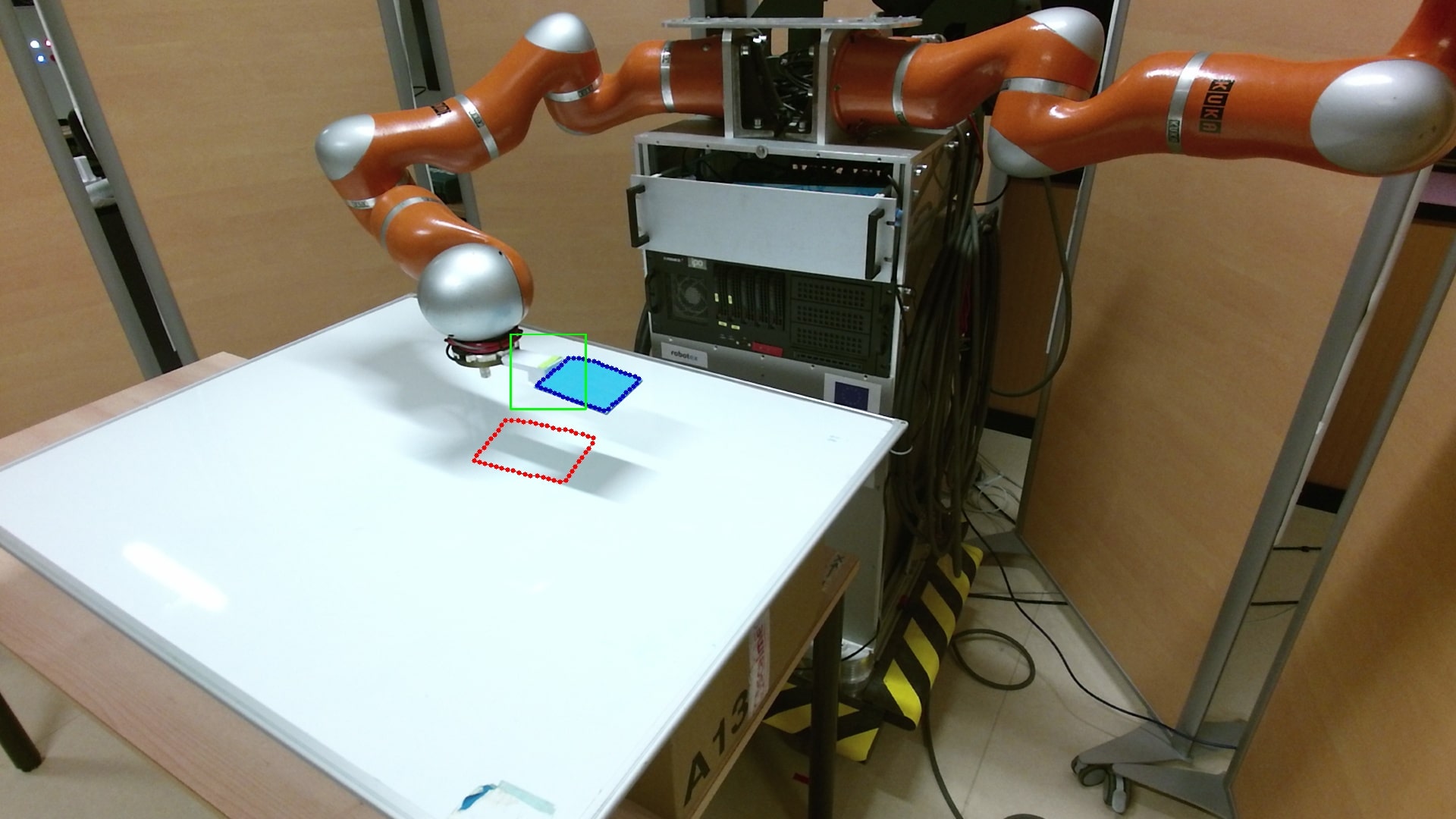}}
	\hspace{1mm}
	\subfloat{\includegraphics[width = 0.11\textwidth]{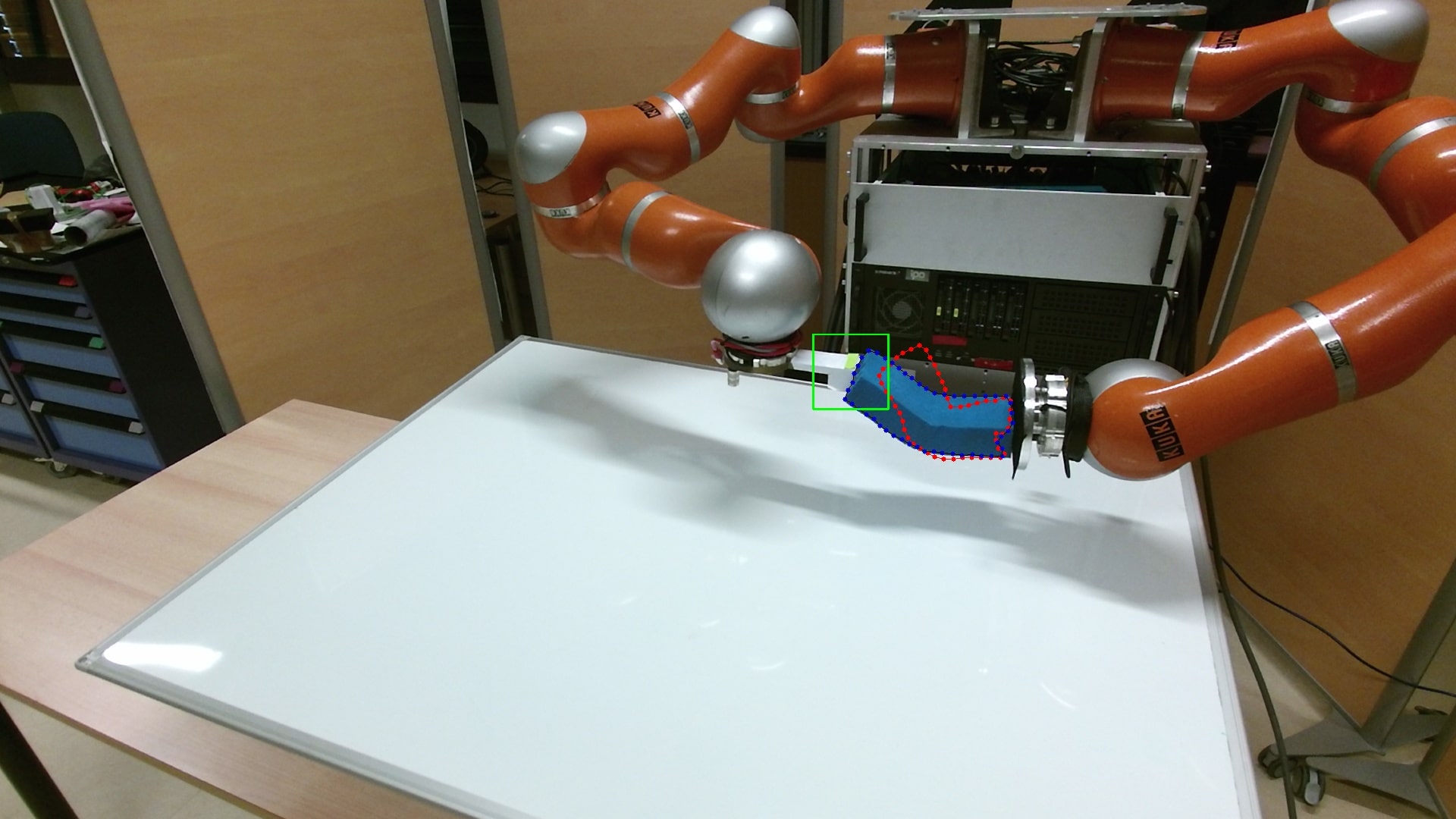}}   
	\hspace{1mm}
	\subfloat{\includegraphics[width = 0.11\textwidth]{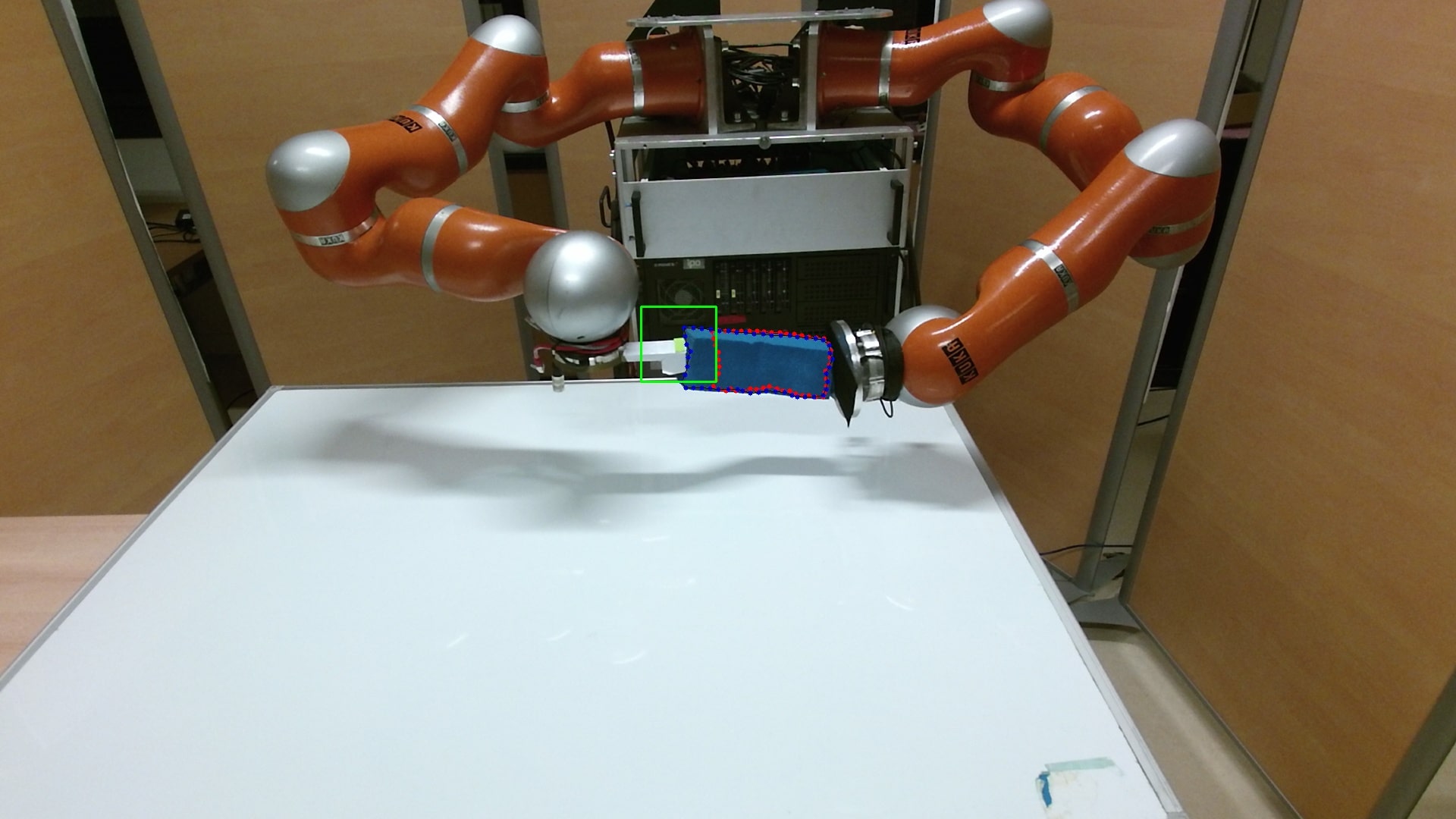}}
    \\
    \subfloat{\includegraphics[width = 0.11\textwidth]{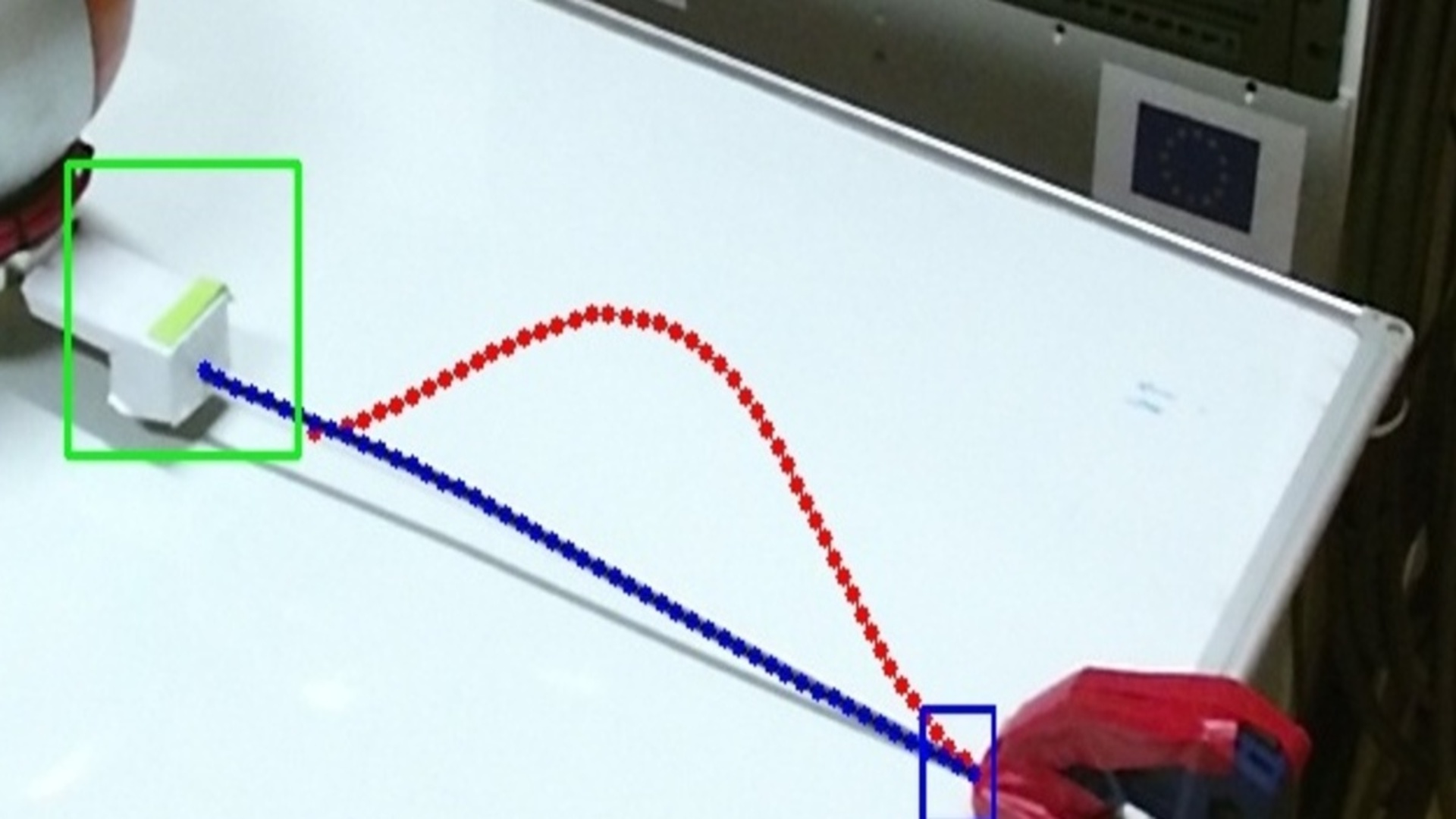}}
    \hspace{1mm}
    \subfloat{\includegraphics[width = 0.11\textwidth]{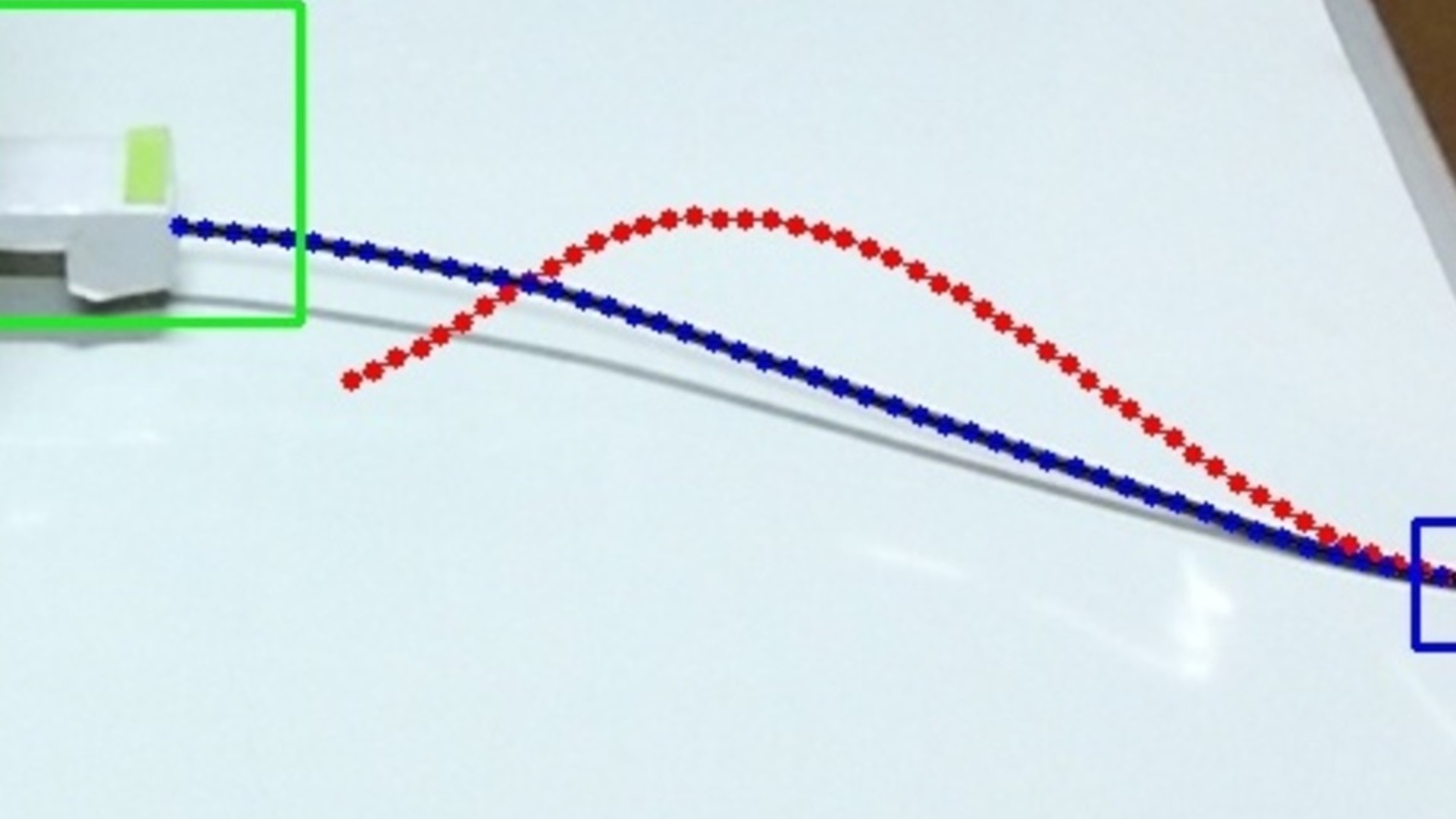}}
    \hspace{1mm}
    \subfloat{\includegraphics[width = 0.11\textwidth]{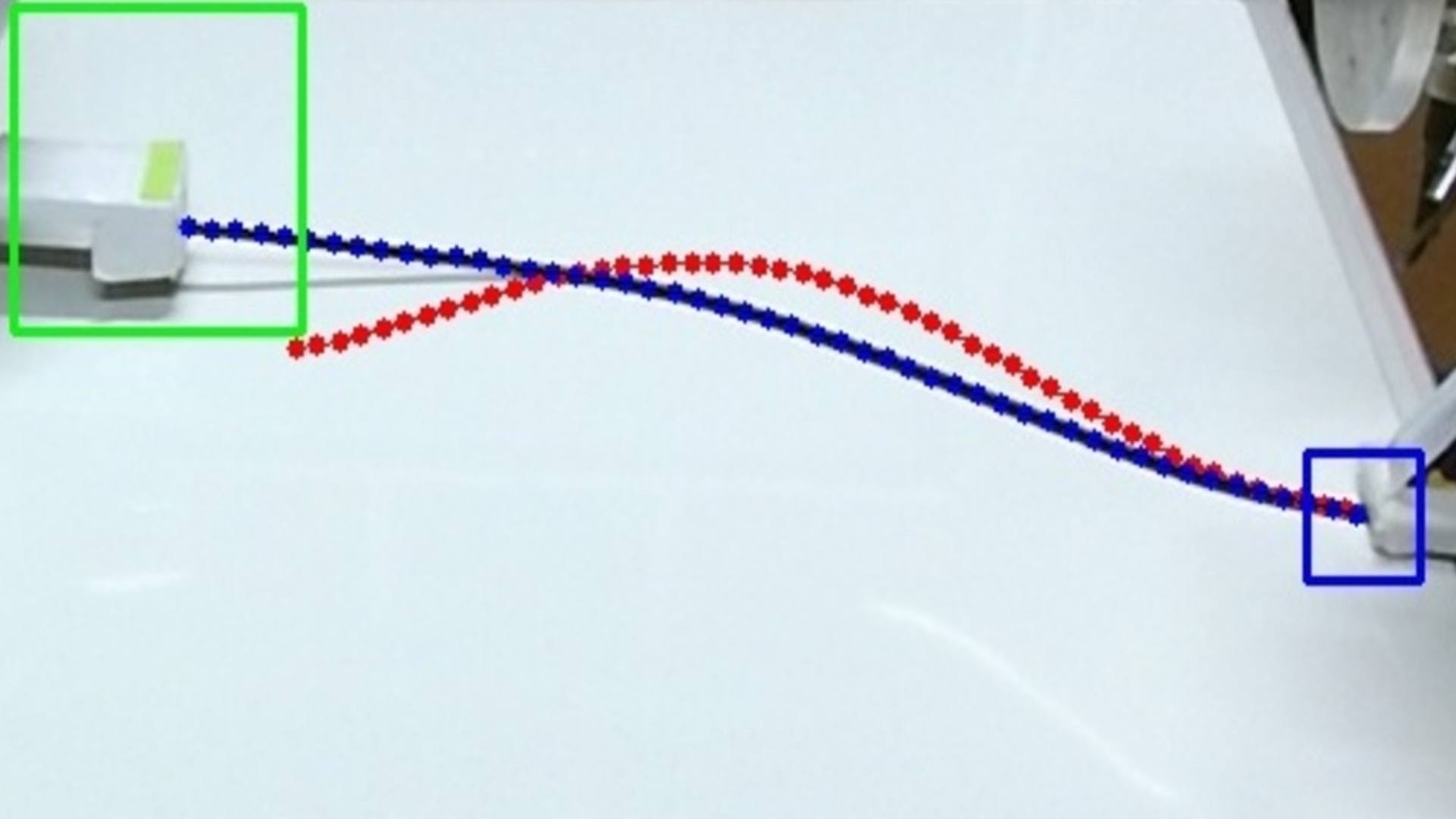}}
    \hspace{1mm}
	\subfloat{\includegraphics[width = 0.11\textwidth]{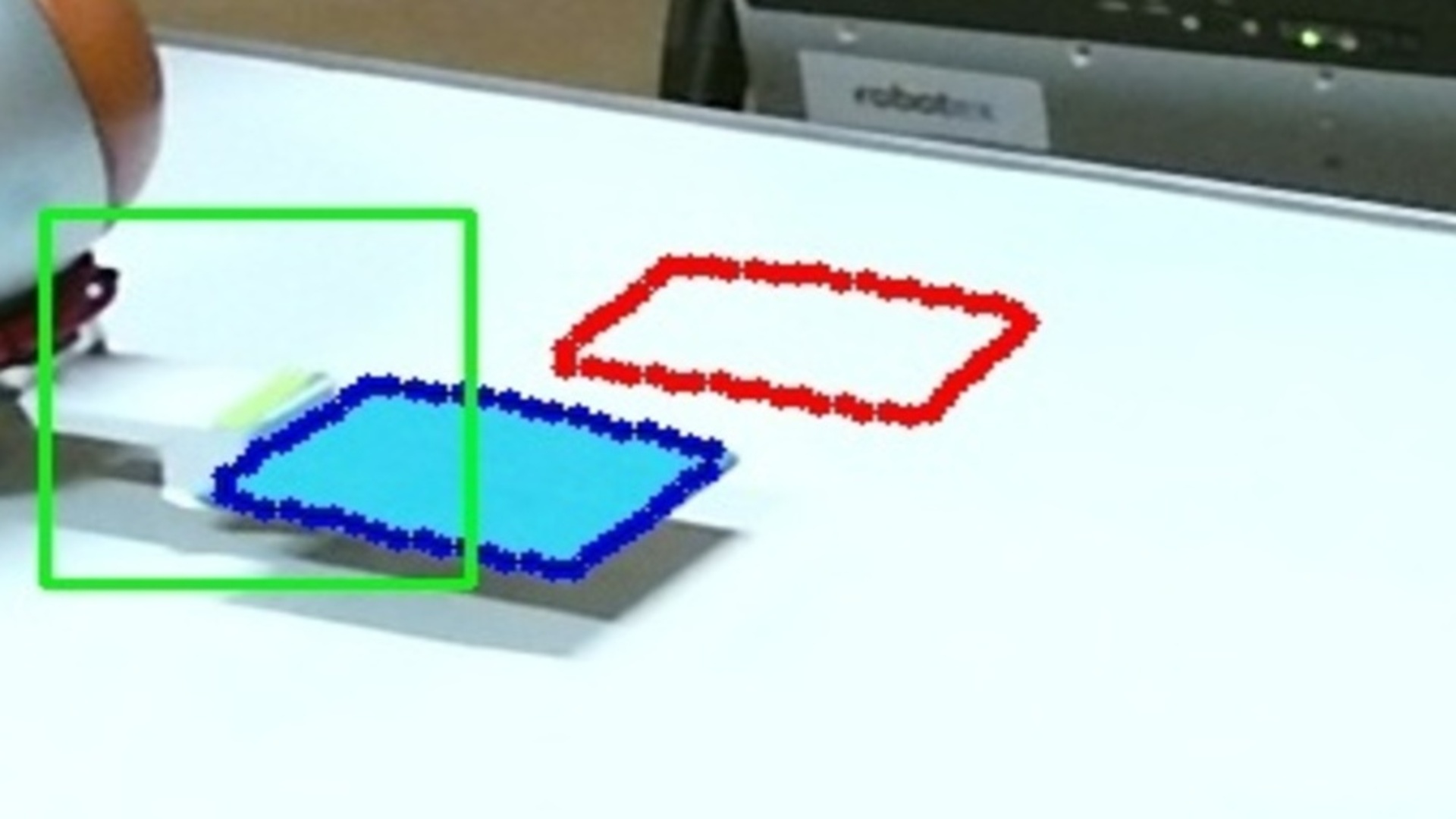}}   
	\hspace{1mm}
    \subfloat{\includegraphics[width = 0.11\textwidth]{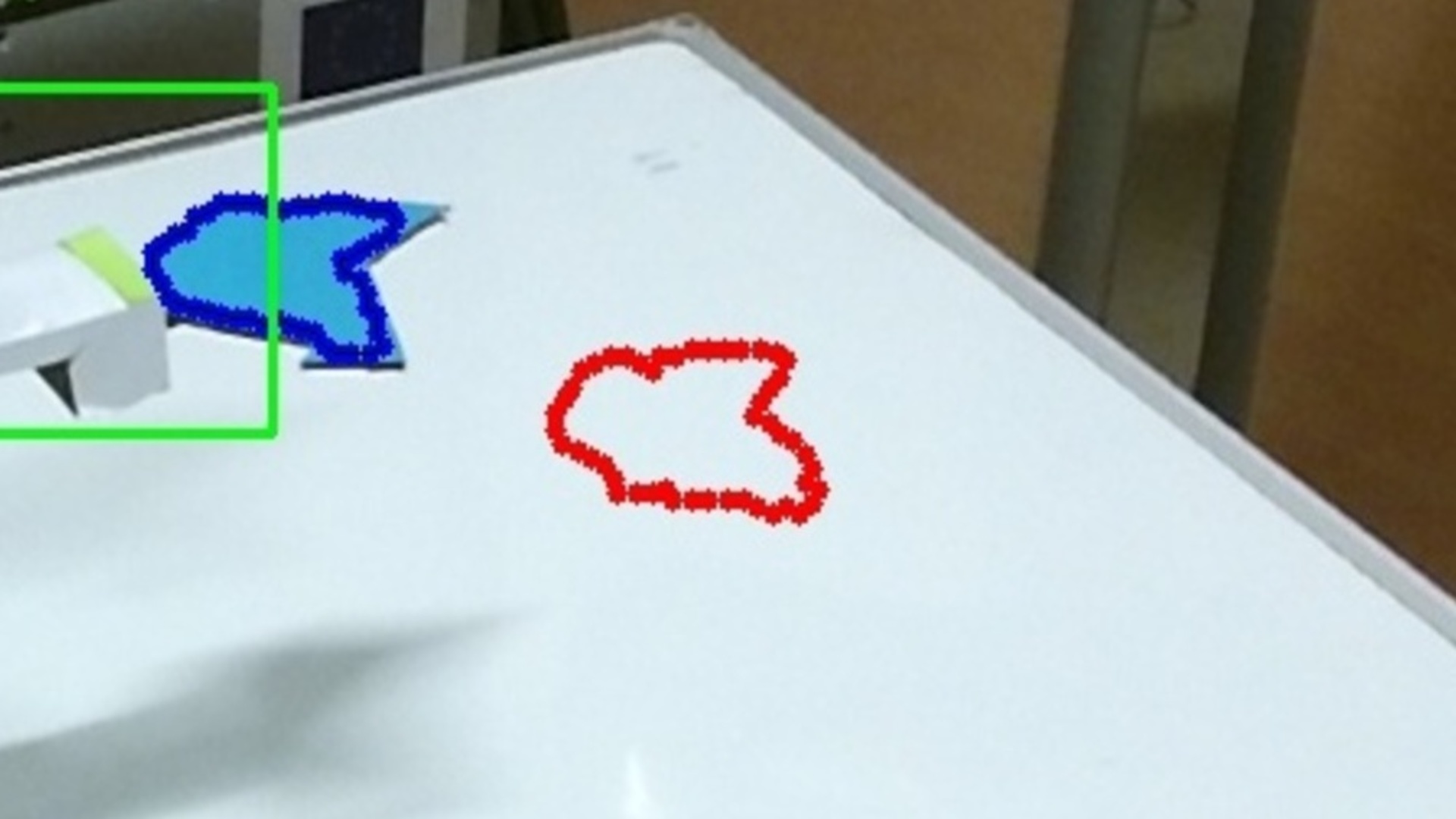}}   
	\hspace{1mm}
	\subfloat{\includegraphics[width = 0.11\textwidth]{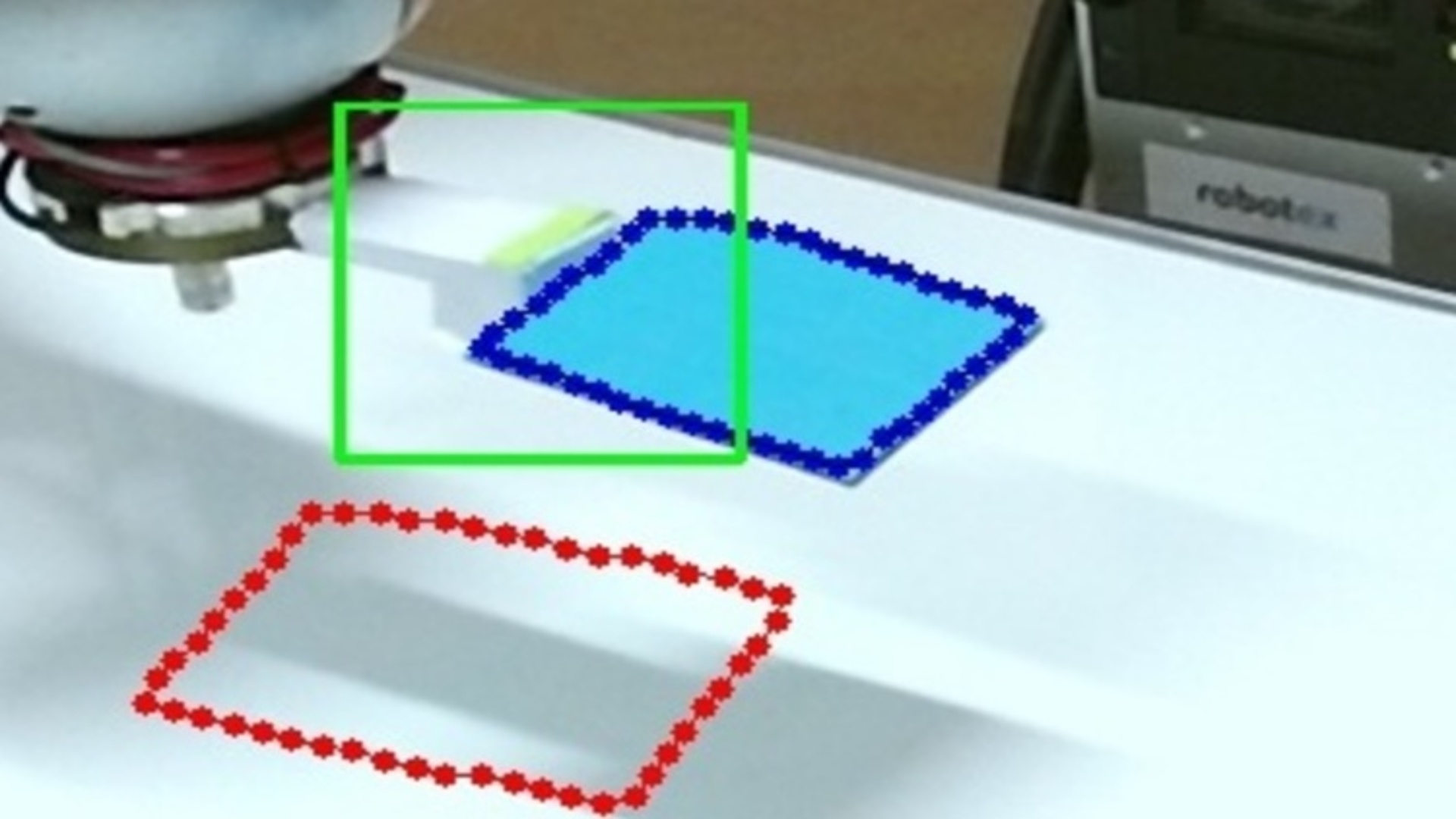}} 
	\hspace{1mm}
	\subfloat{\includegraphics[width = 0.11\textwidth]{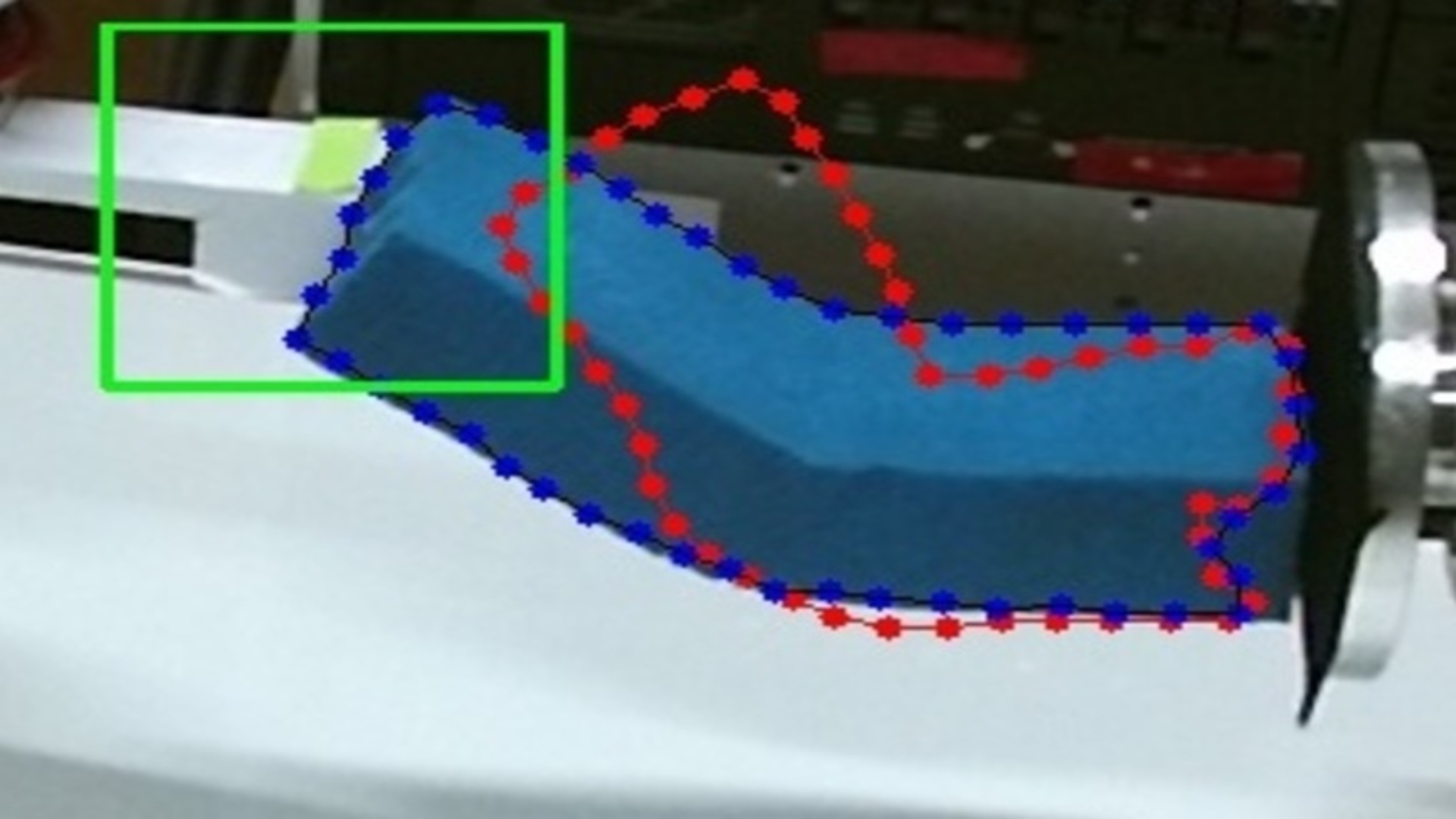}}   
	\hspace{1mm}
	\subfloat{\includegraphics[width = 0.11\textwidth]{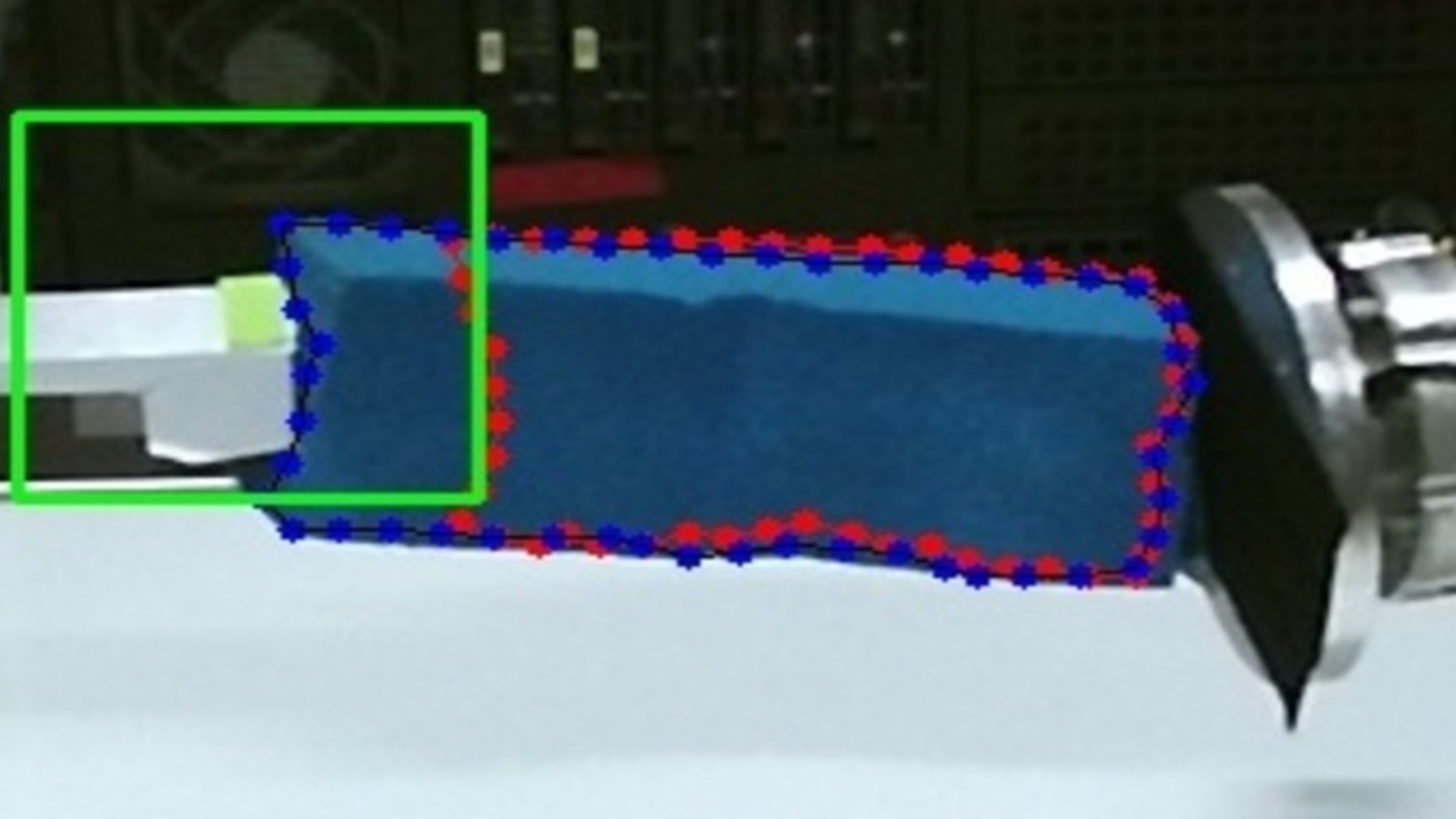}}
    \\
    \subfloat{\includegraphics[width = 0.11\textwidth]{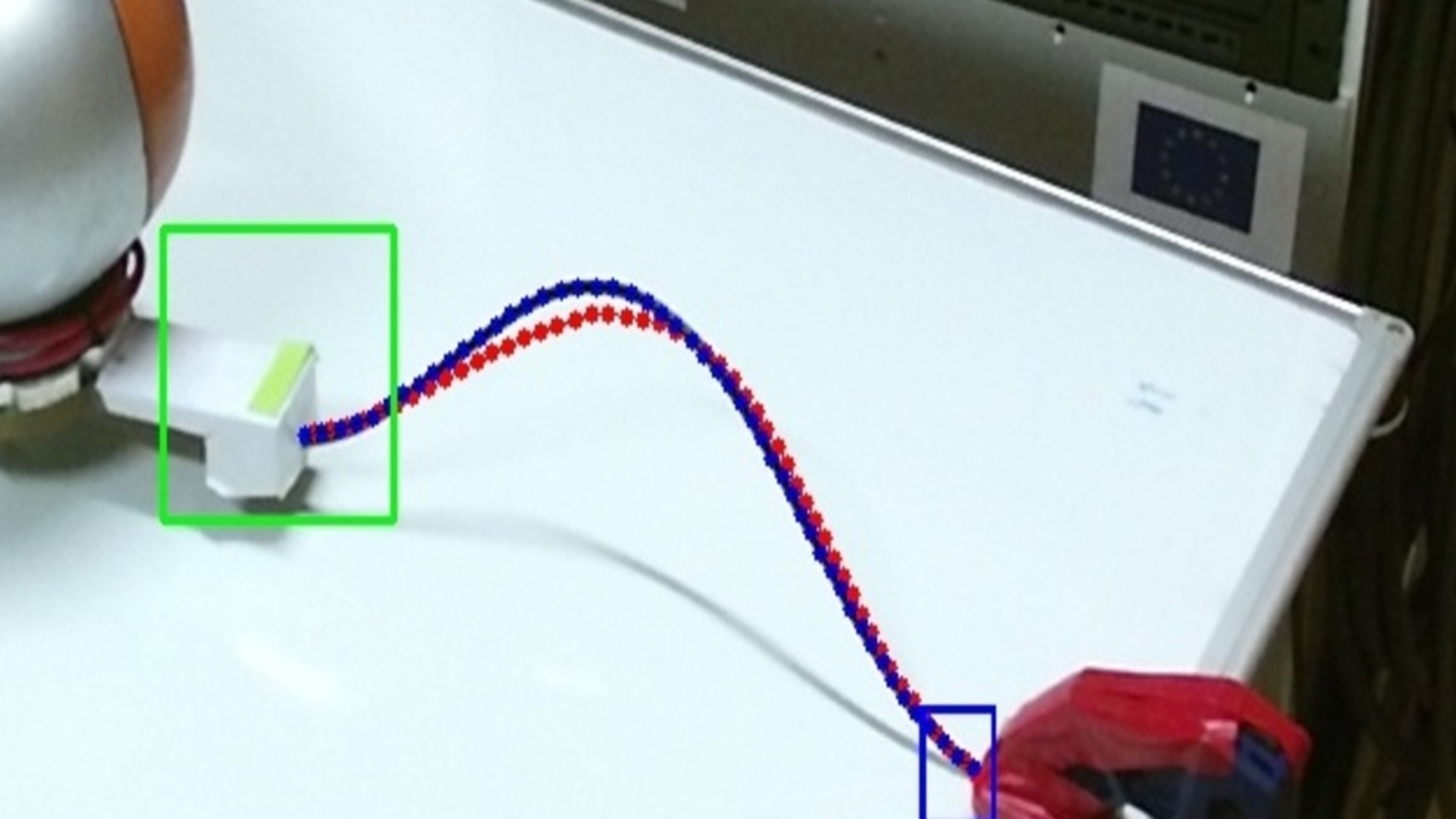}} 
    \hspace{1mm}
    \subfloat{\includegraphics[width = 0.11\textwidth]{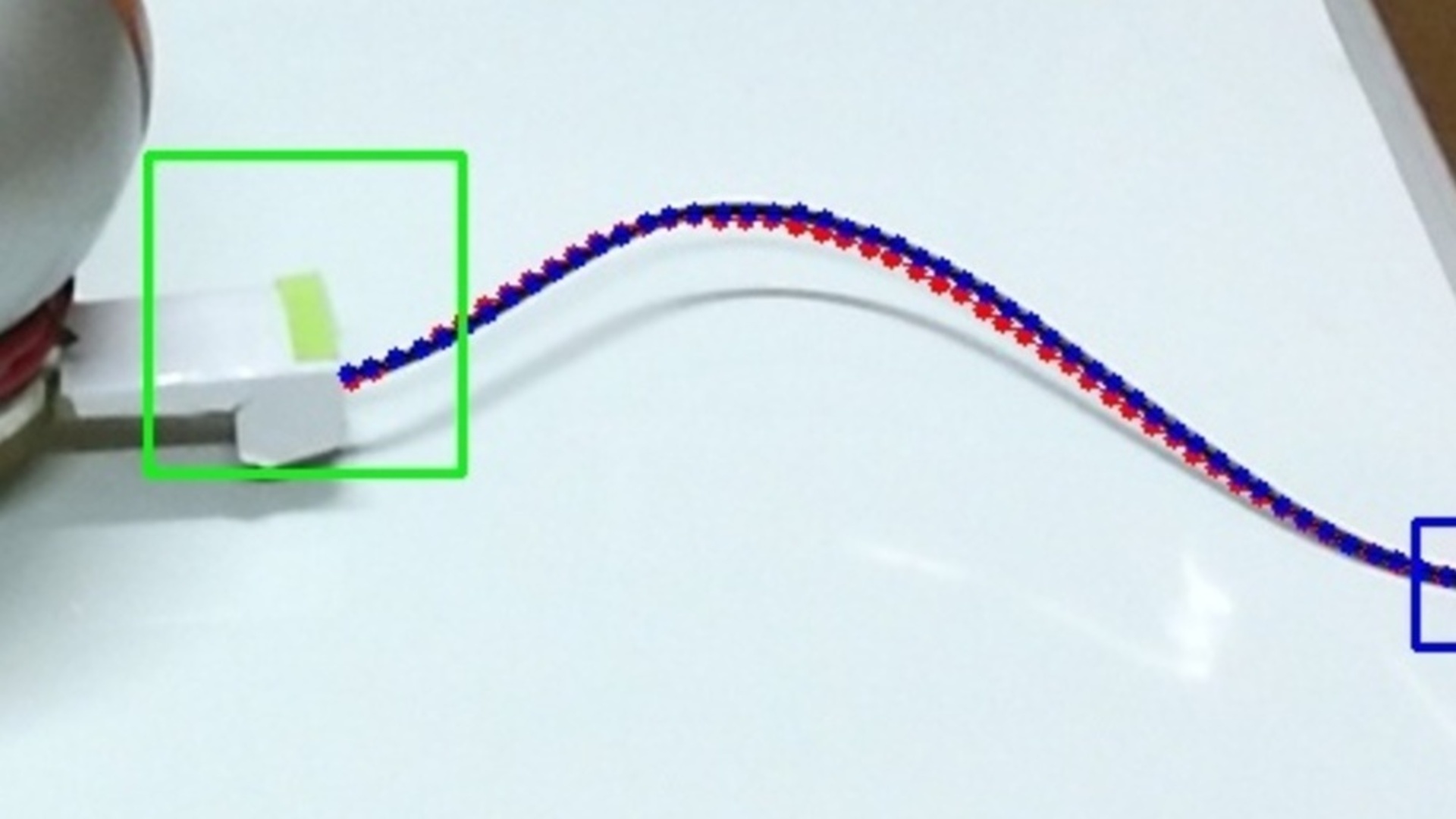}}
    \hspace{1mm}
    \subfloat{\includegraphics[width = 0.11\textwidth]{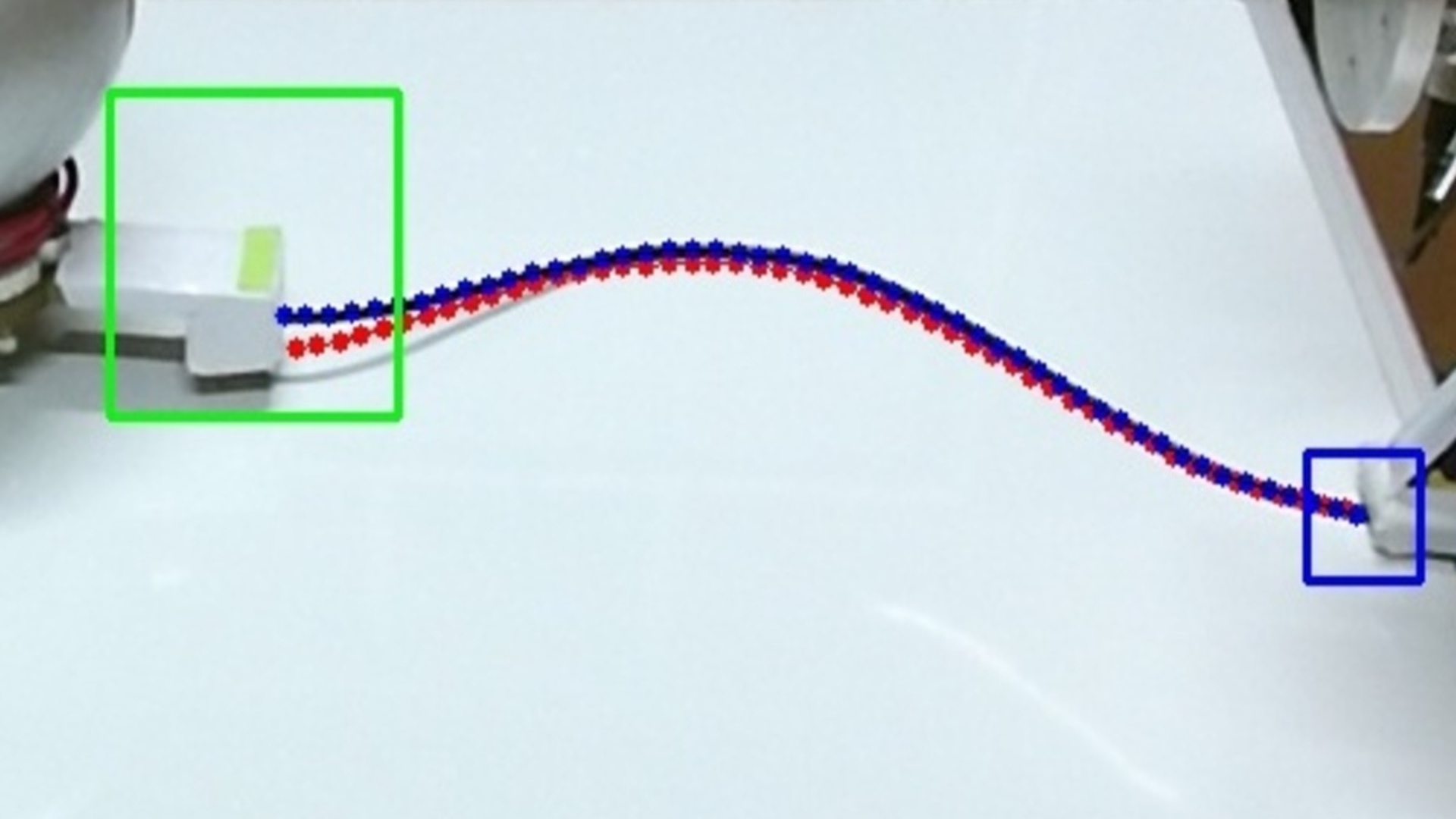}}
    \hspace{1mm}
	\subfloat{\includegraphics[width = 0.11\textwidth]{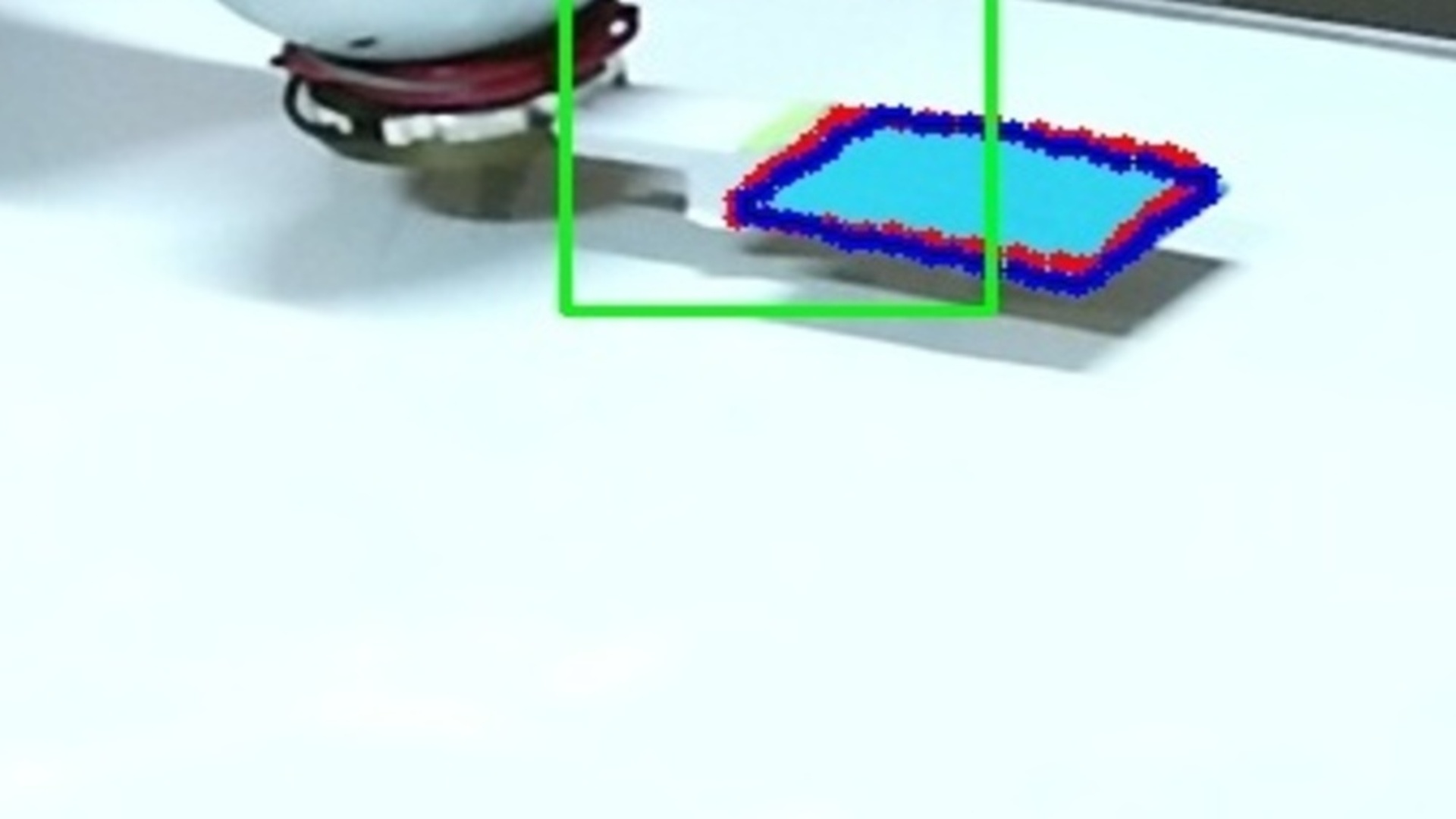}}   
	\hspace{1mm}
    \subfloat{\includegraphics[width = 0.11\textwidth]{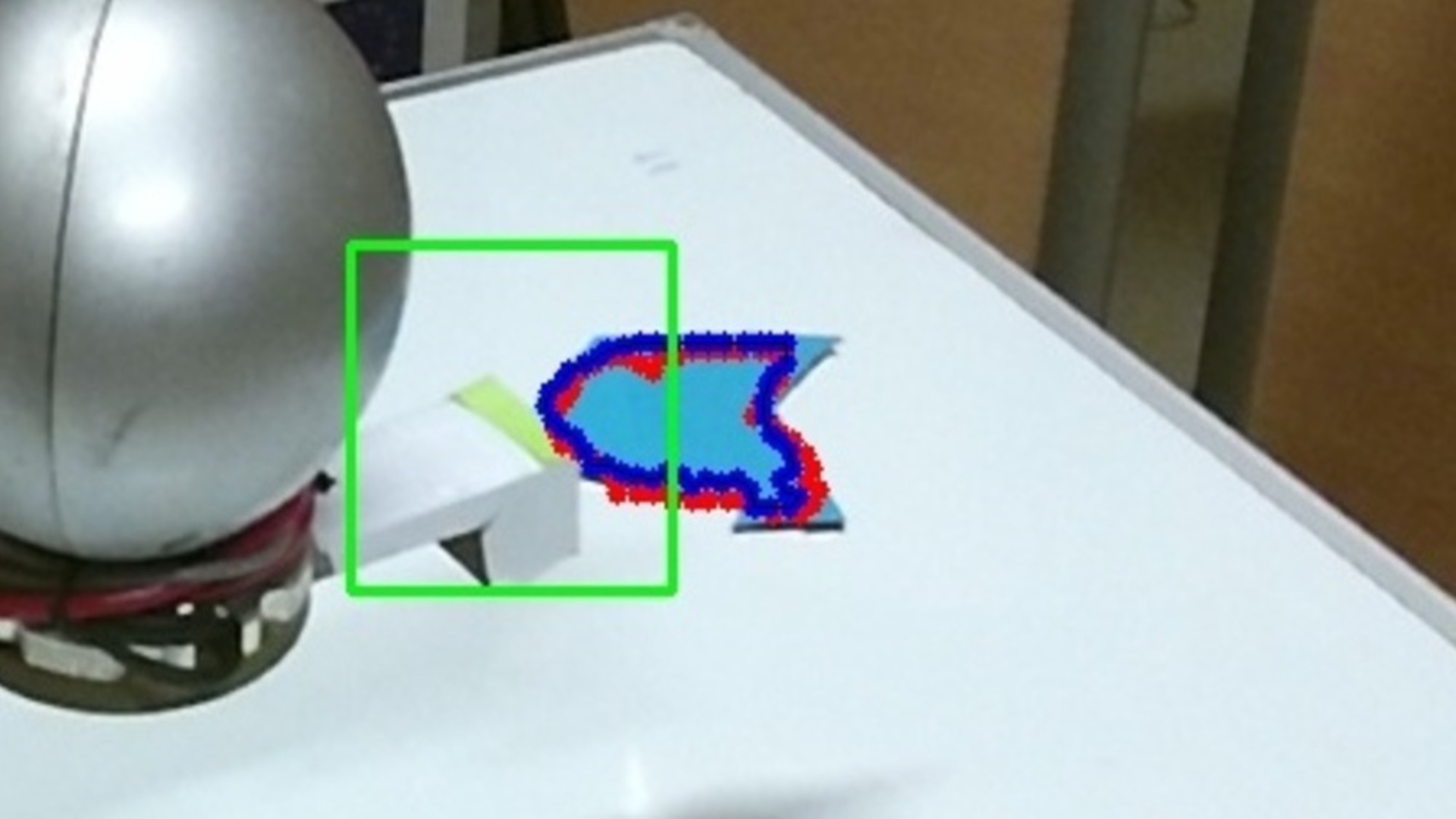}}   
	\hspace{1mm}
	\subfloat{\includegraphics[width = 0.11\textwidth]{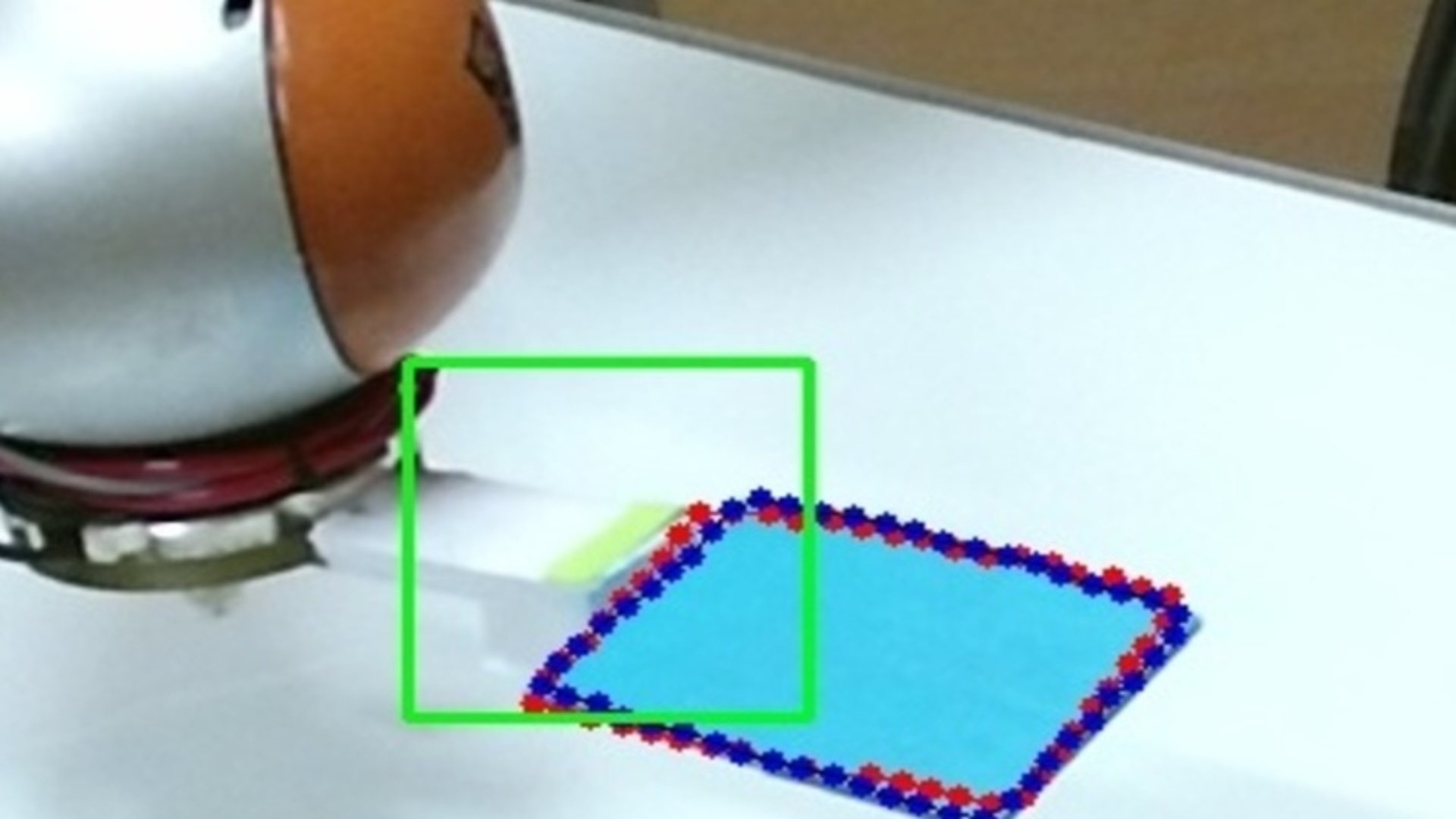}}  
	\hspace{1mm}
    \subfloat{\includegraphics[width = 0.11\textwidth]{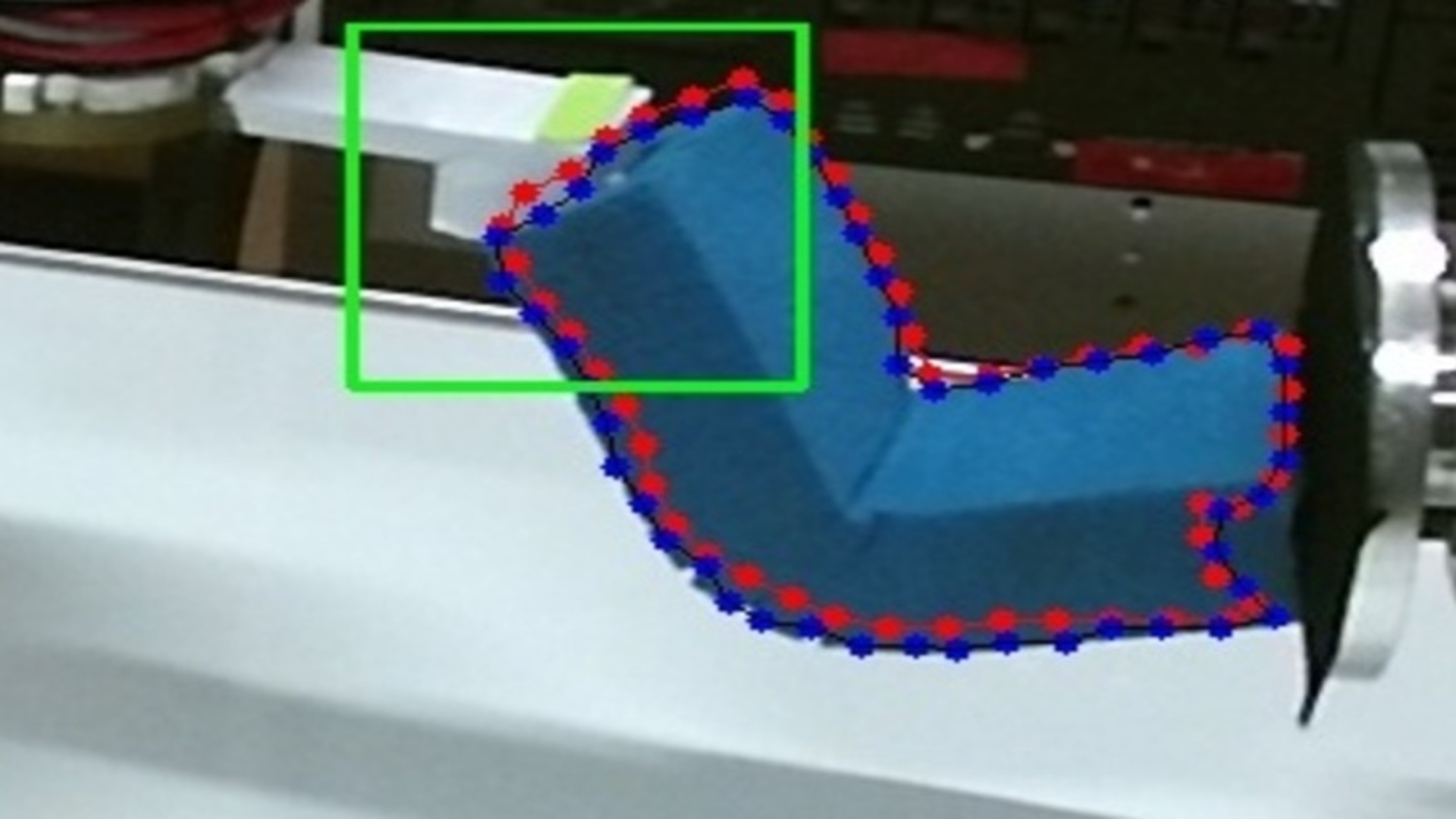}}   
	\hspace{1mm}
	\subfloat{\includegraphics[width = 0.11\textwidth]{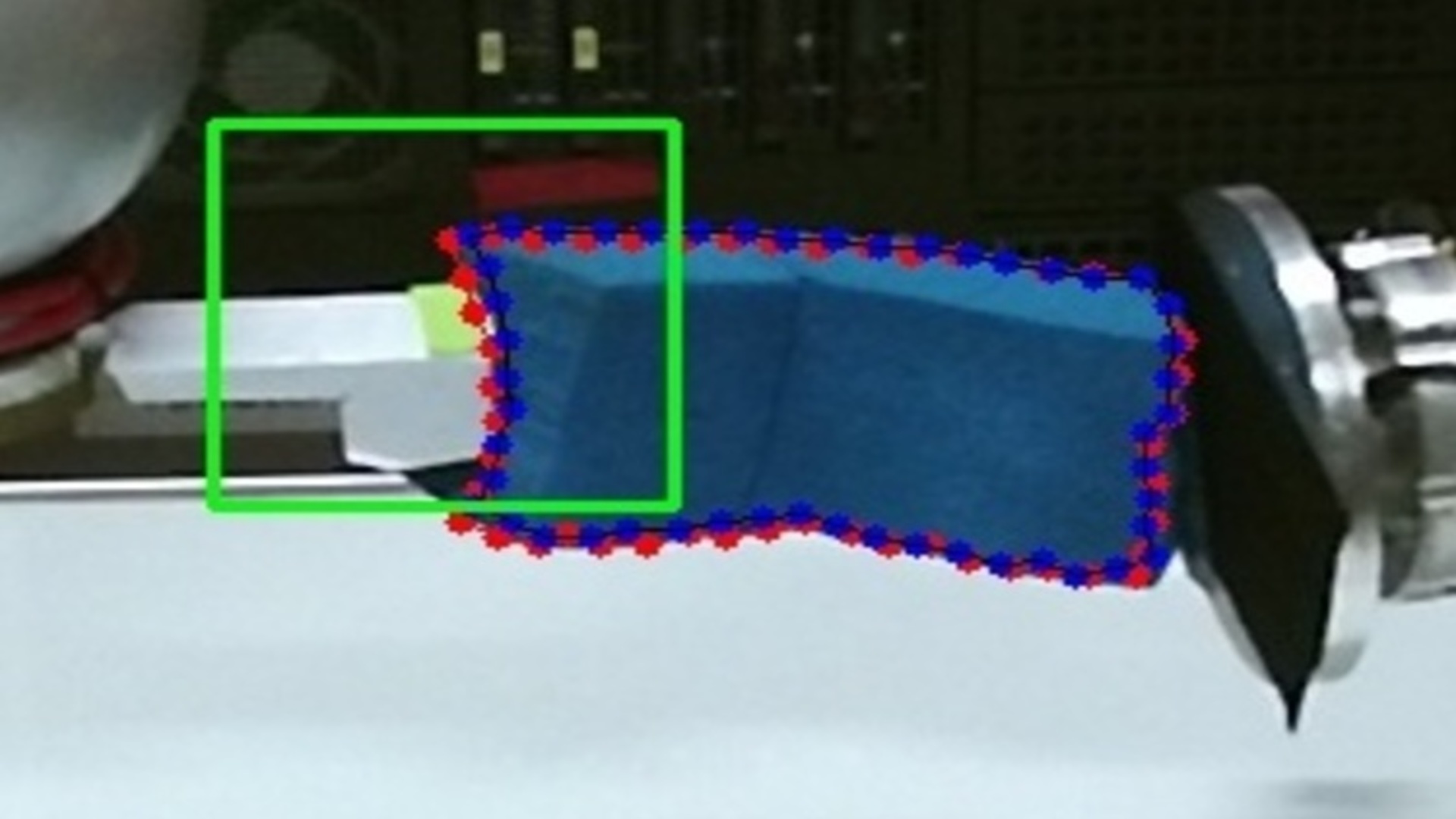}} 
    \caption{Eight experiments with the robot manipulating different objects. From left to right: a cable (columns 1 -- 3), a rigid object (columns 4 -- 6) and a sponge (columns 7 and 8). The first row shows the full Kinect V2 view, and the second and the third columns zoom in to show the manipulation process at the first and last iterations. The red contour is the target one, whereas the blue contour is the current one. The green square indicates the end-effector.}
    \vspace{-0.7cm}
    \label{fig:robot_experiment}
\end{figure*}

\begin{figure*}[!thpb]
    \centering
    \subfloat[cable -- column 1]{\includegraphics[width = 0.25\textwidth]{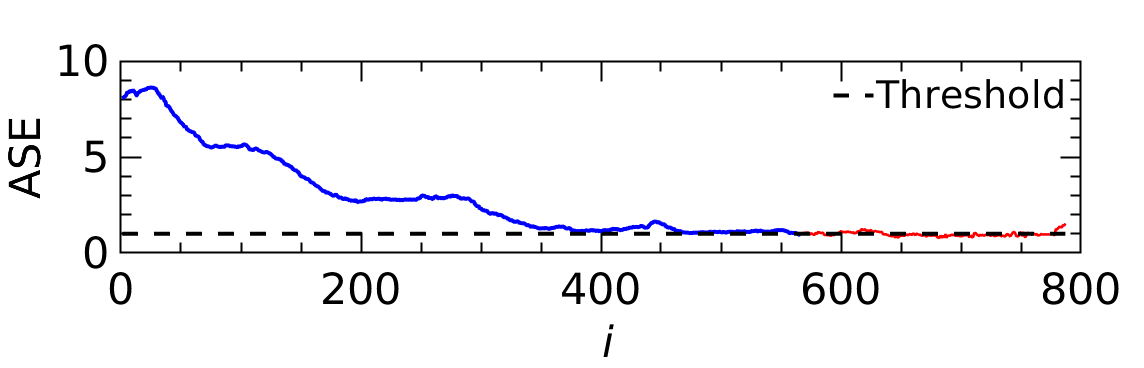}\label{fig:cable_exp_error_1}}
    \subfloat[cable -- column 2]{\includegraphics[width = 0.25\textwidth]{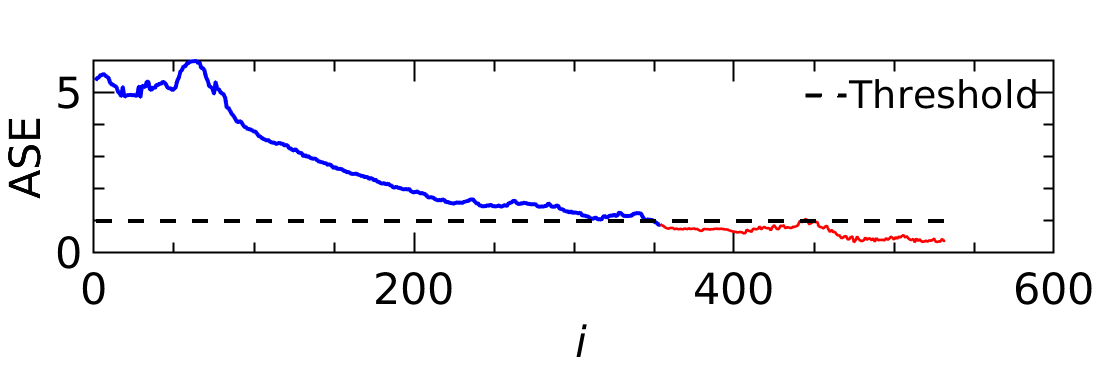}\label{fig:cable_exp_error_2}}
    \subfloat[cable -- column 3]{\includegraphics[width = 0.25\textwidth]{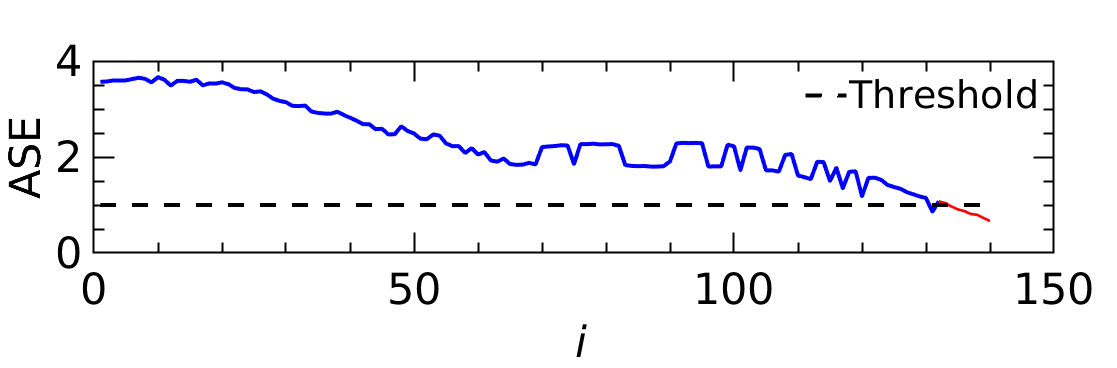}\label{fig:cable_exp_error_3}}
    \subfloat[rigid object -- column 4]{\includegraphics[width = 0.25\textwidth]{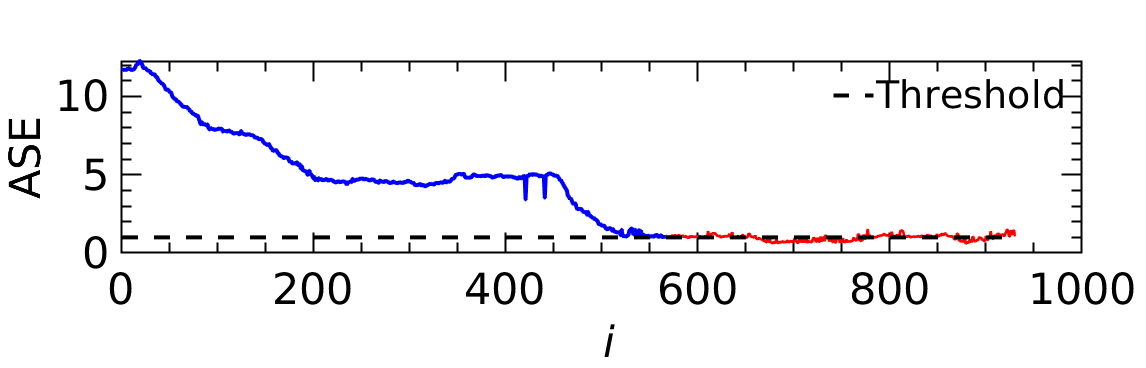}\label{fig:rigid_exp_error_1}}
    \\
    \subfloat[rigid object -- column 5]{\includegraphics[width = 0.25\textwidth]{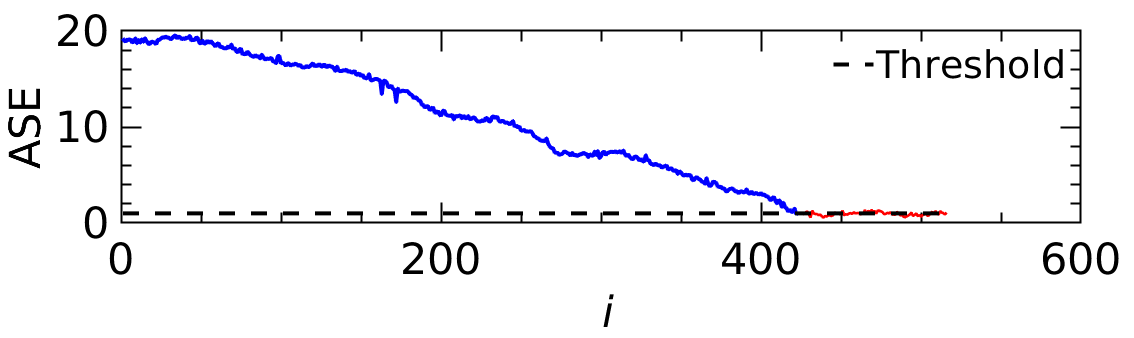}\label{fig:rigid_exp_error_2}}
    \subfloat[cable -- column 6]{\includegraphics[width = 0.25\textwidth]{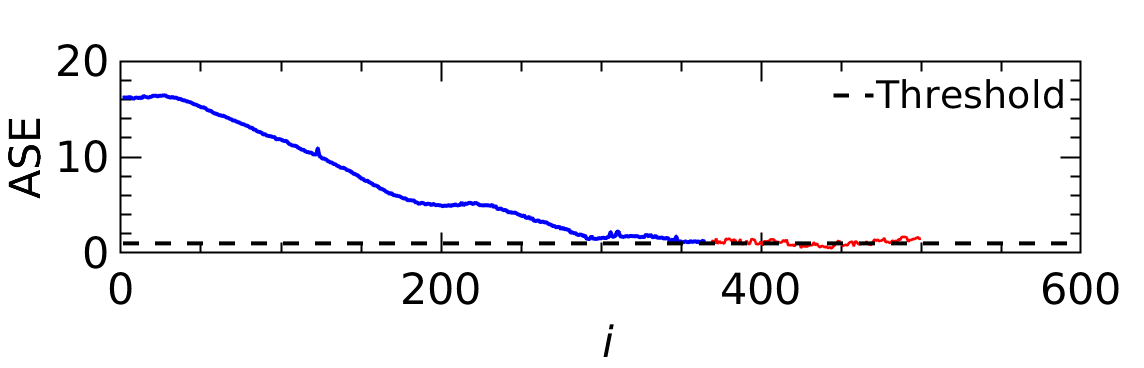}\label{fig:rigid_exp_error_3}}
    \subfloat[sponge -- column 7]{\includegraphics[width = 0.25\textwidth]{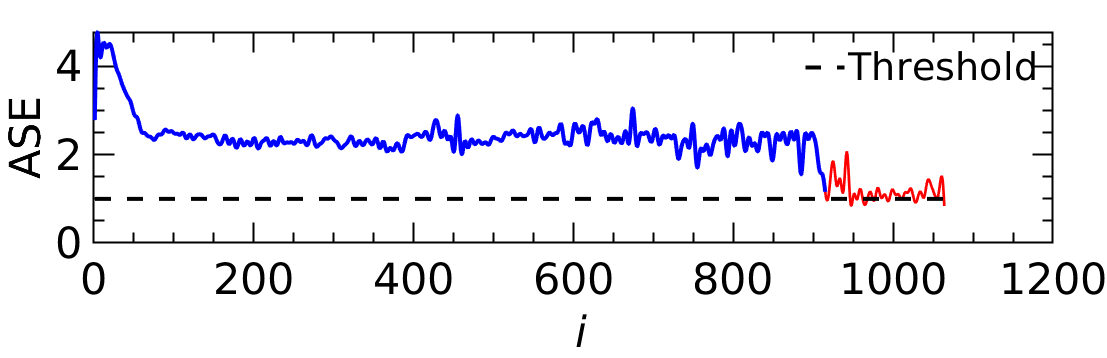}\label{fig:foam_exp_error_1}}
    \subfloat[sponge -- column 8]{\includegraphics[width = 0.25\textwidth]{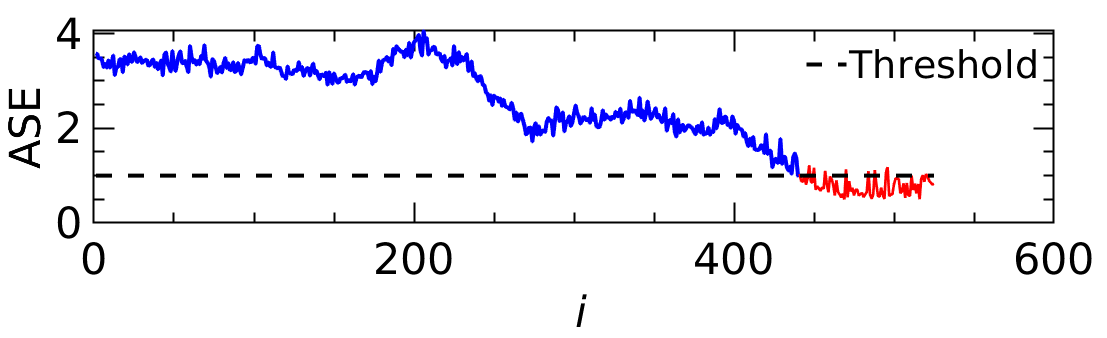}\label{fig:foam_exp_error_2}}
    \caption{Evolution of $e_i$ at each iteration $i$, for the $8$ experiments of Fig. \ref{fig:robot_experiment}. The black dashed lines indicate the threshold $\text{ASE} = 1$ pixel. The blue curves show $e_i$ until the termination condition, whereas the red curves show the error until manual termination by the human operator.}
    \vspace{-0.3cm}
    \label{fig:cable_experiment_error}
\end{figure*}

Figure \ref{fig:cable_experiment_error} shows the decreasing trend of error \text{ASE} for each experiment. The initial increase of \text{ASE} in the experiments can be due to the random motion at the beginning of the experiments. In general, we found that \text{ASE} is more noisy for the closed than for the open contour. This discontinuity is visible in Figures \ref{fig:cable_exp_error_3} and \ref{fig:rigid_exp_error_1} (zigzag evolution). Such noise is likely introduced by the way we sampled the contour. Also, the noise in the contour extraction is more visible on Fig. \ref{fig:foam_exp_error_1} and \ref{fig:foam_exp_error_2}. The two plots show that our framework can converge to the target in the presence of noisy image processing. When we have false contour data, the value of \text{ASE} may encounter a sudden discontinuity. Figure \ref{fig:false_contour_data} shows examples of these false samples, output by the image processing pipeline. Despite these errors, thanks to the ``forgetting nature'' of the receding horizon and to the relatively small window size ($M = 5$), the corrupted data will soon be forgotten, and it will not hinder the overall manipulation task. Yet, the overall framework would benefit from a more robust sensing strategy, as in \citep{Cheng2019Occlusion}. 

\begin{figure}[!thpb]
    \centering
    \subfloat{\includegraphics[width = 0.4\textwidth]{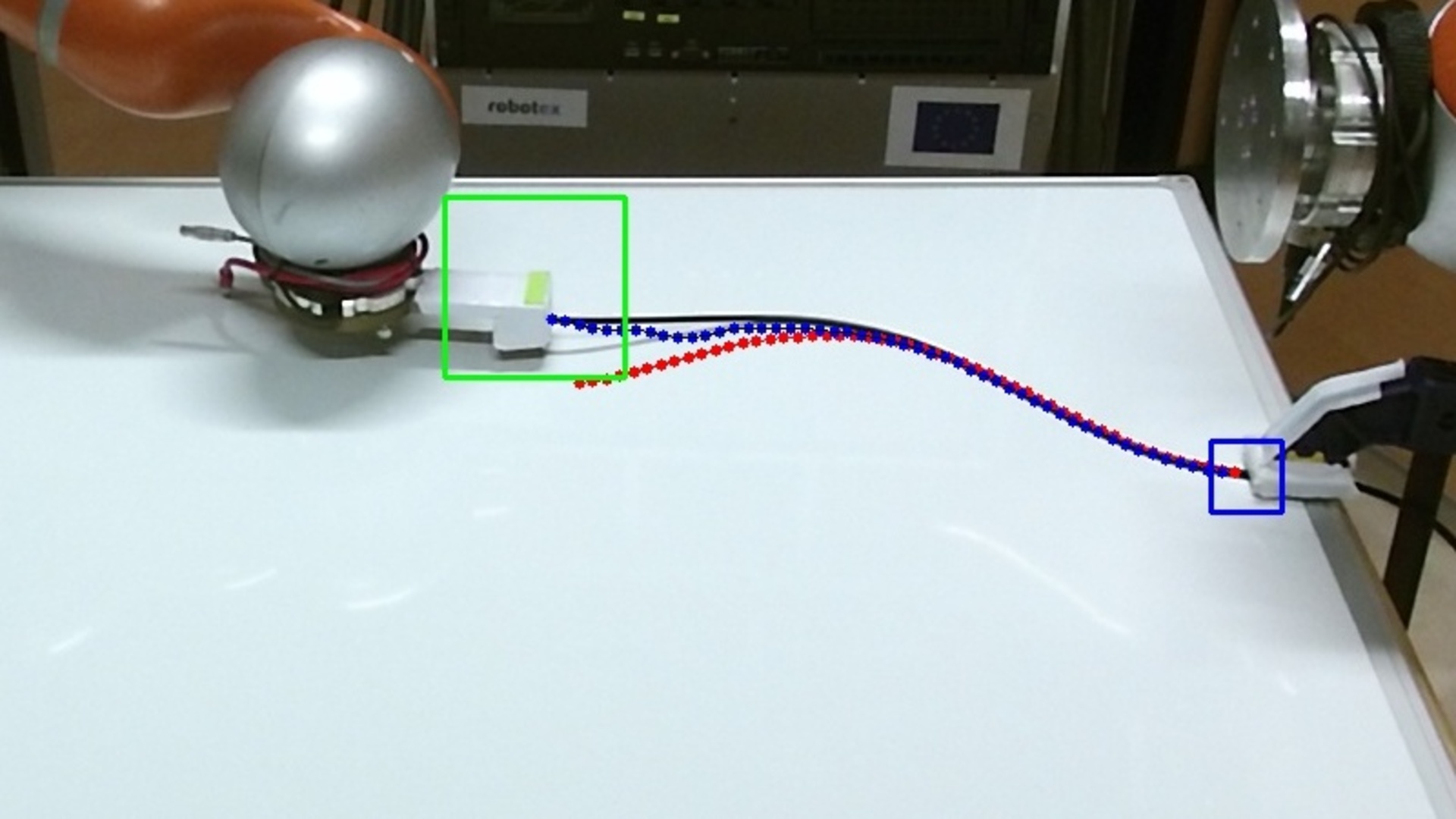}\label{fig:cable_exp_error_1}}
    \hspace{5mm}
    \subfloat{\includegraphics[width = 0.4\textwidth]{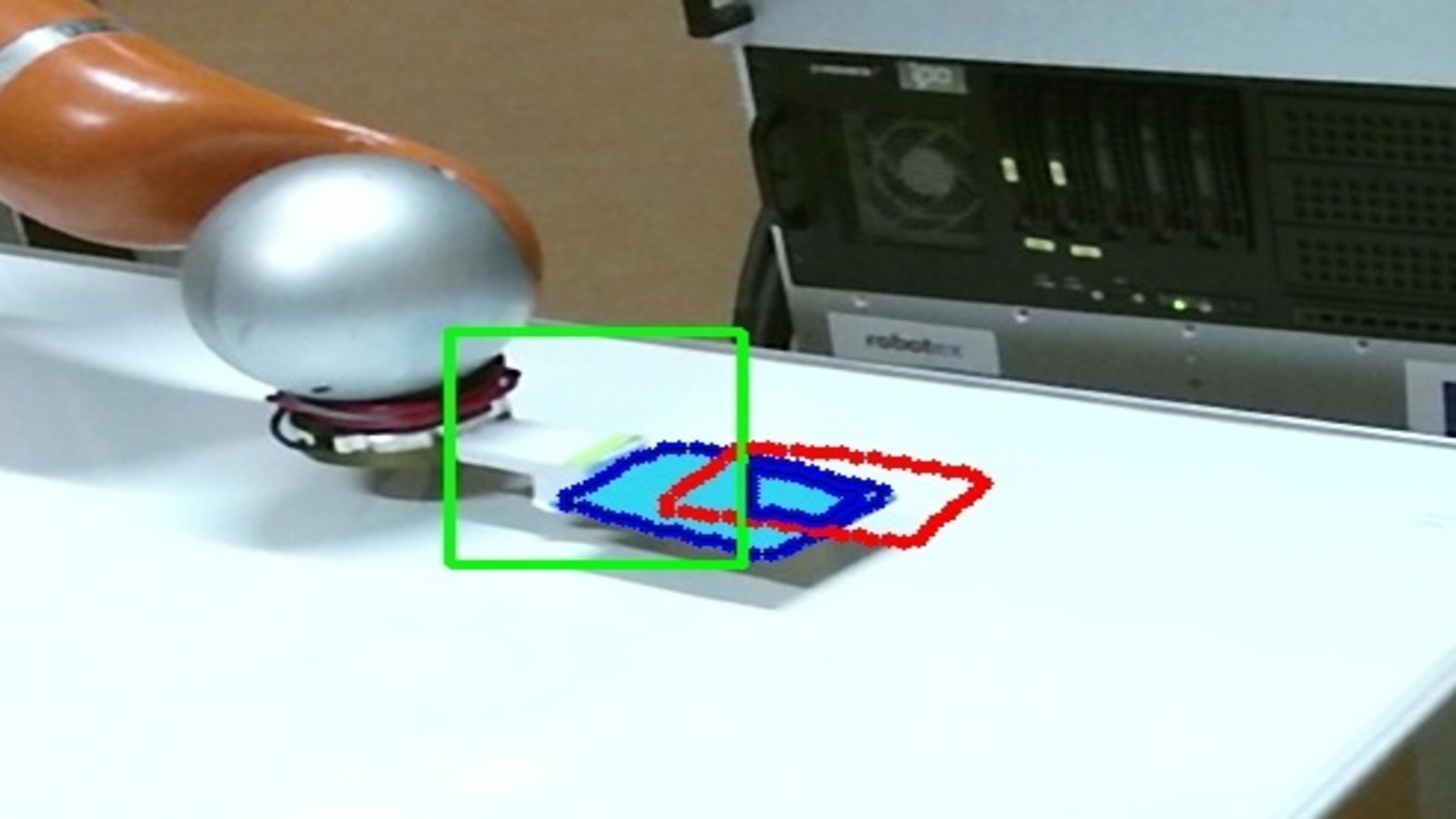}\label{fig:cable_exp_error_2}}
    \caption{False contour data from the image can cause noise in \text{ASE}.}
    \label{fig:false_contour_data}
\end{figure}

\begin{figure}[!thpb]
    \centering
    \subfloat[]{\includegraphics[width = 0.24\textwidth]{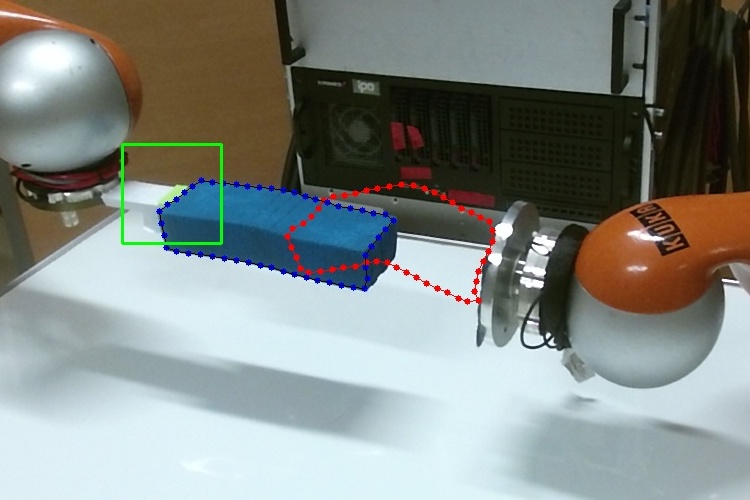}\label{fig:move_shape_1-init}}
    \hspace{0.5mm}
    \subfloat[]{\includegraphics[width = 0.24\textwidth]{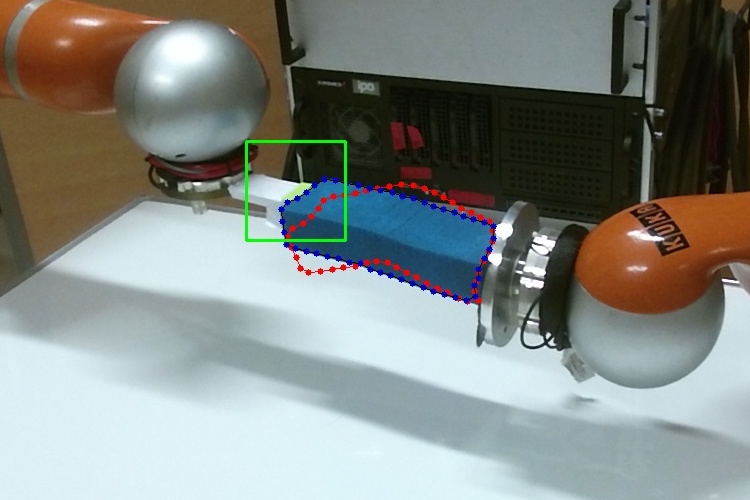}\label{fig:move_shape_1-move}}
    \hspace{0.5mm}
    \subfloat[]{\includegraphics[width = 0.12\textwidth]{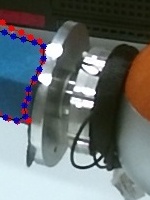}\label{fig:move_shape_1-alignment}}
    \hspace{0.5mm}
    \subfloat[]{\includegraphics[width = 0.24\textwidth]{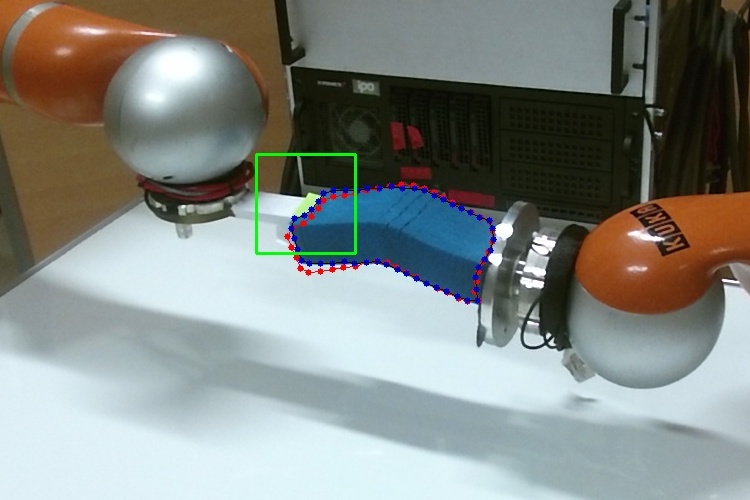}\label{fig:move_shape_1-shape}}
    \\
    \subfloat[]{\includegraphics[width = 0.24\textwidth]{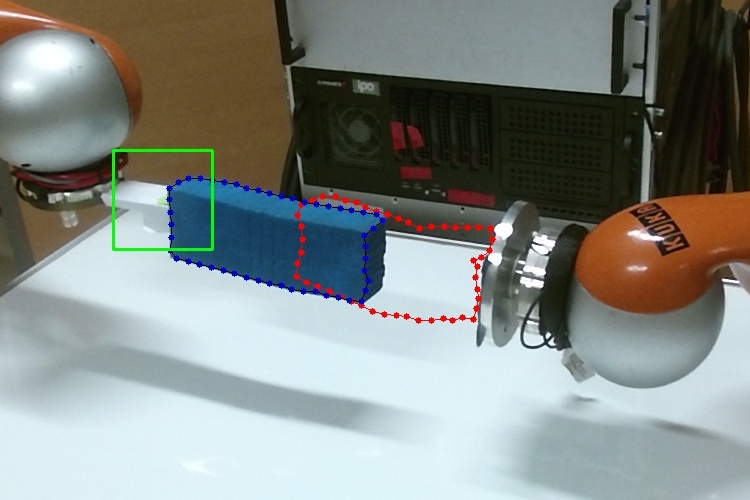}\label{fig:move_shape_2-init}}
    \hspace{0.5mm}
    \subfloat[]{\includegraphics[width = 0.24\textwidth]{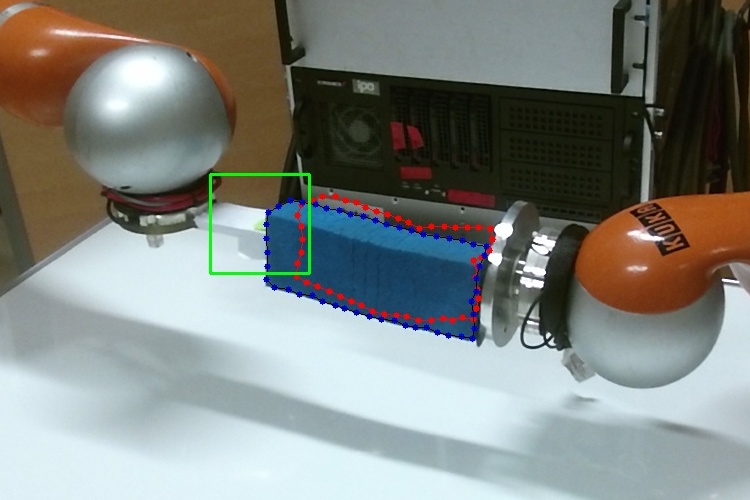}\label{fig:move_shape_2-move}}
    \hspace{0.5mm}
    \subfloat[]{\includegraphics[width = 0.12\textwidth]{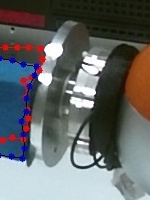}\label{fig:move_shape_2-alignment}}
    \hspace{0.5mm}
    \subfloat[]{\includegraphics[width = 0.24\textwidth]{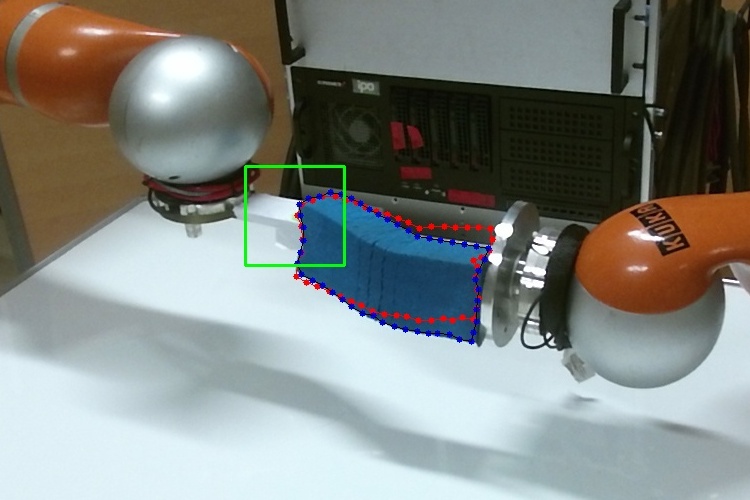}\label{fig:move_shape_2-shape}}
    \caption{Two ``move and shape'' experiments grouped into two rows. The target contour (red dotted) is far from the initial one. This requires the robot to 1) move the object, establish contact with the right -- fixed -- robot arm, 2) give the object the target shape, by relying on the contact. The first column shows the starting configuration, the second column presents the contact establishment, and the third column zooms in to show the alignment. The last column shows the final results.}
    \label{fig:move_and_shape}
\end{figure}

\begin{figure}[!thpb]
    \centering
    \subfloat[move and shape -- row 1]{\includegraphics[width = 0.4\textwidth]{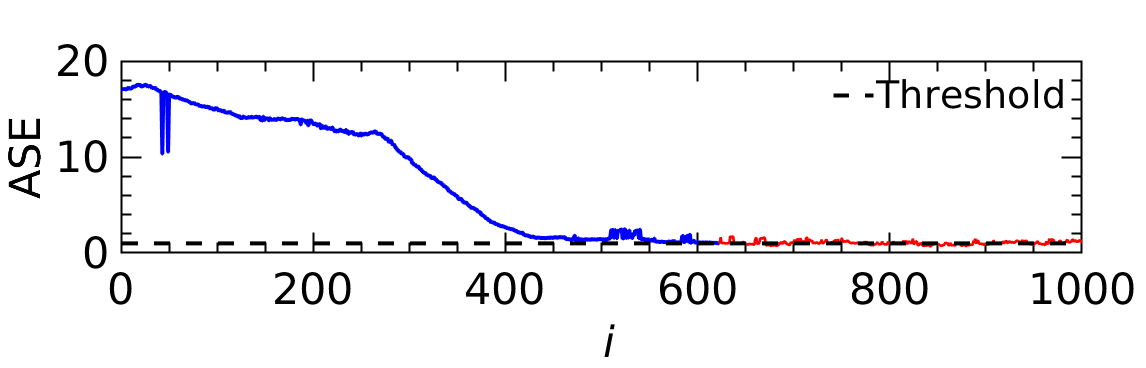}\label{fig:move_shape_error_1}}
    \subfloat[move and shape -- row 2]{\includegraphics[width = 0.4\textwidth]{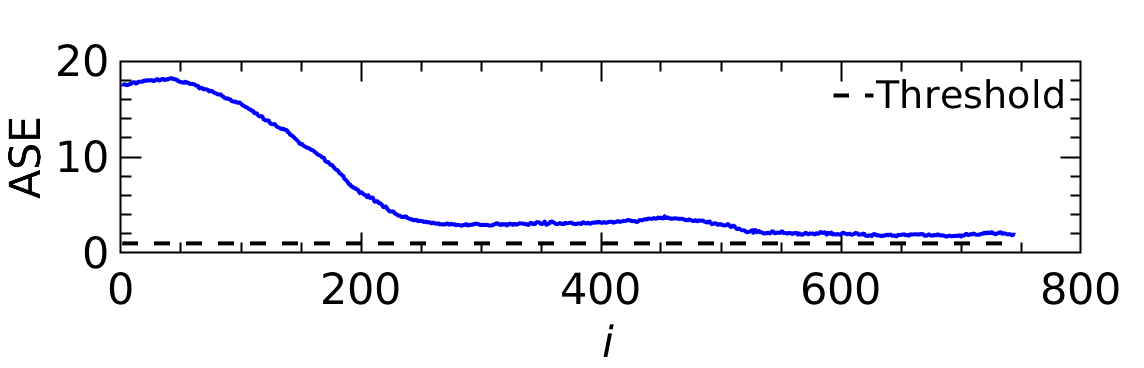}\label{fig:move_shape_error_2}}
    \caption{The evolution of $e_i$ for the experiments of Fig. \ref{fig:move_and_shape}. The black dashed line indicates the threshold $\text{ASE} = 1$. The blue curves show $e_i$ until the termination condition, whereas the red curves show the error until manual termination by the human operator.}
    \label{fig:move_shape_error}
\end{figure}

Finally, since our framework can deal with both rigid and deformable objects, we tested it in two experiments where the same object (a sponge) can be both rigid (in the free space), and deformed (when in contact with the environment). These experiments require the robot to: 1) move the object, establish contact, 2) give the object the target shape, by relying on the contact. Figure \ref{fig:move_and_shape} presents these two original ``move and shape'' servoing experiments with the corresponding errors \text{ASE} plotted in Fig. \ref{fig:move_shape_error}. We use a second fixed robot arm to generate the deforming contact. As the figures and curves show, both experiments were successful.

The success of the ``move and shape'' task is largely dependent on the contact establishment. However, even when the initial contact has some misalignment (see Fig. \ref{fig:move_shape_1-alignment} and \ref{fig:move_shape_2-alignment}), our framework can still reduce the \text{ASE} to give a reasonable final configuration (see Fig. \ref{fig:move_shape_2-shape} and Fig. \ref{fig:move_shape_error_2}).

\section{Conclusion}\label{sec:conclusion}

In this paper, we propose algorithms to automatically and concurrently generate object representations (feature vectors) and models of interaction (interaction matrices) from the same data. We use these algorithms to generate the control inputs enabling a robot to move and shape the said object, be it rigid or deformable. The scheme is validated with comprehensive experiments, including a target contour that requires both moving and shaping. We believe it is unprecedented in previous research. Our framework adopts a model-free approach. The system characteristics are computed online with visual and manipulation data. We do not require camera calibration, nor a priori knowledge of the camera pose, object size or shape. 

The proposed approach has two major limitations: 1. \emph{The challenge of extending it to 6 DoF motion}, 2. \emph{Global convergence cannot be guaranteed}. Below, we discuss each limitation and present possible solutions.

An open question remains the management of 6 DOF motion of the robot. Indeed, while the proposed controller can be easily generalized to 6 DOF motions, it relies on a sufficiently accurate extraction of feature vectors from vision sensors. A very challenging task is to generate complete and reliable 3D feature vectors of objects from a limited sensor set, due to partial views of the object and to occlusions. To extend it, the framework should benefit from robust deformation sensing. In addition, since the approach relies on local linear models, we expect that with higher DOF, the algorithm will more likely get stuck in local minima. 

The second drawback is that the representation and model of interaction are local. Thus, they cannot guarantee global convergence. In addition, our framework cannot infer whether a shape is reachable or not. This drawback is solvable by using a global deformation model for control. But as we mentioned earlier, a global model usually requires an offline identification phase which we want to avoid. In fact, for different objects, we will need to re-identify the model. There is a dilemma in using a global deformation model.

Maybe one of the possible solutions to this dilemma is to have both our method and deep learning based methods run in parallel. While our scheme enables fast online computation and direct manipulation, the extracted data can be used by a deep neural network to obtain a global interaction mapping. Once a global mapping is learned, it can later be used for direct manipulation and to infer feasibility of the goal shape.

\section*{Acknowledgement}
This work is supported in part by the EU H2020 research and innovation programme as part of the project VERSATILE under grant agreement No 731330, by the Research Grants Council (RGC) of Hong Kong under grant number 14203917, and by the PROCORE-France/Hong Kong RGC Joint Research Scheme under grant F-PolyU503/18.

\bibliography{biblio.bib}
\bibliographystyle{elsarticle-harv} 

\end{document}